\documentclass{article}
\usepackage[accepted]{icml2024}

\usepackage{microtype}
\usepackage{graphicx}
\usepackage{subfigure}
\usepackage{booktabs} 
\usepackage{hyperref}

\usepackage{amsmath}
\usepackage{amssymb}
\usepackage{amsthm}
\usepackage{mathtools}
\usepackage{enumitem}
\usepackage{multirow}
\usepackage{multicol}
\usepackage{rotating}
\usepackage{color}
\usepackage[capitalize,noabbrev]{cleveref}

\theoremstyle{plain}
\newtheorem{theorem}{Theorem}[section]
\newtheorem{proposition}[theorem]{Proposition}
\newtheorem{lemma}[theorem]{Lemma}

\theoremstyle{definition}
\newtheorem{definition}[theorem]{Definition}

\theoremstyle{remark}
\newtheorem{remark}[theorem]{Remark}

\def\R{{\mathbb{R}}}

\def\N{{\mathbb{N}}}
\def\E{{\mathbb{E}}}
\def\Pr{{\mathbb{P}}}
\def\Var{{\mathbb{V}}}

\def\Risk{{\mathcal{E}}}
\def\RRisk{{\mathcal{R}}}

\def\p{{\mathcal{P}}}
\def\F{{\mathcal{F}}}
\def\D{{\mathcal{D}}}

\def\H{{\mathcal{H}}}
\def\G{{\mathcal{G}}}
\def\S{{\mathcal{S}}}

\def\Hyp{{\mathcal{H}}}

\def\sign{{\hbox{\rm{sign}}}}

\def\d{{\hbox{\rm{d}}}}

\def\tr{{{\rm{tr}}}}

\def\argmin{{\hbox{\rm{argmin}}}}
\def\argmax{{\hbox{\rm{argmax}}}}
\def\Rad{{\hbox{\rm{Rad}}}}
\def\Softmax{{\hbox{\rm{Softmax}}}}

\def\bfepsilon{{\mathbf{\varepsilon}}}

\def\ID{{\mathbf{1}}}

\icmltitlerunning{Generalization Bound and New Algorithm for Clean-Label Backdoor Attack}

\begin{document}
 
\twocolumn[
\icmltitle{Generalization Bound and New Algorithm for Clean-Label Backdoor Attack
}




\begin{icmlauthorlist}
\icmlauthor{Lijia Yu}{2,5}
\icmlauthor{Shuang Liu}{3,1}
\icmlauthor{Yibo Miao}{3,1}
\icmlauthor{Xiao-Shan Gao}{3,1,4}
\icmlauthor{Lijun Zhang}{2,5,1}
\end{icmlauthorlist}

\icmlaffiliation{2}{Institute of Software, Chinese Academy of Sciences, Beijing 100190, China}
\icmlaffiliation{3}{Academy of Mathematics and Systems Science, Chinese Academy of Sciences, Beijing 100190, China}
\icmlaffiliation{1}{University of Chinese Academy of Sciences, Beijing 100049, China}
\icmlaffiliation{4}{Kaiyuan International Mathematical Sciences Institute}
\icmlaffiliation{5}{State Key Laboratory of Computer Science}

\icmlcorrespondingauthor{Xiao-Shan Gao}{xgao@mmrc.iss.ac.cn}
\icmlcorrespondingauthor{Lijun Zhang}{zhanglj@ios.ac.cn}

\icmlkeywords{Machine Learning, ICML}

\vskip 0.3in
]
\printAffiliationsAndNotice{}



\begin{abstract}
\noindent
The generalization bound is a crucial theoretical tool for assessing the generalizability of learning methods and there exist vast literatures on generalizability of normal learning, adversarial learning, and data poisoning. Unlike other data poison attacks, the backdoor attack has the special property that the poisoned triggers are contained in both the training set and the test set and the purpose of the attack is two-fold.
To our knowledge, the generalization bound for the backdoor attack has not been established.
%
In this paper, we fill this gap by deriving algorithm-independent generalization bounds in the clean-label backdoor attack scenario. Precisely, based on the goals of backdoor attack, we give upper bounds for the clean sample population errors and the poison population errors in terms of the empirical error on the poisoned training dataset.
%
%
Furthermore, based on the theoretical result, a new clean-label backdoor attack is proposed that computes the poisoning trigger by combining adversarial noise and indiscriminate poison. We show its effectiveness in a variety of settings.
%
%
\end{abstract}

\section{Introduction}
\label{sec-1}

The generalization bound is a key theoretical tool for assessing the generalizability of a learning method.
Let $\widehat\F$ be the classification result of a neural network $\F$. Then the main purpose of a learning algorithm is to minimize the {\em population error} on the data distribution $\D_S$, i.e. 
$\Risk(\F,\D_S)= 
\E_{(x,y)\sim \D_S}[\ID(\widehat \F(x)\ne y)]$.
%
On the other hand, the training procedure can only minimize the {\em empirical error} on a given finite training set $\D_{\tr}$ i.i.d. sampled from $\D_S$, i.e. 
$\E_{(x,y)\in\D_{\tr}}[\ID(\widehat \F(x)\ne y)]=\frac{1}{|\D_{\tr}|}\sum_{(x,y)\in\D_{\tr}} \ID(\widehat \F(x)\ne y)$. 
%
A theoretical way to ensure generalizability is to control the generalization gap between population error and empirical error, and an upper bound of such a gap is called the {\em generalization bound}. 
The learning algorithm has generalizability if the generalization bound approaches zero when the size of the training set is sufficiently large.

Classic generalization bounds were given in terms of the VC-dimension or the Rademacher complexity
\cite{mohri2018foundations}. 
Recently, algorithm-independent generalization bounds depending on the size of the DNNs were given   
\cite{harvey2017nearly,neyshabur2017exploring,bartlett2019nearly}.
Algorithm-dependent generalization bounds were given in the algorithmic stability setting \cite{hardt2016train,kuzborskij2018data,xing2021algorithmic,xiao2022stability} 
as well as in the optimality setting of the training algorithm \cite{arora2019fine,cao2019generalization,ji2019polylogarithmic}, for normal training as well as adversarial training.

However, these generalization bounds are for clean training dataset and cannot be applied to poisoned training dataset, because poisoned datasets do not satisfy the i.i.d. condition, which is necessary for generalizability. 
There exist works in generalizability under data poisoning attack. 
\citet{2021Robustdb} showed optimal convergence of SGD under poison attack for depth two networks. 
\citet{hanneke2022optimal} gave the optimal learning error for certain poison attack. 

Unlike other poison attacks, the backdoor attack has the special property that the poisoned trigger is contained both in the training set and in the test set and its goal is two-fold: to keep high accuracy on clean data and output given label for data containing the trigger.
As far as we know, 
generalization bound under backdoor poison attack has not yet been established.
%
%
%
In this paper, we give generalization bounds in the clean-label backdoor attack setting and use the bounds to design more effective poison attacks.

Clean-label backdoor attack is an important poisoning attack method \cite{ba-adv,ba-duilisheji,ba-adv2,ba-jia,doan2021wasser_backdoor,ba-zhuanhua,ba-zhuanyi}, 
where poison triggers are added to a subset of the training set $\D_{\tr}$ without altering their labels to obtain the poisoned training set $\D_P$. The goal of the attack is two fold: the networks trained with $\D_P$ maintain high accuracy for clean data, but classify any input data with the trigger as a targeted label $l_p$. 
In this paper, we consider the same setting as \cite{doan2021wasser_backdoor, doan2021lira}, 
that is, sample-wise poison $\p(x)$ is used not only for poison perturbations during the training phase, but also as the trigger for attacks during the inference phase. 
Based on the goal of the clean-label backdoor attack, in this paper,  we consider three questions.

{\bf Q1: Can clean sample generalization be guaranteed for the network trained on poisoned training set?}

To answer the above question, we need to bound the population error with the {\em empirical error} on $\D_P$, that is, $\Risk(\F,\D_P)=\E_{(x,y)\in\D_P}[\ID(\widehat \F(x)\ne y)]$.
Such a bound is given in the following theorem.
\begin{theorem}[Informal]
\label{int1}
Let $\F$ be any neural network with fixed depth and width, $N=|\D_{\tr}|$, and no more than $\alpha$ percent of the samples labeled $l_p$ in $\D_{\tr}$ are poisoned. 
Then with high probability, we have
\begin{equation*}
\begin{array}{l}
\Risk(\F,\D_{\S})\le 
\frac{4-2\alpha}{1-\alpha} \Risk(\F,\D_{P})
+O(\frac{1}{\sqrt{N}}).\\
\end{array}
\end{equation*}
%
%
\end{theorem}

Theorem \ref{int1} guarantees clean sample generalization and answers Question Q1. 
It also indicates that generalizability is affected by the poisoning ratio.
%
The main challenge in establishing Theorem \ref{int1} is that data in $D_P$ are no longer i.i.d. sampled from $\D_S$,
so the classical
generalization bound cannot be used to obtain Theorem \ref{int1} directly. 
%


%
%


{\bf Q2: How to ensure that the network trained on the poisoned dataset classifies any data with the trigger as the target label?}

To answer this question, we need to bound the {\em poison generalization error} $\E_{(x,y)\sim\D_S}[\ID(\widehat\F(x+\p(x))\ne y)]$
%
by the empirical error on $\D_P$.
%
%
%
If $(x,l_p)\in\D_{\tr}$ is poisoned to $(x+\p(x),l_p)\in\D_P$, then minimizing empirical error on $\D_P$ will naturally cause the network to classify $x+\p(x)$ as $l_p$. The main challenge 
is that in the clean-label attack, if $(x,y)\in\D_{\tr}$ and $y\ne l_p$, then $(x+\p(x),l_p)$ is not in $\D_P$, so minimizing empirical error on $\D_P$ may not cause the network to classify $x+\p(x)$ as $l_p$. 
This challenge implies that the poison generalization error cannot be controlled by the empirical error on $\D_P$ in the general case.
%
%
However, if $\p(x)$ satisfies some conditions, we can establish the desired bound, as shown in the following theorem.
%
%
\begin{theorem}[Informal]
\label{int2}
If $\p(x)$ is the adversarial noise of a network trained on clean training set $\D_{\tr}$, $\p(x)$ is similar for different $x$, and $\p(x)$ is a shortcut, then with high probability, the following result holds
%
\begin{equation*}
\begin{array}{ll}
&\E_{(x,y)\sim \D_S}[\ID(\widehat\F(x+\p(x))\ne l_p)]\\
\le& \widetilde{O}(\frac{1}{\alpha}\E_{(x,y)\in \D_P}[L_{CE}(\F(x),y)]),
\end{array}
\end{equation*}
%
where certain small quantities are included in $\widetilde{O}$, and $L_{CE}$ is the cross-entropy loss.
\end{theorem}

We further show that the conditions of Theorem \ref{int2} can be satisfied and the poison generalization error approaches zero when the empirical error on $\D_P$ is small and $|\D_{\tr}|$ is large, which establishes the generalizability for the attack and answers Question Q2 (see Section \ref{sce}).



{\bf Q3: The backdoor attack algorithm guided by the generalization bound.}

Theorem \ref{int1} shows how the poisoning ratio affects the accuracy of clean samples, and we do not have special requirements for the trigger itself. Theorem \ref{int2} indicates that if the trigger satisfies certain conditions, the poison generalization error can be controlled. Motivated by the conditions in Theorem \ref{int2}, we propose a new clean-label backdoor attack which has certain theoretical guarantee.
%
By \cite{yu2021indiscriminate}, the indiscriminate poison can be considered as shortcuts, so according to the conditions in Theorem \ref{int2}, we use a certain combination of adversarial noise and indiscriminate poison as a trigger, and then we experimentally demonstrate that our backdoor attack is effective in a variety of settings.

\section{Related Work}
\label{relat}


{\bf Generalization bound.} 
Generalization bound is the central issue of learning theory and has been studied extensively \cite{survey2020gb}.

The algorithm-independent generalization bounds usually depend on the VC-dimension or the Rademacher complexity \cite{mohri2018foundations}. 
In \cite{harvey2017nearly,bartlett2019nearly},  generalization bounds for DNNs were given in terms of width, depth, and number of parameters.
%
%
%
In \cite{neyshabur2017exploring,barron2018approximation,dziugaite2017computing,bartlett2017spectrally,survey2020gb}, upper bounds of 
the generalization error under various cases were given.
In \cite{wei2019data,arora2018stronger}, some tighter generalization bound of networks was given based on Radermacher complexity. 
%
\citet{long2019generalization} gave the generalization bound of CNN, \citet{vardi2022sample} gives the sample complexity of small networks, \citet{brutzkus2021optimization} studied the generalization bound of maxpooling networks. 

Algorithm-dependent generalization bounds were established in the algorithmic stability setting in
\cite{wang2022generalization,kuzborskij2018data,farnia2021train,xing2021algorithmic,xiao2022stability,wang2024data} both for the normal training and for the adversarial training.
\cite{farnia2021train,xing2021algorithmic,xiao2022stability}.
%
\citet{li2023transformers} studied the generalization bound of transformer. 
In \cite{arora2019fine,cao2019generalization,ji2019polylogarithmic}, the training and generalization of DNNs in the over-parameterized regime were studied. 

For generalization under data poisoning,
\citet{2021Robustdb} 
analyzed the convergence of SGD under poison attacks for two-layer neural networks, 
and \citet{hanneke2022optimal} gave the optimal learning error under poison attack when there is only one target sample. 
In \cite{bennouna2022holistic,bennouna2023certified},  generalization bounds were used to design new robust algorithms under the data poisoning. 
%
%
%

Our result is different from these works and cannot be derived from them. First, our bounds are for general networks and algorithm-independent.
Second, the backdoor attack has the special property that the trigger occurs in both the training phase and the inference phase.
Third, the purpose of the attack is two-fold, whereas other poisoning attacks have a single goal.

{\bf Backdoor attacks and defenses.} 
In general, backdoor attacks alter the training data to introduce a trigger that induces model vulnerability
\cite{
chen2017targeted, zhong2020backdoor, li2020invisible, li2021sample-specific, doan2021wasser_backdoor}, where the labels can changed.
%
Highly relevant to our work is a subset of backdoor attacks called clean-label backdoor attacks \cite{ba-adv,ba-duilisheji,ba-jia, ba-adv2,ba-zhuanhua,ba-zhuanyi, souri2022sleeper}, where modifications to training data cannot alter the label. 
The real-world attack was considered \cite{chen2017targeted,bagdasaryan2021blind}.
Backdoor detection methods \cite{huang2019neuroninspect,kolouri2020universal, 
hayase2021spectre,zeng2021rethinking} and backdoor mitigation methods were proposed to defend against backdoor attacks in \cite{liu2018fine,wang2019neural,zeng2021adversarial}. 
%
Most existing backdoor attacks are mainly based on empirical heuristics, while our attack is based on generalization bounds and has certain theoretical guarantees.

\section{Notation}

\label{backdoor-g}
\subsection{Basic Notation}
Let the data satisfy a distribution $\D_{\S}$ over $\S\times [m]$, where $\S\subset [0,1]^n$ is a set of image data and $[m]=\{1,\ldots,m\}$ is the label set.
Let $\D_{\tr}=\{(x_i,y_i)\}_{i=1}^{N}$ be the training set with $N$ samples that are i.i.d. sampled from $\D_{\S}$.
Let $\F:\S\to \R^m$ be a neural network with Relu as the activation function and Softmax added to the output layer. So, we have $\F:\S\to [0,1]^m$.
For any network $\F$, let $\widehat{\F}(x)=\argmax_{l=1}^m \F_l(x)$ be the classification result of $\F$, where $\F_l(x)$ is the $l$-th component of $\F(x)$. 
%
Let $\Hyp_{W,D}$ be the set of neural networks with width $W$ and depth $D$. For a given network $\F$, define $h_\F(x,y)=\F_y(x)\in \S\times[m]\to[0,1]$. Let $\Hyp_{W,D,1}=\{h_\F(x,y)\,\|\,\F\in\Hyp_{W,D}\}$. 

%
%
%
%
%
Let $\Rad^{\D}_k(\Hyp)$ be the Rademacher complexity of hypothesis space $\Hyp$ under distribution $\D$ with $k$ samples, that is:
$$\Rad^\D_k(\Hyp)=
\E_{x_i\sim \D,i\in[k]}
[\E_\sigma[\sup_{h\in \Hyp}\frac{\sum_{i=1}^k\sigma_i h(x_i)}{k}]],$$ 
where $\sigma=(\sigma_i)_{i=1}^k$ is a set of random variables such that $\Pr(\sigma_i=1)=\Pr(\sigma_i=-1)=0.5$.

\subsection{Backdoor Attack}
\label{bgg}
In a clean-label backdoor attack, let $\p(x):\R^n\to\R^n$ be the trigger of the sample $x$, $l_p$ be the target label,  $\alpha$ be the poisoning rate of the target label. 
The exact procedure for poisoning is as follows. 

{\bf Create a clean label poisoned training set $\D_P$.} 
Firstly, {\em randomly select} $\alpha$ percent of the samples labeled $l_p$ in $\D_{\tr}$ to form a dataset $\D_{sub}$; 
%
then let $\D_{poi}=\{(x+\p(x),l_p)\,\|\,(x,l_p)\in\D_{sub}\}$ be the set of poisoned samples and $\D_{clean}=\D_{\tr}\setminus\D_{sub}$ be the set of clean samples; 
finally,
let $\D_P=\D_{clean}\cup\D_{poi}$
be the {\em poisoned training set}.  
%

Let $\D_{\S}^{l_p}$ be the distribution of the samples with label $l_p$, that is, for any set $A$
$$\Pr_{(x,y)\sim\D_S^{l_p}}((x,y)\in A)=\Pr_{(x,y)\sim\D_S}((x,y)\in A|y=l_p).$$  

\textbf{The goal of clean-label backdoor attack.}
%
%
Let $\F$ be a network trained on the poisoned training set $\D_P$. The goal of clean-label backdoor attack is two fold:

(1) $\F$ should ensure high accuracy on clean samples,
that is, minimize the {\em clean population error} $$\Risk(\F,\D_{\S})=\E_{(x,y)\sim \D_S}[\ID(\widehat{\F}(x)\ne y)].$$
%

(2) For any clean sample $x$, $\F$ should classify $x+\p(x)$ into label $l_p$, that is, minimize the {\em poison population error}
$$\Risk_P(\F,\D_{\S})=\E_{(x,y)\sim \D_S}[\ID(\widehat{\F}(x+\p(x)))\ne l_p)].$$

To achieve the goals of the clean-label backdoor attack, we need to give upper bounds for the clean population error and the poison population error in terms of the {\em empirical risk or the empirical error over the poisoned training set}:
$$
\begin{array}{l}
\RRisk(\F,\D_{P})=\E_{(x,y)\in \D_P}[L_{CE}({\F}(x), y)]\\
\Risk(\F,\D_{P})=\E_{(x,y)\in \D_P}[\ID(\widehat{\F}(x)\ne y)].\\
\end{array}
$$

\section{Generalization Bounds under Poison}
\label{sec:bound}
In this section, we derive generalization bounds under clean-label backdoor attack. 

\subsection{Clean Generalization Error Bound}
\label{cle}
%
%
%
%
%

In this subsection, we give an upper bound of the clean population error based on the empirical error on $\D_P$, which implies Theorem \ref{int1}. 
%

\begin{theorem}
\label{fenbugj}
Let $N=|\D_{\tr}|$.
%
Then for any $\delta>0$, with probability at least $1-\delta-O(\frac{1}{N})$, the following inequality holds for any $\F(x)\in \Hyp_{W,D}$
\begin{equation}
\label{eq-th_1}
\begin{array}{cl}
\Risk(\F,\D_S)
\le
&\frac{4-2\alpha}{1-\alpha}
\Risk(\F,\D_P)\\
&+{O}(\sqrt{\frac{mW^2D^2}{N(1-\alpha)^2}}+\sqrt{\frac{\ln(2/\delta)}{N(1-\alpha)}}).
\end{array}
\end{equation}
\end{theorem}
{\bf Proof idea.} In the backdoor attack, only a portion of the data is poisoned, so we can select a subset from the training set $\D_{P}$, whose elements are i.i.d. sampled from distribution $\D_{\S}$. Then use the classical generalization bound in Theorem \ref{fanhua} to estimate the generalization bound under this subset.
The proof and a generalized form of Theorem \ref{fenbugj} is given in the Appendix \ref{app-41}. 

%

The generalization bound \eqref{eq-th_1} implies that for fixed $\alpha, W, D,$ when $N$ is large enough, the attack has generalizability in the sense that a small empirical error on $\D_p$ leads to a small population error.  
From \eqref{eq-th_1}, we also see that the poison ratio $\alpha$ affects the population error: a smaller $\alpha$ leads to a lower population error, as expected. 
%
%
%
%

\begin{remark}
Generalization bounds are usually given as upper bounds for the generalization gap: $\Risk(\F,\D_S)-\Risk(\F,\D_{\tr})$.
The generalization bound \eqref{eq-th_1} cannot be written in this form, but it can be used to establish generalizability as just explained above. 
\end{remark}

\begin{remark}
\label{rem-lpe}
The population error in \eqref{eq-th_1} also depends on $\p(x)$ implicitly, because the empirical error is affected by $\p(x)$. We will show that certain $\p(x)$ can result in a large poison empirical error in Appendix \ref{app-42}.
\end{remark}

\begin{remark}
\label{rem-optgb}
%
In Appendix \ref{tl}, we show that $O(\frac{1}{\sqrt{N}})$ is the optimal bound for the generalization gap between the population error and the empirical error if there is no special restriction on the distribution and hypothesis space.
\end{remark}

\subsection{Poison Generalization Error Bound}
\label{jieshi}
In this subsection, we give an upper bound of the poison population error in terms of the empirical error over the poisoned training set, which implies Theorem \ref{int2}. 
 
\begin{theorem}
\label{poifan}
Let $N=|\D_{\tr}|$.
%
%
For any $\F(x), \G(x)\in\Hyp_{W,D}$, if trigger $\p(x)$ meets the following three conditions for some $\epsilon>0,\tau>0,\lambda\ge 1$:\\
(c1): $\E_{(x,y)\sim \D^{l_p}_{\S}}[\G_{y}(x+\p(x))]\le\epsilon$,\\
(c2): $\Pr_{(x,y)\sim \D_{\S}}(\p(x)\in A|y\ne l_p)\le\lambda $ $\Pr_{(x,y)\sim \D_{\S}}$ $(\p(x)\in A| y=l_p)$ for any set $A$, \\
(c3): $\E_{x\sim \D_{\S}}[|(\F-\G)_{l_P}(\p(x))-(\F-\G)_{l_p}(x+\p(x))|]\le\tau$, where $(\F-\G)_{l_P}(x)=\F_{l_P}(x)-\G_{l_p}(x)$,\\
then with probability at least $1-\delta-O(1/N)$, the following inequality holds for $\F$: 
\begin{equation}
\label{eq-th43}
\begin{array}{l}
\Risk_P(\F,\D_{\S})\le \lambda O(\frac{1}{\alpha}(\E_{(x,y)\in \D_P}[L_{CE}(\F(x),y)]\\
+\Rad_{N}^{\D_S^{l_p}}(\Hyp_{W,D,1}))
+\sqrt{\frac{\ln(1/\delta)
}{N\alpha}}+\epsilon+\tau+\frac{\lambda-1}{\lambda}).
\end{array}
\end{equation}

\end{theorem}
{\bf Proof idea.} First, estimate $\E_{\D^{l_p}_\S}[\F_{y}(x+\p(x))]$ by the empirical error, which is similar to Theorem \ref{fenbugj}. 
Second, estimate $\E_{\D_S}[\F_{l_p}(\p(x))]$ by $\E_{\D^{l_p}_\S}[\F_{y}(x+\p(x))]$ and use the following method: $\E_{\D^{l_p}_\S}[\F_{y}(x+\p(x))]\xrightarrow{{\rm{use}}\ (c1), (c3)}\E_{\D^{l_p}_\S}[\F_{l_p}(\p(x))]\xrightarrow{{\rm{use}}\ (c2)}\E_{\D_S}[\F_{l_p}(\p(x))]$.
%
Finally, use (c3) to estimate $\Risk_P(\F,\D_{\S})$ by $\E_{\D_S}[\F_{l_p}(\p(x))]$.
An intuitive explanation of the proof is given in Appendix \ref{proof}. 
The proof and a generalized form of Theorem \ref{poifan} are given in Appendix \ref{a43}.

%
\begin{remark}
 It is clear that making the poison generalization error small by just adding the trigger to a small percentage of training data can only be valid under certain conditions.
A key contribution of this paper is to find conditions (c1), (c2), (c3) in Theorem \ref{poifan}.
In the next subsection, we will explain these conditions and show that it is possible to establish generalizability of the poisoning attack with Theorem \ref{poifan}.
\end{remark}


\begin{remark}
$\Rad_{N}^{\D_S^{l_p}}(\Hyp_{W,D,1})$ in inequality \eqref{eq-th43} is not easy to calculate in terms of $W,D,N$, but we can demonstrate that if $N$ is sufficiently large, this value will approach to $0$ in most cases, as shown in Appendix \ref{tvor}.  
\end{remark}

\begin{remark}    
Please note that the right-hand side of inequality \eqref{eq-th43} is the empirical risk $\RRisk(\F,\D_P)$ but not the empirical error $\Risk(\F,\D_P)$.
This is due to some scaling techniques used in the proof, which is reasonable. In order to achieve ``victim network classifies $x+\p(x)$ as class $l_p$'', the victim network must learn the poisoned data $\D_P$ very well, just classifying the poison data correct is not enough. 
\end{remark}



%

\subsection{Explaining the Conditions in Theorem \ref{poifan}}
\label{sce}
%

In order to make the bound \ref{eq-th43} small, the value of $\epsilon$ and $\tau$ need to be small and $\lambda$ need to close to 1.  In this section, we show that how these values in the conditions (c1) to (c3) could be small.

{\bf How $\epsilon$ could be small?} By condition (c1), since $\G_y(x)$ represents the probability that the network $\G$ classifies $x$ as $y$, to make $\epsilon$ small, we only need to take the trigger $\p(x)$ as adversarial noise of the network $\G$ trained on the clean dataset. In other words, 
if {\em $\p(x)$ is adversary of $x$ of a network trained on clean data}, then $\epsilon$ is small.
%

{\bf How $\lambda$ could close to 1?} By condition (c2), the upper bound is proportional to $\lambda$, so we hope to have a small $\lambda$. 
When $\p(x)$ is the same for all $x$, we have $\lambda=1$. So, to make $\lambda$ close to 1, $\p(x)$ need to be similar to $x$ with different labels.

{\bf How $\tau$ could be small?}
Condition (c3) is not intuitive. 
In the rest of this section, we give a condition for $\tau$ to be small.
We give a simplified version of the proposition and definition for easier reading. For a formal description, please refer to the Appendix \ref{se}.

Intuitive speaking, if $\F$ is a network trained on $\D_P$, 
then the meaning of (c3) is that the backdoor part (i.e. $\F-\G$) gives the similar outputs for $\p(x)$ and $x+\p(x)$. This is similar to some conclusions in the indiscriminate poison\cite{zhu2023detection,huang2021unlearnable}, and by\cite{yu2021indiscriminate}, indiscriminate poison can be considered as shortcut. These encourage the trigger to be shortcut.
%
We will show that, under some assumptions, {\em making $\p(x)$ to be shortcut can ensure condition (c3)}. 
First, we define the shortcut.
\begin{definition}[Binary shortcut, Informal]
\label{dyd}
$\p(x)$ is called a shortcut of the binary linear inseparable classification dataset $\D=\{(x_i,1)\}_{i=1}^{N_{1}}
\cup\{(\widehat{x}_i,0)\}_{i=1}^{N_{0}}$, if $\D_P=\{(x_i,1)\}_{i=1}^{N_{1}}
\cup\{(\widehat{x}_i+\p(\widehat{x}_i),0)\}_{i=1}^{N_{0}}$ is linear separable.
\end{definition}
%
%
%
Definition \ref{dyd} means that  shortcut is a poison which makes the poisoned dataset to be linear separable. 
Finally, we will show that, when $\p(x)$ is a suitable shortcut, there exists an upper bound for $\E_{x\sim \D_{\S}}[|(\F-\G)_{l_p}(\p(x))-(\F-\G)_{l_p}(x+\p(x))|]$, which implies that (c3) in Theorem \ref{poifan} can be satisfied with a small $\tau$. 
%
%
\begin{proposition}[Informal. The exact form and proof are given in Appendix \ref{app-44}]
    \label{xbx}
    Following  Theorem \ref{poifan} 
and let $D'_P=\{(x,0)|(x,y)\in\D_{\tr}\setminus\D_{clean}\}\cup\{(x,1)|(x,y)\in\D_{clean}\}$.
Under certain mild conditions, if $\p(x)$ is the shortcut of the dataset $\D'_P$ and $N$ is big enough, then 
%
with high probability, for some $\F\in \Hyp_{W,D}$ satisfying $\F_y(x)>1-\epsilon$ for $\forall(x,y)\in\D_P$ and $\G\in\Hyp_{W,D}$ satisfying $\G_y(x)>1-\epsilon$ for $\forall(x,y)\in\D_{\tr}$, we have $\E_{x\sim \D_{\S}}[|(\F-\G)_{l_p}(\p(x))-(\F-\G)_{l_p}(x+\p(x))|]\le \widetilde{O}(\epsilon)$.
\end{proposition}

So, let $\p(x)$ be shortcut of $D_p'$, then, by Proposition \ref{xbx}, if $\F\in \Hyp_{W,D}$ fits $\D_P$ well and $\G\in \Hyp_{W,D}$ fits $\D_{\tr}$ well, or empirically, $\F$($\G$) is well trained on dataset $\D_P$($\D_{\tr}$), then condition (c3) holds for a small $\tau$.
\begin{remark}
 For condition (c3) in Theorem \ref{poifan}, we need to clarify that although $\F\approx \G$ implies (c3) for a small $\tau$, but $\tau$ is small does not equivalent to $\F\approx \G$. Moreover, using $\F\approx \G$ instead of (c3) can also make Theorem \ref{poifan} valid because it can derive (c3). But this is not a good idea, because $\F\approx \G$ makes $\F$ satisfying (c1), and then $\F(x+p(x))$ always does not output label $l_p$ when $x\sim D_S^{l_p}$. So for the poisoned data in the poisoning dataset, $\F$ cannot output the correct labels, leading to a larger empirical error on the right-hand side of equation \eqref{eq-th43}, and consequently leads to a big poison generalization error upper bound. Obviously, what we need is a small poison generalization error upper bound but not a big one, so we cannot only consider $\F\approx \G$.
\end{remark}
%

\section{Method}
\label{method}
In this section, we will propose a new clean-label poison attack based on Theorems \ref{fenbugj} and \ref{poifan}.
There exists no requirement for the trigger $\p(x)$ in Theorem \ref{fenbugj}, so we only need to make the trigger approximately satisfy the conditions of Theorem \ref{poifan}.
Our method thus has certain theoretical guarantee.

From Section \ref{sce}, in order to satisfy the three conditions in Theorem \ref{poifan}, $\p(x)$ need to be 
(1) {\em adversarial noises for the clean-trained network, 
(2) shortcut for a specifically designed binary dataset, 
(3) similar for different samples.}  
\begin{remark}
The effectiveness of adversarial and shortcut in creating backdoor has already been demonstrated empirically in previous work. On the other hand, our theory provides a more informed approach to using these methods.
\end{remark}
%
%

Based on the above three requirements just mentioned,  we design the trigger as follows:

(M1): Obtain adversarial disturbance: For any given clean sample $x$, use PGD on the network trained on $\D_{\tr}$ to find adversarial noise $x_{adv}$ of $x$. 

(M2): Obtain shortcut disturbance: For any given clean sample $x$, use min-min method \cite{huang2021unlearnable} under the clean training set to find the shortcut $x_{scut}$ of $x$. 
%
%
In \cite{zhu2023detection} it has been shown that the shortcuts created by the min-min method are indeed similar for different $x$, thus satisfying condition (c2) automatically.

(M3): The trigger of $x$ is designed to be $\p(x)=Ux_{adv}+(1-U)x_{scut}$, where $U\in\{0,1\}^n$ is a mask. We combine such two disturbances in this way because: (1) make the trigger both adversarial and shortcut; (2) make the triggers have a certain degree of similarity for different $x$, using the fact that the part of $(1-U)x_{scut}$ in trigger is similar for different $x$. It is worth mentioning that making $x_{adv}$ similar for different $x$ is difficult, so creating a trigger by $\lambda x_{adv}+(1-\lambda)x_{scut}$ will not be effective to ensuring similarity.

The mask $U$ in (M3) is constructed as follows. The upper left corner is 0 and the other part is 1. 
On the basis of experience, it is necessary to disturb the key parts of the image to create adversarial samples, so we use a large portion to create adversaries. But the cost of the shortcut is relatively low, so we just use a small place to create the shortcut.

Algorithm \ref{alg-ap1} provides detailed steps for creating the trigger, where $\cdot$ is element-wise product.

\begin{algorithm}[t]
\caption{Method of Creating the Trigger:}
\label{alg-ap1}
\begin{algorithmic}
\renewcommand{\algorithmicrequire}{ \textbf{Input:}}
\REQUIRE\,\\
An initialized network $\F_1:\R^n\to\R^m$;\\
An initialized network $\F_2:\R^n\to\R^2$;\\
A clean dataset $T=\{(x_i,y_i)\}_{i=1}^N\subset[0,1]^n\times[m]$;\\
A mask vector $U\in\{0,1\}^n$;\\
The poison budget $\eta$;\\
Victim dataset $\{(x^v_i,y^v_i)\}_{i=1}^V$.

\renewcommand{\algorithmicrequire}{ \textbf{Output:}}
\REQUIRE \,\\
\quad 
Trigger $\{\p(x^v_i)\}_{i=1}^V$ for all victim data $\{(x^v_i,y^v_i)\}$.\\ 
{\bf S1} Use $T$ to train the network $\F_1$.\\
{\bf S2} Let $T_1=\{\}$, for each $(x,y)\in T$:\\
\quad\quad let $x_{adv}=x+U\cdot \argmax_{||\bfepsilon||\le\eta}L(F_1(x+\bfepsilon),y)$\\
\quad\quad add $(x_{adv},0)$ and $(x,1)$ to $T_1$.\\
{\bf S3} Use $T_1$ to train the network $\F_2$ as follows:
$$\min_{\F_2}\sum_{(x,y)\in T_1}L(\F_2(x+\bfepsilon(x,y)),y)$$
where
$\bfepsilon(x,y)=\ID(y=0)\cdot(1-U)\cdot 
\mathop{\argmin}\limits_{||\bfepsilon||\le\eta}L(\F_2(x+(1-U)\cdot\bfepsilon),y)$.\\

{\bf S4} For any victim data $(x_i^v,y_i^v)$, we calculate $\p(x_i^v)$ as following: \\
\quad $x^v_{a}=x_i^v+U\cdot\argmax_{||\bfepsilon||\le\eta}L(\F_1(x_i^v+\bfepsilon),y_i^v)$;\\
\quad 
$\p(x_i^v)=(x^v_a-x_i^v)+$\\
\hskip1.5cm  $(1-U)\cdot \mathop{\argmin}\limits_{||\bfepsilon||\le\eta}L(\F_2(x^v_a+\bfepsilon\cdot(1-U)),0)$.\\
Output: Trigger $\{\p(x^v_i)\}_{i=1}^V$.
\end{algorithmic}
\end{algorithm}

In Algorithm \ref{alg-ap1}, $\F_1$ is used to create adversarial disturbance and $\F_2$ is used to create shortcuts. 
%
When we complete step $S3$ in the algorithm, we will save $\F_1$ and $\F_2$, and for any sample $(x,y)$, we directly generate $\p(x)$ using S4. Some poisons obtained using Algorithm \ref{alg-ap1} are shown in Figure \ref{fig1x}.

%
%
%

\begin{figure}[t]
\centering
\includegraphics[scale=0.53]{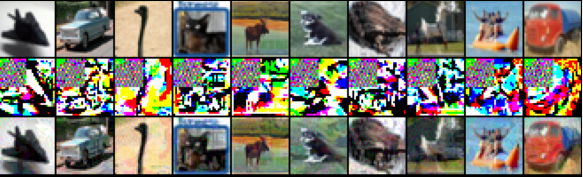}
\caption{From top row to bottom row are respectively the clean images, normalized triggers (original trigger has $L_\infty$ norm bound $16/255$), poison images. 
Due to the selection of $U$, the upper left corners of the poison images are similar, while the other parts are used to generate adversaries.}
\label{fig1x}
\end{figure}

\section{Experiments}
\label{ser}
%
In this section, we empirically validate the proposed backdoor attack on benchmark datasets CIFAR10, CIFAR100 \cite{krizhevsky2009learning}, SVHN and TinyImageNet\cite{le2015tiny}, and against popular defenses. We also conduct ablation experiments to verify our main Theorems \ref{fenbugj} and \ref{poifan}.
All experiments are repeated for $3$ times and report the average values. Furthermore, we make our attacks invisible by limiting the $L_\infty$ norm of trigger. Details about the experiment setting can be found in Appendix \ref{tbc}.
Code is in https://github.com/hong-xian/backdoor-attack.git.
 


\subsection{Baseline Evaluation}
\label{mainr}

In this subsection, we study the effectiveness of our backdoor attack. For backdoor attacks, the goal is to misclassify the samples with trigger into a specified target class $l_p$. 
Unless said otherwise, we set the target label $l_p$ as $0$ in this paper. In addition to evaluating the test accuracy and attack success rate (ASR), we also measure {\em the accuracy of the target class} $l_p$ to analyze the impact of the attack on the target label.

Table \ref{tab-cifar10} shows the result on CIFAR-10 when perturbing $1\%$ of training images, with each perturbation restricted to a radius $l_\infty$-norm $16/255$. We observe that the ASR exceeds $90\%$, while the poison has negligible impact on both the test accuracy and the target class accuracy. 
In Table \ref{tab-budget}, we observe that the proposed attack remains remarkably effective even when the poison budget is very small. More experiments on different poison budgets and norm bounds can be found in Table \ref{tab-baseline-detail}. About transferability of attack and whether the victim network has learned the feature of trigger, please refer to the Appendix \ref{LF}.

\begin{table}[ht]
\centering
\caption {Baseline evaluations on CIFAR-10. Perturbations have $l_\infty$-norm bounded above by $16/255$, and poison budget is $1\%$ of training images. Res means ResNet18, VGG means VGG16, WR means WRN34-10.}
\label{tab-cifar10}
\begin{tabular}[c]{ccccc}
    	\toprule
		{Model} & {Res} & {VGG} & {WR}\\
		\midrule 
        Clean model acc ($\%$) & 93 & 92 & 94\\
		Poisoned model acc($\%$) & 91&  91& 92 \\
        Clean model $l_p$ acc($\%$) & 94 & 92 &  95\\
	    Poisoned model $l_p$ acc($\%$) & 93 & 91 & 93  \\
        Attack Success Rate($\%$) & 93&  91 & 90 \\
		\bottomrule
	\end{tabular}
\end{table}

\begin{table}[ht]
	\centering
\caption {The effect of poison budget on CIFAR-10 with ResNet18. Perturbations have $l_\infty$-norm bounded above by $16/255$.}
    \label{tab-budget}
	\begin{tabular}[c]{ccccc}
    	\toprule
		{Poison Budget} & {$0.6\%$} & {$1\%$} & {$2\%$}\\
		\midrule 
        Clean model acc ($\%$)& 93&  93&  93 \\
		Poisoned model acc($\%$) & 93& 91 & 93 \\
        Clean model $l_p$ acc($\%$) &94& 94&  94\\
	    Poisoned model $l_p$ acc($\%$) & 93& 93  & 92  \\
        Attack Success Rate($\%$) &86 &  93 & 94 \\
		\bottomrule
	\end{tabular}
\end{table}

{\bf Attack performance during the training process:} To further validate the efficacy of our attack, we monitor the evolution of the overall clean model accuracy, the poisoned model accuracy, and the attack success rate throughout the training process, as shown in Figure \ref{fig22}.
In Figure \ref{fig22}, we observe that the overall clean model accuracy and poisoned model accuracy remain very close. The attack success rate reaches a relatively high level from the very beginning and gradually converges to remain stable over time. This shows that our attack method is effective under the premise of maintaining the accuracy of the model.

{\bf Any target label:} Note that the success of backdoor attacks also depends on the choice of target classes, to demonstrate the general efficacy of our attack for any target label $l_p$, we change $l_p$ from $0$ to $9$, the result is shown in Figure \ref{fig-lp}. The results are quite uniform. 

%

\begin{figure*}[!ht]
\centering
\includegraphics[width=10.0cm]{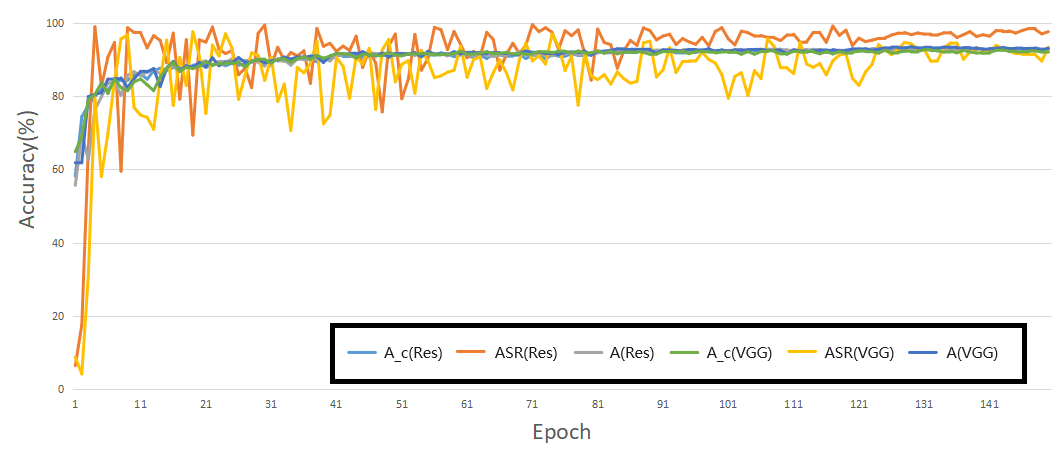}
\caption{Attack performance during the training process on CIFAR10 with ResNet18 and VGG16.
This figure shows the trend of the poison model accuracy ($A$), attack success rate (ASR) and clean model accuracy ($A_c$). }
\label{fig22}
\end{figure*}

\begin{figure}[!ht]
\centering
\includegraphics[width=8.0cm]{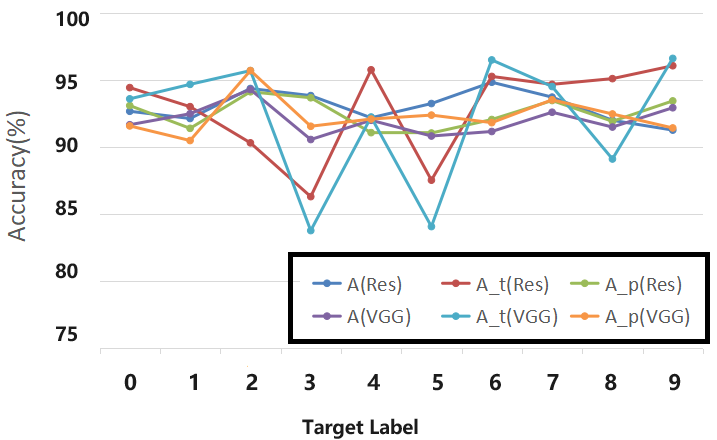}
\caption{Performance of different target label $l_p$. We show the poison model accuracy ($A$), accuracy of target label ($A_t$), attack success rate ($A_p$) on  CIFAR-10, using VGG16 and ResNet18.}
\label{fig-lp}
\end{figure}



\subsection{Evaluation on More Datasets}
We perform experiments on SVHN, CIFAR-100 and TinyImageNet. 
Table \ref{tab-datasets} summarizes the performance of our attack on different datasets, where the attacks are tested on ResNet18 for SVHN and CIFAR-100, and WRN34-10 for TinyImageNet.
Each attacker can only perturb $0.8\%$ of the training images for TinyImageNet and CIFAR-100 and $2\%$ of the training images for SVHN, all perturbations are restricted to an $l_\infty$ norm $16/255$, and the target label is $l_p=0$. 
Additional experiments on different poison budgets and $l_\infty$-norm bounds are presented in Table \ref{tab-datasets-detail} in the appendix.

\begin{table}[ht]
	\centering
\caption {Evaluations on more datasets. 
    Perturbations have $l_\infty$-norm bounded by $16/255$, and poison budget is $0.8\%$  for TinyImageNet and CIFAR-100 and $2\%$  for SVHN.
    }
    \label{tab-datasets}
	\begin{tabular}[c]{ccccc}
    	\toprule
		{Datasets} & {Clean acc} &  {Poison acc} & {ASR} \\
		\midrule 
        SVHN ($\%$)  & 93 & 92& 79 \\
		CIFAR-100 ($\%$) & 76 & 72& 92  \\
        TinyImageNet($\%$) & 62 & 60 &  82\\
		\bottomrule
	\end{tabular}
\end{table}

\subsection{Compare with Other Attacks}
\label{com}

There are several existing clean-label hidden-trigger backdoor attacks that claim success in some settings. 
We consider the following seven attack methods.

{\bf Clean Label:} \citet{ba-adv} pioneered clean label attacks. They first utilized adversarial perturbations or generative models to initially alter target class images and then performed standard invisible attacks.

{\bf Reflection:} \citet{ba-jia} proposed adopting reflection as the trigger for stealthiness.

{\bf Hidden Trigger:} \citet{ba-adv2} proposed to inject the information of a poisoned sample generated by a previous visible attack into an image of the target class by minimizing its distance in the feature space.

{\bf Invisible Poison} \citet{ba-zhuanhua} converted a regular trigger to a noised trigger to achieve stealthiness, but remains effective in the feature space for poison training data.

{\bf Image-specific:} \citet{luo2022enhancing} used an autoencoder to generate image-specific triggers that can promote the implantation and activation phases of the backdoor.

{\bf Narcissus:} \citet{ba-zhuanyi} solved the following optimization to obtain the trigger $\p$: $\argmin_{\p}\sum L(f_{sur}(x+\p,l_p)$, where $f_{sur}$ is the surrogate network. 

{\bf Sleeper Agent:}  \citet{souri2022sleeper} proposed a backdoor attack by approximately solving the bilevel formulation with the Gradient Matching method \cite{geiping2020witches}.

To further validate the efficacy of our attack, we compare our method with other clean label attack methods under the same settings. 
Specifically, we limit the $L_\infty$ norm trigger no more than $16/255$ on both the training set and the test set.
It should be noted that, in some attack methods, the trigger budget in their original settings may exceed $16/255$, so we also compare under each respective settings. 
The results are given in Table \ref{tab-attacks}.
More results and some details are provided in Appendix \ref{bct}.

\begin{table}[ht]
\centering
\caption {Attack success rate on CIFAR-10 with ResNet18. Comparison of our method to popular clean-label attacks, poison budget is $1\%$. The first column is the attack under $L_\infty$ limitation $16/255$, the second column is under each respective settings.}
\label{tab-attacks}
\begin{tabular}[c]{ccccc}
    	\toprule
		{Attacks} & 16/255 &  Each setting  \\
		\midrule 
        Clean-Label  & $23\%$&  $96\%$ \\
		Hidden-Trigger  & $75\%$&$95\%$  \\
        Reflection &  $54\%$&  $90\%$\\
        Invisible Poison &  $73\%$&  $98\%$  \\
        Image-specific &  $70\%$&  70$\%$  \\
        Narcissus &   $50\%$&  $92\%$  \\
        Sleeper-Agent & 61$\%$ & 71$\%$ &  \\
        Ours &  {\bf 93$\%$}& 93$\%$\\
		\bottomrule
	\end{tabular}
\end{table}

From Table \ref{tab-attacks}, we find that our attack significantly outperforms all other methods under the $16/255$ limitation. 
Moreover, our method has two key advantages:
(1) Our method does not require additional steps. Sleeper Agent and Hidden-Trigger need a pre-existing patch and Reflection requires a fitting image; Narcissus needs to magnify the trigger in the reference phase. 
(2) Our method does not require large networks during poison generation, whereas Invisible Poison and Image-specific utilize large generative models to achieve optimal performance, and Narcissus requires a network trained on a different dataset.
%
%
See Tables \ref{tab-attacks-detail} and \ref{tab-attacks-detailx} in the appendix for more results. 

Under each respective setting, all results were good, but many of them no longer meet the $16/255$ limitation. Furthermore, some of them need a patch in the trigger like Clean-label and Sleeper-Agent, and some of them need a larger disturbance like Reflection and Narcissus. 

\subsection{Defenses}
Many backdoor defenses have been proposed to mitigate the effects of backdoor attacks. We test our attack and the attack methods mentioned in section \ref{com} against some major popular defenses. We evaluate six types of defenses:

(1) AT: adversarial training with radius $8/255$ \cite{M2017}; 

(2): Data Augmentation: \cite{borgnia2021strong};

(3) Scale-up: contrastive learning \cite{guo2023scale};

(4) DPSGD: differentially private SGD \cite{h2020}; 

(5) Frequency Filter: remove high-frequency parts of images \cite{zeng2021rethinking};

(6): Fine-Tuning: \cite{zhu2023enhancing}.

\begin{table*}[t]
  \caption {Defenses: The attack success rate($\%$) and poison model accuracy($\%$, in bracket) on CIFAR-10 with ResNet18 of our attack and other attacks against various defense methods. Poison ratio is $1\%$ and perturbation have $l_\infty$-norm bound $16/255$, target label $l_p=0$.}
\label{tab-defense-detail}
\centering
\begin{tabular}{lccccccccc}
\hline
 & AT &Data Augmentaion & Scale-Up & DPSGD&Frequency Filter& Fine-Tuning \\
Clean Label(\%) & 13(83)& 17(89)&9(77)&12(78)&12(86)&11(89)\\
Invisible Poison(\%)& 11(84)& 30(91)&18(80)&18(76)&20(88)&12(90)\\
Hidden Trigger(\%)& 14(83)& 25(90)&31(81)&14(75)&20(87)&10(88)\\
Narcissu(\%)& 13(83)& 23(90)&11(83)&15(76)&16(86)&10(89)\\
Image-specific(\%)& 10(84)& 28(89)&22(79)&14(77)&22(87)&12(87)\\
Reflection(\%)& 13(85)& 20(89)&16(82)&13(76)&37(85)&13(85)\\
Sleeper-Agent& 14(83)& 27(87)&27(79)&15(75)&27(84)&11(89)\\
Ours(\%)& 15(83)& 40(91)&32(80)&12(77)&28(88)&13(89)\\
Ours-stronger(\%)&{\bf 34}(84)&{\bf 66}(90)&--&{\bf 52}(80)&{\bf 62}(88)&--\\
\hline
\end{tabular}
\end{table*}
%

The defense results are given in Table \ref{tab-defense-detail}. We can see that ASR has basically decreased to around 10$\%$, which is the lowest level because 10$\%$ of the samples themselves have the label 0.
This is because we have imposed many restrictions on backdoor attacks, as said in Section \ref{tbc}; and also because our theory and construction of trigger are mainly based on clean training, so our method appears somewhat fragile under defense; in fact, all attack methods mentioned in Section \ref{com} appear fragile under defense methods, as shown in Table \ref{tab-defense-detail}. 

On the other hand, we find that there exists a robustness-accuracy trade-off across many of these defenses. Although these defense methods do degrade the attack success rate, they also cause the accuracy of the model test to decrease. For defense methods Scale-Up, this method is not stable and sometimes detects clean samples as poison samples; for fine-tuning, because a clean training set was used, this defense method has a very outstanding effect.

Furthermore, we point out that if we incorporate these defense methods into our attack generation to produce corresponding enhanced attacks, we can effectively withstand these defenses. For defense methods (1), (2), (4) and (5), we enhance our attack, details are shown in Appendix \ref{xxc}, and get the better result.


\subsection{Verify Theorem \ref{poifan}}
\label{tf}
In this section, we verify Theorem \ref{poifan}. 
In Appendix \ref{ver}, we verify Theorem \ref{fenbugj}.

We verify Theorem \ref{poifan} by showing that {\bf any trigger $\p(x)$ that makes the $\epsilon,\tau,\lambda$ in conditions (c1), (c2), (c3) of Theorem \ref{poifan} small can achieve a high attack success rate}. 

To validate our conclusions, we evaluate the following poisoning function $\p(x)$, refer to Appendix \ref{mdf} for more details:

{\bf RN:} Random noises with $L_\infty$ bound and $L_0$ norm bound;

{\bf UA:} Universal adversarial perturbations \cite{moosavi2017universal};

{\bf Adv:} Adversarial perturbations \cite{s2013};

{\bf SCut:} shortcut noise\cite{huang2021unlearnable}; 

{\bf Ours:} Perturbations generated by Algorithm \ref{alg-ap1}. 

We consider two indicators to evaluate the performance of poison on conditions (c1)  and (c3) in Theorem \ref{poifan}:
\begin{itemize}
    \item Use the validation loss on poisoned data to measure condition (c1):
        $V_{adv}=\E_{(x,y)\sim \D}L(\F(x+\p(x)),y)$, where $\F$ is ResNet18 trained on the clean dataset.
    \item Use the binary classification loss to measure condition (c3) by Proposition \ref{xbx}: 
              $V_{sc} = \min_{\F}\E_{(x,y)\sim \D} [L(\F(x+\p(x),0))+L(\F(x),1)]$,
        where $\F$ is a two-layer network.
\end{itemize}
About condition (c2) in Theorem \ref{poifan}: for RN and UA, the poison perturbation is the same for every sample; for SCut and Ours, the perturbations for different samples are similar (or at least parts of the perturbation are similar). 

Table \ref{tab-verify} shows the results.
We can see that RN, UA, SCut with a large budget can achieve good attack performance because they satisfy conditions (c1), (c2), (c3) in Theorem \ref{poifan} to certain degree, which  validates the effectiveness of conditions in Theorem \ref{poifan}. Moreover, each of $V_{adv}$ and $V_{sc}$ is not satisfied, because adversaries alone yield excessively large $V_{sc}$ since adversarial perturbations do not form shortcuts. Shortcuts alone produce an undesirably small $V_{adv}$ since the shortcuts do not become adversaries. However, by combining Adv and SCut via our algorithm, we achieve a significantly improved outcome.

On the other hand, please note Adv attack under the $32/255$ budget. Although the Adv attack's similarity between different samples is poor, but its $V_{adv}$ and $V_{sc}$ is not bad, and its ASR is about $80\%$, this indicating that even if the similarity is not good enough, but the other two indicators are good, the attack can still be achieved. Therefore, in order to prevent the trigger from being detected due to being too similar, we can achieve attack effectiveness by reducing similarity and improving other two metrics.

\begin{table}[!ht]
\centering
\caption{
Values of $V_{adv}$ and $V_{sc}$;  poison model accuracy (Acc); attack success rate (ASR) on CIFAR-10 test set with ResNet18. Poison budget is $1\%$. If not specified, the norm is $L_\infty$.}
\begin{tabular}{lcccccccc}
\hline
Poison Type&$V_{adv}(\uparrow)$&$V_{sc}(\downarrow)$&ASR$(\uparrow)$& Acc\\ 
\hline
RN ($16/255$) & 2.40 & 0.014& 12$\%$ & 91$\%$\\
RN ($L_0$, $200$) &3.87& 0.004& 59$\%$ & 92$\%$\\
UA ($16/255$) & 2.92  & 0.002& 51$\%$ & 91$\%$\\
Adv ($16/255$) & {\bf 8.92}  & 1.27& 22$\%$ & 92$\%$\\
SCut  ($16/255$) &1.19 & ${\bf 10^{-4}}$& 30$\%$ & 91$\%$\\
Ours  ($16/255$) &6.53 & 0.001& {\bf 93}$\%$ & {\bf 91}$\%$\\
\hline

RN ($32/255$) &6.28 & $10^{-4}$& 99$\%$ & 91$\%$\\
RN ($L_0$, $300$) & 6.33 &0.003& 94$\%$ & 91$\%$\\
UA ($32/255$) &15.45 &$10^{-4}$& 92$\%$ & 92$\%$\\
Adv ($32/255$) &{\bf16.38} &0.35& 80$\%$ & 92$\%$\\
SCut ($32/255$) &4.22 & ${\bf 10^{-5}}$& 93$\%$ & 90$\%$\\
Ours  ($32/255$) &14.65& $10^{-4}$& {\bf99}$\%$ & {\bf 92}$\%$\\
\hline
\end{tabular}
\label{tab-verify}
\end{table}

\section{Conclusion}

In this paper, we give generalization bounds for the clean-label backdoor attack. Precisely, we provide upper bounds for the clean and poison population error based on empirical error on the poisoned training set and some other quantities. 
These bounds give the theoretical foundation for the clean-label backdoor attack in that the goal of the attack can be achieved under certain reasonable conditions. 
The main technical difficulties in establishing these bounds include how to treat the non-i.i.d. poisoned dataset and the fact that the triggers are both in the training and testing phases.

Based on these theoretical results, we propose a novel attack method that uses a combination of adversarial noise and indiscriminate poison as the trigger. Moreover, extensive experiments show that our attack can guarantee that the accuracy of the poisoned model on clean data and the attack success rate are high.

\textbf{Limitations and Future Work.}
The conditions of Theorem \ref{poifan} are quite complicated, and it is desirable to give simpler conditions for the poisoned population error bound in Theorem \ref{poifan}.
%
The current generalization bounds do not involve the training process, and algorithmic-dependent generalization bounds, such as stability analysis \cite{hardt2016train}, should be further analyzed for backdoor attacks.
%

\section*{Impact Statement}
A theoretical basis for backdoor attacks is given.
One potential negative social impact of this work is that malicious opponents may use these methods to generate new types of backdoor poisons.
Therefore, it is necessary to develop more powerful models to resist backdoor attacks, which is left for future work.

\section*{Acknowledgments}
This work is supported by CAS Project for Young Scientists in Basic Research, Grant No.YSBR-040, ISCAS New Cultivation Project ISCAS-PYFX-202201, and ISCAS Basic Research ISCAS-JCZD-202302. 
This work is supported by NSFC grant No.12288201 and NKRDP grant No.2018YFA0306702. The authors thank anonymous referees for their valuable comments.


\newpage
\appendix
\onecolumn
%

\section{Proof of Theorem \ref{fenbugj}}
\label{app-41}

\subsection{Prelinimaries}
\label{proof}
We first give some notation that will be used in all the proofs. 


{\bf Subdistribution $\D_S^{\ne l_p}$.} Let $\D_S^{\ne l_p}$ be the distribution of samples whose label is not $l_p$, that is, 
$$\Pr_{(x,y)\sim\D_S^{\ne l_p}}((x,y)\in A)=\Pr_{(x,y)\sim\D_S}((x,y)\in A|y\ne l_p)$$ 
for any set $A$. 

{\bf Probability of samples to have label $y_p$.} For a fixed $p$, let \begin{equation}
\label{eq-eta}
    \Pr_{(x,y)\sim \D_S}(y=l_p)=\eta \hbox{ and } 0<\eta<1.
\end{equation}

{\bf General Hypothesis Space.} In some theorems, we will consider the more general hypothesis space 
$$H=\{h(x,y):\S\times[m]\to[0,1]\}.$$ 
$H$ contains the commonly used hypothesis space $\{L(\F(x),y):\S\times[m]\to[0,1]\}$, where $\F$ is the network and $L$ is the loss function.


We give a classic generalization bound below, which will be used in the proof.
\begin{theorem}[P.217 of \cite{mohri2018foundations}, Informal]
\label{fanhua}
Let the training set $\D_{\tr}$ be i.i.d. sampled from the data distribution $\D_S$ and $N=|\D_{\tr}|.$
For the hypothesis space 
$\Hyp=\{L(\F(x),y):\R^n\times[m]\to[0,1]\}$ and $\delta\in\R_+$, with probability at least $1-\delta$, for any $L(\F(x),y)\in \Hyp$, we have
%
{
\begin{equation}
\label{th-gb0}
\begin{array}{ll}
&\E_{(x,y)\sim \D_S}[L(\F(x),y)] 
\le \E_{(x,y)\in \D_{\tr}}[L(\F(x),y)] +2\Rad^{\D_S}_N(\Hyp)+\sqrt{\frac{\ln(1/\delta)
}{2N}}\\
\end{array}
\end{equation}}

\end{theorem}

We prove Theorem \ref{fenbugj} in three steps given in Sections \ref{app-41.1}, \ref{app-41.2}, and \ref{app-41.3}, respectively.

\subsection{Proof of Theorem \ref{fenbugj1}}
\label{app-41.1}

We first prove the following theorem, which gives a generalization bound for a more general hypothesis space.
\begin{theorem}
\label{fenbugj1}
Let $\D_{S}$, $\D_P$, $\D_S^{l_p}$, $\alpha$ be defined in Section \ref{backdoor-g} and 
$\D_S^{\ne l_p}$ be defined in Section \ref{proof}. 
%
%
Then for any hypothesis space $H=\{h(x,y):\S\times[m]\to[0,1]\}$, 
with probability at least $1-\delta-\frac{4\eta}{4\eta+(1-\eta) N}-\frac{4(1-\eta)}{\eta N+4(1-\eta)}$, for any $h\in H$, it holds
\begin{equation}
\label{eq-th_1c}
\begin{array}{cc}
&\E_{(x,y)\sim \D_S}[h(x,y)]
\le
\frac{4-2\alpha}{1-\alpha}\E_{(x,y)\in \D_P}[h(x,y)]+4\Rad^{\D_S^{\ne l_p}}_N(H)+\frac{4}{1-\alpha}{\Rad^{\D_S^{l_p}}_N(H)}+2\sqrt{\frac{\ln(2/\delta)}{N(1-\alpha)}}.
\end{array}
\end{equation}
\end{theorem}


As mentioned previously, the samples in $\D_{P}$ do not satisfy ``i.i.d. sampled from $\D_{\S}$'', and we will find a subset of $\D_{\S}$, whose samples are i.i.d. sampled from $\D_{\S}$. The core of the proof of theorem \ref{fenbugj1} is the following two lemmas, which show how to select such a subset. 
\begin{lemma}
\label{i.i.d.lemma1}
Use notations in Theorem \ref{fenbugj1}. 
Let $X$ be the random variable of the number of samples with label $\ne l_p$ in $\D_{P}$. Let $k$ be a given number in $\{1,2,\dots,N\eta\}$. 
If $X\ge k$, we randomly select $k$ samples without label $l_p$ in $\D_P$ and let $D_l$ be the set of these samples; otherwise, let $D_l$ be the set of all samples without label $l_p$ in $\D_P$. 
%
%
   Let $D_l$ satisfy the distribution $\D_{S_0}$. Then we have $\Pr_{x\sim \D_{\S_0}}(x\in A|X\ge k)=\Pr(X_{k}^{\D_S^{\ne l_p}}\in A)$ for any set $A$, where $X_{k}^{\D_S^{\ne l_p}}$ means i.i.d. sampling $k$ data from distribution $\D_S^{\ne l_p}$.
\end{lemma}
\begin{proof}
By the Bayesian formula, we have
\begin{equation}
\label{xby}
\begin{array}{ll}
&\Pr_{x\sim \D_{\S_0}}(x\in A|X\ge k)\\
=&\Pr_{x\sim \D_{\S_0}}(x\in A,X\ge k)/\Pr_{x\sim \D_{\S_0}}(X\ge k)\\
=&\sum_{i=k}^N \Pr_{x\sim \D_{\S_0}}(x\in A,X=i)/\Pr_{x\sim \D_{\S_0}}(X\ge k)\\
=&\sum_{i=k}^N \Pr_{x\sim \D_{\S_0}}(x\in A|X=i)\Pr(X=i)/\Pr_{x\sim \D_{\S_0}}(X\ge k)\\
=&\sum_{i=k}^N \Pr_{x\sim \D_{\S_0}}(x\in A|X=i)\Pr(X=i|X\ge k). \\
\end{array}
\end{equation}
We will show that $\Pr_{x\sim \D_{\S_0}}(x\in A|X=i)=\Pr(X_{k}^{\D_S^{\ne l_p}}\in A)$ for any $i\ge k$, and hence the lemma.

Since the poison does not change labels, $X$ is also the number of samples without label $l_p$ in $\D_{\tr}$. Then for any $i\ge k$, when $X=i$, we will traverse all possible selection methods for $D_l$ to calculate $\Pr_{x\sim \D_{\S_0}}(x\in A|X=i)$.


Note that the $N$ samples in $\D_{\tr}$ are i.i.d samples from $\D_S$. We will consider the order of these samples. Let $(x_j,y_j)\in\D_{\tr}$ be the $j$-th element selected from the distribution $\D_\S$. 
If we add poison to $(x_j,y_j)$, then it becomes $(x_j+\p(x_j),y_j)\in\D_P$; if we do not add poison to $(x_j,y_j)$, then $(x_j,y_j)\in\D_P$.

Let $D_{y\ne l_p}\subset[N]$ be the set of $k$ such $(x_k,y_k)\in\D_P$ satisfying $y_k\ne l_p$. By considering all the possible situations of $D_{y\ne l_p}$, we have
\begin{equation*}
\begin{array}{ll}
&\Pr_{x\sim \D_{\S_0}}(x\in A|X=i)\\
=&\Pr_{x\sim \D_{\S_0}}(x\in A,X=i)/\Pr(X=i)\\
=&\sum_{\D_{y\ne l_p},|\D_{y\ne l_p}|=i}\Pr_{x\sim \D_{\S_0}}(x\in A,\D_{y\ne l_p})/\Pr(X=i)\\
=&\sum_{\D_{y\ne l_p},|\D_{y\ne l_p}|=i}\Pr_{x\sim \D_{\S_0}}(x\in A|D_{y\ne l_p})\frac{\Pr(\D_{y\ne l_p})}{\Pr(X=i)}\\
=&{\sum_{\D_{y\ne l_p},|\D_{y\ne l_p}|=i}\Pr_{x\sim \D_{\S_0}}(x\in A|D_{y\ne l_p})}/{C_{N}^i}.\\
\end{array}
\end{equation*}
Thus, for all $\D_{y\ne l_p}$ satisfying $|\D_{y\ne l_p}|=i$, we traverse all the possibilities of the sample index $\{i_1,i_2,\dots,i_{k}\}$ selected by $\D_l$ and then have 
{
\begin{equation*}
\begin{array}{ll}
&\Pr_{x\sim \D_{\S_0}}(x\in A|D_{y\ne l_p})\\
=&\frac{\sum_{i_1,i_2,\dots,i_{k}\in D_{y\ne l_p}}\Pr((x_{i_1},x_{i_2},\dots,x_{i_{k}})\in A)}{C_{i}^{k}}\\
=&\frac{1}{C_{i}^{k}}\sum_{i_1,i_2,\dots,i_{k}\subset D_{y\ne l_p}}\Pr(X_{k}^{\D_S^{\ne l_p}}\in A)\\
=&\Pr(X_{k}^{\D_S^{\ne l_p}}\in A)
\end{array}
\end{equation*}
}
where $x_i$ are i.i.d. and $C_a^b$ is the combination number of selecting $b$ samples from $a$ samples.
We thus have:
\begin{equation*}
\begin{array}{ll}
&\Pr_{x\sim \D_{\S_0}}(x\in A|X=i)\\
=&\frac{\sum_{\D_{y\ne l_p,|\D_{y\ne l_p}|=i}}\Pr_{x\sim \D_{\S_0}}(x\in A|D_{y\ne l_p})}{C_{N}^i}\\
=&\frac{\sum_{\D_{y\ne l_p},|\D_{y\ne l_p}|=i}\Pr(X_{k}^{\D_S^{\ne l_p}}\in A)}{C_{N}^i}\\
=&\Pr(X_{k}^{\D_S^{\ne l_p}}\in A).
\end{array}
\end{equation*}
Finally, we have
\begin{equation*}
\begin{array}{ll}
&\Pr_{x\sim \D_{\S_0}}(x\in A|X\ge k)\\
=&\sum_{i=k}^N \Pr_{x\sim \D_{\S_0}}(x\in A|X=i)\Pr(X=i|X\ge k)\\
=&\sum_{i=k}^N \Pr(X_{k}^{\D_S^{\ne l_p}}\in A)\Pr(X=i|X\ge k)\\
=&\Pr(X_{k}^{\D_S^{\ne l_p}}\in A)\sum_{i=k}^N\Pr(X=i|X\ge k)\\
=&\Pr(X_{k}^{\D_S^{\ne l_p}}\in A).
\end{array}
\end{equation*}
This proves the lemma.
\end{proof}
For samples with label $l_p$, we have the similar conclusion.
\begin{lemma}
\label{i.i.d.lemma2}
Let $X$ be the random variable of the number of samples with label $l_p$ in the set $\D_{P}$. Let $k$ to be a given number in $\{1,2,\dots,[N\eta]\}$. 
   If $X\ge k$, we randomly select $(1-\alpha)k$ samples without trigger but with label $l_p$ in $\D_P$, and let these samples form the set $D_{l_p}$; 
   otherwise, we select all samples without trigger but with label $l_p$ in $\D_P$, and make these samples the set $D_{l_p}$. 
%
  %
Let $D_{l_p}$ obey the distribution $\D_{\S_1}$. Then we have $\Pr_{x\sim \D_{\S_1}}(x\in A|X\ge k)=\Pr(X_{(1-\alpha)k}^{\D^{l_p}_{\S}}\in A)$ for any set $A$, where $X_{(1-\alpha)k}^{\D^{l_p}_{\S}}$  means i.i.d. sample $(1-\alpha)k$ samples from distribution $\D^{l_p}_{\S}$.
\end{lemma}
\begin{proof}
    By the Bayesian formula, we have
\begin{equation*}
\begin{array}{ll}
&\Pr_{x\sim \D_{\S_1}}(x\in A|X\ge k)
=\sum_{i=k}^N \Pr_{x\sim \D_{\S_1}}(x\in A|X=i)\Pr(X=i|X\ge k). \\
\end{array}
\end{equation*}
The intermediate steps are similar to equation \eqref{xby} in the proof of Lemma \ref{i.i.d.lemma1}, so we omit them.

Now we prove $\Pr_{x\sim \D_{\S_1}}(x\in A|X=i)=\Pr(X_{(1-\alpha)k}^{\D^{l_p}_{\S}}\in A)$ for any $i\ge k$.
Note that the $N$ samples in $\D_{\tr}$ are i.i.d selected from $\D_S$. Now we will consider the order of these samples. Let $(x_j,y_j)\in\D_{\tr}$ be the $j$-th element selected from the distribution $\D_\S$. If we add poison to $(x_j,y_j)$, then it becomes $(x_j+\p(x_j),y_j)\in\D_P$; if we do not add poison to $(x_j,y_j)$, then $(x_j,y_j)\in\D_P$.

For any $i\ge k$, when $X=i$, there must be $|\D_{l_p}|=(1-\alpha)k$. Let $D_{y=l_p}\subset[N]$ be the set of $j$ such that $(x_j,y_j)\in\D_P$ satisfied $y_j=l_p$. 
Now we consider all the possible situation of $D_{y=l_p}$:
\begin{equation}
\label{sbbc}
\begin{array}{ll}
&\Pr_{x\sim \D_{\S_1}}(x\in A|X=i)\\
=&\Pr_{x\sim \D_{\S_1}}(x\in A,X=i)/\Pr(X=i)\\
=&\sum_{D_{y=l_p},|D_{y=l_p}|=i}\Pr_{x\sim \D_{\S_1}}(x\in A,D_{y=l_p})/\Pr(X=i)\\
=&\sum_{D_{y=l_p},|D_{y=l_p}|=i}\Pr_{x\sim \D_{\S_1}}(x\in A|D_{y=l_p})\Pr(D_{y=l_p})/\Pr(X=i)\\
=&\sum_{D_{y=l_p},|D_{y=l_p}|=i}\Pr_{x\sim \D_{\S_1}}(x\in A|D_{y=l_p})/C_N^i.\\
\end{array}
\end{equation}
Let $D^{poi}_{y=l_p}\subset[N]$ be the set of $j$ satisfying that $x_j$ is a poisoned sample, and $D^{no\ poi}_{y=l_p}\subset[N]$ be the set of $j$ satisfying $(x_j,y_j)\in \D_{l_p}$. It is easy to see that $D^{no\ poi}_{y=l_p},D^{poi}_{y=l_p}\subset D_{y=l_p}$.
Then, for any given $D_{y=l_p}$ such that $|D_{y=l_p}|=i$, we traverse all possibilities of $D^{poi}_{y=l_p}$ and $D^{no\ poi}_{y=l_p}$ to calculate $\Pr_{x\sim \D_{\S_1}}(x\in A|D_{y=l_p})$:
\begin{equation*}
\begin{array}{ll}
&\Pr_{x\sim \D_{\S_1}}(x\in A|D_{y=l_p})\\
=&\sum_{\{i_k\}_{k=1}^{i\alpha},\{i_j\}_{j=1}^{[N\eta](1-\alpha)}}\Pr(\{i_j\}_{j=1}^{[N\eta](1-\alpha)}\in A)\Pr(D^{poi}_{y=l_p}=\{i_k\}_{k=1}^{i\alpha},D^{no\ poi}_{y=l_p}=\{i_j\}_{j=1}^{[N\eta](1-\alpha)}|D_{y={l_p}})\\
=&\sum_{\{i_k\}_{k=1}^{i\alpha},\{i_j\}_{j=1}^{[N\eta](1-\alpha)}}\Pr(X_{(1-\alpha)k}^{\D^{l_p}_{\S}}\in A)\Pr(D^{poi}_{y=l_p}=\{i_k\}_{k=1}^{i\alpha},D^{no\ poi}_{y=l_p}=\{i_j\}_{j=1}^{[N\eta](1-\alpha)}|D_{y={l_p}})\\
=&\Pr(X_{(1-\alpha)k}^{\D^{l_p}_{\S}}\in A)\sum_{\{i_k\}_{k=1}^{i\alpha},\{i_j\}_{j=1}^{[N\eta](1-\alpha)}}\Pr(D^{poi}_{y=l_p}=\{i_k\}_{k=1}^{i\alpha},D^{no\ poi}_{y=l_p}=\{i_j\}_{j=1}^{[N\eta](1-\alpha)}|D_{y={l_p}})\\
=&\Pr(X_{(1-\alpha)k}^{\D^{l_p}_{\S}}\in A).\\
\end{array}
\end{equation*}
The $x_i$ are i.i.d. and $C_a^b$ is the number of combinations to select $b$ samples from $a$ samples. Substituting it into inequality \eqref{sbbc}, we have
\begin{equation*}
\begin{array}{ll}
&\Pr_{x\sim \D_{\S_1}}(x\in A|X=i)\\
=&\sum_{D_{y=l_p},|D_{y=l_p}|=i}\Pr_{x\sim \D_{\S_1}}(x\in A|D_{y=l_p})/C_N^i\\
=&\sum_{D_{y=l_p},|D_{y=l_p}|=i}\Pr(X_{(1-\alpha)k}^{\D^{l_p}_{\S}}\in A)/C_N^i\\
=&\Pr(X_{(1-\alpha)k}^{\D^{l_p}_{\S}}\in A).\\
\end{array}
\end{equation*}
This is what we want. 
Finally, we have 
\begin{equation*}
\begin{array}{ll}
&\Pr_{x\sim \D_{\S_1}}(x\in A|X\ge k)\\
=&\sum_{i=k}^N \Pr_{x\sim \D_{\S_1}}(x\in A|X=i)\Pr(X=i|X\ge k) \\
=&\sum_{i=k}^N \Pr(X_{(1-\alpha)k}^{\D^{l_p}_{\S}}\in A)\Pr(X=i|X\ge k) \\
=&\Pr(X_{(1-\alpha)k}^{\D^{l_p}_{\S}}\in A).
\end{array}
\end{equation*}
The lemma is proved.
\end{proof}

Now, we prove Theorem \ref{fenbugj1}.
\begin{proof}
By the Bayesian formula, we have  
\begin{equation}
\label{V}
\begin{array}{ll}
&\E_{(x,y)\sim \D_S}[h(x,y)]\\
=&\Pr_{(x,y)\sim \D_S}(y=l_p) \E_{(x,y)\sim \D^{l_p}_\S}[h(x,y)]
+\Pr_{(x,y)\sim \D_S}(y\ne l_p) \E_{(x,y)\sim \D_S^{\ne l_p}}[h(x,y)]\\
=&\eta \E_{(x,y)\sim \D_S^{\ne l_p}}[h(x,y)]+(1-\eta)\E_{(x,y)\sim \D_S^{\ne l_p}}[h(x,y)].
\end{array}
\end{equation}
Now we will separately estimate $\E_{(x,y)\sim \D_S^{\ne l_p}}[h(x,y)]$ and $\E_{(x,y)\sim \D^{l_p}_\S}[h(x,y)]$.

{\bf Upper bound of $\E_{(x,y)\sim \D_S^{\ne l_p}}[h(x,y)]$}.
We give such a bound in the following Results (c1), (c2), and (c3). 

{\bf Result (c1)}: Let random variable $X$ be the number of samples with labels not equal to $l_p$ in $\D_{\tr}$. Then with probability at least $1-\frac{4\eta}{4\eta+N-N\eta}$, we have $X\ge N(1-\eta)/2$.

Let $\D_{\tr}=\{(x_i,y_i)\}_{i=1}^N$.
Then $X=\sum_{i=1}^N \ID(y_i\ne l_p)$, so $\E[X]=\E[\sum_{i=1}^N \ID(y_i\ne l_p)]=\sum_{i=1}^N \E[\ID(y_i\ne l_p)]=N(1-\eta)$, and the variance $\Var[X]=\Var[\sum_{i=1}^N \ID(y_i\ne l_p)]=\sum_{i=1}^N \Var[\ID(y_i\ne l_p)]=N(1-\eta)\eta$, because $\ID(y_i\ne l_p)$ are independent events. 
By the Cantelli inequality, we have $\Pr(X\le N(1-\eta)/2)\le \frac{\Var[X]}{\Var[X]+(\E[X]-(N(1-\eta)/2))^2}=\frac{4\eta}{4\eta+N-N\eta}$, and Result (c1) is proved.

{\bf Result (c2):} 
We randomly select $N(1-\eta)/2$ samples without label $l_p$ in $\D_\p$ and let $D_l$ be the set of these samples. 
If the number of samples without label $l_p$ in $\D_\p$ is less than $N(1-\eta)/2$, then we select all samples without label $l_p$, and let $D_l$ be the set of these samples. 

Assuming that the set $\D_l$ obeys the distribution $\D_{\S_0}$, we have $\Pr_{x\sim \D_{\S_0}}(x\in A|X\ge N(1-\eta)/2)=\Pr(X_{N(1-\eta)/2}^{\D_S^{\ne l_p}}\in A)$ for any set $A$, where $X_{N(1-\eta)/2}^{\D_S^{\ne l_p}}$ is the set of $N(1-\eta)/2$ data i.i.d. sampled from distribution $\D_S^{\ne l_p}$. 

Following Lemma \ref{i.i.d.lemma1}, Result (c2) shows that when $X\ge N(1-\eta)/2$, $\D_{l}$ can be seen as i.i.d. sampled from $\D_S^{\ne l_p}$.

{\bf Result (c3)}: 
With probability $1-\frac{4\eta}{4\eta+N-N\eta}-\delta/2$, we have 
\begin{equation*}
\begin{array}{ll}
&\E_{(x,y)\sim \D_S^{\ne l_p}}[h(x,y)]
\le\frac{\sum_{(x,y)\in \D_{P}}h(x,y)}{N(1-\eta)/2}+2\Rad^{\D_S^{\ne l_p}}_{N(1-\eta)/2}(H)+\sqrt{\frac{\ln(2/\delta).
}{N(1-\eta)}}.
\end{array}
\end{equation*}


$D\in(\R^n\times [m])^{N(1-\eta)/2}$ is called a bad set, if $|D|=N(1-\eta)/2$, and 
\begin{equation}
\label{B-0}
\begin{array}{ll}
&\E_{(x,y)\sim \D_S^{\ne l_p}}[h(x,y)]
>\frac{\sum_{(x,y)\in D}h(x,y)}{N(1-\eta)/2}+2\Rad^{\D_S^{\ne l_p}}_{N(1-\eta)/2}(H)+\sqrt{\frac{\ln(2/\delta)
}{N(1-\eta)}}
\end{array}
\end{equation} 
for some $f\in H$. Let $S_b=\{D\|D\ \hbox{ is a bad set}\}$.

By Theorem \ref{fanhua}, if the samples in $\D$ are i.i.d. sampled form $\D_S^{\ne l_p}$ and $|D|=N(1-\eta)/2$, then with probability $1-\delta/2$, we have  $\E_{(x,y)\sim \D_S^{\ne l_p}}[h(x,y)]\le\frac{\sum_{(x,y)\in D} h(x,y)}{N(1-\eta)/2}+2\Rad_{N(1-\eta)/2}(H)+\sqrt{\frac{\ln(2/\delta)
}{N(1-\eta)}}$ for any $f\in H$. 

So by the definition of bad set, we have $\Pr(X_{N(1-\eta)/2}^{\D_S^{\ne l_p}}\in S_b)\le\delta/2$, where $X_{N(1-\eta)/2}^{\D_S^{\ne l_p}}$ is mentioned in result (c2). And by Result (c2), we have that: when $X\ge N(1-\eta)/2$, we have  $\Pr(D_l\in S_b)\le\delta/2$. Then, by Result (c1), we have
\begin{equation*}
\begin{array}{ll}
&\Pr(D_l\notin S_b,X\ge N(1-\eta)/2)\\
=&\Pr(D_l\notin S_b|X\ge N(1-\eta)/2)\Pr(X\ge N(1-\eta)/2)\\
\ge&(1-\delta/2)(1-\frac{4\eta}{4\eta+N-N\eta})(by\ (c1))\\
\ge&1-\delta/2-\frac{4\eta}{4\eta+N-N\eta}.\\
\end{array}
\end{equation*}



So, with probability at least $1-\delta/2-\frac{4\eta}{4\eta+N-N\eta}$, $D_l$ is not in $S_b$ and $X\ge N(1-\eta)/2$. Hence,   
\begin{equation*}
\begin{array}{ll}
&\E_{(x,y)\sim \D_S^{\ne l_p}}[h(x,y)]\\
\le&\frac{\sum_{(x,y)\in D_l} h(x,y)}{N(1-\eta)/2}+2\Rad_{N(1-\eta)/2}(H)+\sqrt{\frac{\ln(2/\delta)
}{N(1-\eta)}}\\
\le&\frac{\sum_{(x,y)\in \D_P} h(x,y)}{N(1-\eta)/2}+2\Rad_{N(1-\eta)/2}(H)+\sqrt{\frac{\ln(2/\delta)
}{N(1-\eta)}}.
\end{array}
\end{equation*}
This is what we want.

{\bf The upper bound of $\E_{(x,y)\sim \D^{l_p}_\S}[h(x,y)]$.}

We will give such a bound in Results (d1), (d2), and (d3) below, which are similar to results (c1), (c2), (c3).

{\bf Result (d1)}: Let the random variable $X$ be the number of samples with label $l_p$ in $\D_{\tr}$. Then with probability at least $1-\frac{4(1-\eta)}{4(1-\eta)+N\eta}$, we have  $X\ge N\eta/2$. 
The proof is similar to that of Result (c1).

{\bf Result (d2)}: Now we evenly select $[N\eta(1-\alpha)/2]$ samples with label $l_p$ in $\D_\p/\D_{poi}$ and make these samples to form the set $D_{l_p}$. 
If the number of samples with label $l_p$ in $\D_\p/\/D_{poi}$ is smaller than $[N\eta(1-\alpha)/2]$, then we select all of these samples and add these samples to $D_{l_p}$. 

Assume that the set $\D_{l_p}$ obeys the distribution $\D_{\S_1}$. Then we have $\Pr_{x\sim \D_{\S_1}}(x\in A|X\ge [N\eta/2])=\Pr(X_{[N\eta(1-\alpha)/2]}^{\D^{l_p}_{\S}}\in A)$ for any set $A$, where $X_{[N\eta(1-\alpha)/2]}^{\D^{l_p}_{\S}}$ is the set that i.i.d. selecting $[N\eta(1-\alpha)/2]$ samples from distribution $\D_S^{\ne l_p}$.

Following Lemma \ref{i.i.d.lemma2}, Result (d2) shows that when $X \ge [N\eta/2]$, $D_{l_p}$ can be seen as i.i.d. samples from $D^{l_p}_\S$. 

{\bf Result (d3)}: With probability $1-\frac{4-4\eta}{4(1-\eta)+N\eta}-\delta/2$, we have 
\begin{equation*}
\begin{array}{ll}
&\E_{(x,y)\sim \D^{l_p}_\S}[h(x,y)]
\le\frac{\sum_{(x,y)\in \D_{P}}h(x,y)}{N(1-\alpha)\eta/2}+2\Rad^{\D_S^{l_p}}_{N(1-\alpha)\eta/2}(H)+\sqrt{\frac{\ln(2/\delta)
}{N(1-\alpha)\eta}}.
\end{array}
\end{equation*}
This is similar to Result (c3), but using (d1) and (d2) instead.

{\bf To obtain a bound of $E_{(x,y)\sim\D_S}[h(x,y)]$.}

Using the fact $\Rad^D_M(H)\le\frac{N}{M}\Rad^D_N(H)$ for any $M\le N$, distribution $\D$, hypothesis space $H$, and applying (c3), (d3) to equation \eqref{V}, we finally have
\begin{equation}
\begin{array}{ll}
&\E_{(x,y)\sim \D_S}[h(x,y)]\\
=&\eta \E_{(x,y)\sim \D^{l_p}_\S}[h(x,y)]+(1-\eta)\E_{(x,y)\sim \D_S^{\ne l_p}}[h(x,y)]\\
\le& \frac{4-2\alpha}{1-\alpha}\frac{\sum_{(x,y)\in \D_{P}}h(x,y)}{N}+4\Rad^{\D_S^{\ne l_p}}_N(H)+4\frac{\Rad^{\D_S^{l_p}}_N(H)}{1-\alpha}+\sqrt{\frac{\ln(2/\delta)(1-\eta)}{N}}+\sqrt{\frac{\ln(2/\delta)\eta}{N(1-\alpha)}}.
\end{array}
\end{equation}
Then using $\eta<1$, we prove the theorem.
\end{proof}

\subsection{Estimate Rademacher Complexity}
\label{app-41.2}
In this section, we will estimate the Rademacher Complexities in Theorem \ref{fenbugj1} when $H=\{\ID(\F(x)\ne y)\}$, where $\F$ is a network with width $W$ and depth $D$.
We first need a definition:
\begin{definition}
Let $H$ be a hypothesis space. 
Then the growth function $\Pi_{H}(N)$ of $H$ is defined as:
   $$\Pi_{H}(N)=\max_{\{x_i\}_{i=1}^N}|\{(h(x_i))_{i=1}^N\|h\in H\}|.$$
\end{definition}

For a hypothesis space $H=\{h:\R^n\to\{-1,1\}\}$, we can estimate its VCdim \cite{vapnik2015uniform}, and the relationship between VCdim and growth function.
\begin{lemma}[\cite{bartlett2021deep,bartlett2019nearly}]
\label{vcdm}
Let $H$ be the hypothesis space that satisfies: $\F\in H$ if and only if $\F$ is a network with width not more than $W$ and depth not more than $D$, and the activation function of each hidden layer of $\F$ is Relu, the output layer uses $\sign$ as activation function.
Then the VCdim of $H$ is $O(D^2W^2)$.
\end{lemma}

\begin{lemma}[\cite{sauer1972density}]
\label{sau}
Let $H$ be the hypothesis space with VCdim $V$. Then for any $N\ge1$, the growth function satisfies $\Pi_H(N)\le(eN)^V$.
\end{lemma}

\begin{lemma}[\cite{massart2000some,mohri2018foundations}]
\label{rade}
Let $H$ be the hypothesis space with growth function $\Pi_H$, and any $h(x)\in H$ satisfy $|h(x)|\le 1$ for any $x$. 
Then for any distribution $\D$, we have  $\Rad^\D_N(H)=O(\sqrt{\frac{\ln(\Pi_H)}{N}})$.
\end{lemma}

\begin{lemma}[\cite{mohri2018foundations}]
\label{wzj}
Let $H$ be the hypothesis space,  $q\in\R^m\to\{0,1\}$, and $H_q=\{q(h_1(x),h_2(x),\dots,h_m(x))\|h_i\in H\}$. Then $\Pi_{H_q}(N)\le (\Pi_H(N))^m$.
\end{lemma}

Now, we calculate the Rademacher complexity of $H=\{\ID(\widehat{\F}(x)\ne y)\|\F\in\Hyp_{W,D}\}$:
\begin{lemma}
     Let $H=\{\ID(\widehat{\F}(x)\ne y):[0,1]^n\times[m]\to \{0,1\}\|\F\in\Hyp_{W,D}\}$. Then, for any distribution $\D$, we have  $\Rad^\D_N(H)=O(\sqrt{\frac{mW^2D^2}{N}})$. 
\end{lemma}
\begin{proof}
Let $H^0_{W,D}$ be defined as: $\F\in H^0_{W,D}$ if and only if $\F$ is a network with width $W$ and depth $D$, and the activation function of each hidden layer of $\F$ is Relu, the output layer does not use the activation function. And let $H^0=\{\ID(\widehat\F(x)\ne y)\|\F\in H^0_{W,D}\}$. 

Let $H^1_{W,D}$ be defined as: $\F\in H^1_{W,D}$ if and only if $\F$ is a network with width $W$ and depth $D$, and the activation function of each hidden layer of $\F$ is Relu, the output layer uses the activation function $\sign$.

Then we have that $H^0=\{\ID(\widehat\F(x)\ne y)\|\F\in H^0_{W,D}\}=\{\ID(\widehat\F(x)\ne y)\|\F\in\Hyp_{W,D}\}=H$ and 
$\Pi_{H^1_{W,D}}(N)\le(eN)^{O(D^2W^2)}$ by Lemmas \ref{sau} and \ref{vcdm}.

For any $\F\in H^0_{W,D}$, let $\F_{i,j}=\sign(\F_i-\F_j)$, where $\F_i$ is the $i$-th weight of $\F$. Then it is easy to see that, $\F_{i,j}\in H^1_{W,D}$. Since $\ID(\widehat\F(x)\ne y)=\ID(-\sum_{j\ne y}\F_{y,j}+(m-1)-0.1)$. 
By Lemma \ref{wzj}, the growth function of $H^0$ is $(eN)^{mO(D^2W^2)}$; so by Lemma \ref{rade} and $H^0=H$, the Rademacher complexity $\Rad^\D_N(H)$ of $H$ is $O(\sqrt{\frac{mD^2W^2}{N}})$ (ignore minor items). 
This proves the lemma. 
\end{proof}

\subsection{Proof for Theorem \ref{fenbugj}}
\label{app-41.3}
Now we prove Theorem \ref{fenbugj} by using Sections \ref{app-41.1} and \ref{app-41.2}:

\begin{proof}
Taking $H=\{\ID(\F(x)\ne y)\| \F\in\Hyp_{W,D}\}$ and substituting the Rademacher complexity in Section \ref{app-41.2} into Theorem \ref{fenbugj1}, we prove Theorem \ref{fenbugj}.
\end{proof}

\section{Poison Impact Empirical Errors}
\label{app-42}
In this appendix, we prove a proposition to support Remark \ref{rem-lpe}.
In Proposition \ref{jyc} below, we show that if the poison $\p(x)$ is not satisfactory, then it can result in a big empirical error. 

\begin{remark}
Please note that the conclusion ``poison implies big empirical error'' only holds in some situations. 
In fact, when $\p(x)$ is bounded, or the network in the hypothesis space has a strong expressive ability, then the conditions for the proposition in this section will not hold, so the poison will not lead to big empirical error. 
In our experimental result in Section \ref{ser}, there is no need to consider the occurrence of large empirical error, because we bound the trigger $\p(x)$ and use a large network.
\end{remark}
%
%

Let $\D^{poi}$ be the distribution of poisoned data with label $l_p$, that is, $$\Pr_{(x,y)\sim\D_S^{l_p}}(x+\p(x)\in A)=\Pr_{x\sim\D^{poi}}(x\in A)$$ 
for any set $A$.

\begin{proposition}
\label{jyc}
Use the notation introduced in Theorem \ref{fenbugj1}
and let $\eta\le 0.5$ be defined in \eqref{eq-eta}.
We further assume that $h(x,y_1)+h(x,y_2)\ge1$ for any $h\in H$, $x\in\R^n$, $y_1\ne y_2$, and for some $\tau, V\in\R_+$, it holds:\\
(1) $\Pr_{x\sim\D^{poi}}(x\in A)\ge\tau \Pr_{(x,y)\sim\D_S^{\ne l_p}}(x\in A)$ for any set $A$;\\
(2) $\Rad^{\D_S^{\ne l_p}}_N(H)\le V$.\\
Let $Q=\eta^{N}\tau^{N\alpha}$, if $\alpha-\delta/Q> 0$ and $0.5-2\delta\alpha/Q-V\alpha>0$. 
Then with probability $\delta$, for any $h\in H$ we have 
$$\E_{(x,y)\in \D_P}[h(x,y)]> \frac{0.5-2\delta\alpha/Q-V\alpha}{\alpha-\delta/Q}.$$
\end{proposition}

First, we have the following lemma \cite{bertsekas2008introduction}.
\begin{lemma}
\label{mjy}
If distributions $D_1$, $D_2$ and function $h$ satisfy
$\Pr_{x\sim D_1}(h(x)\in A)\le \lambda \Pr_{x\sim D_2}(h(x)\in A)$ for any set $A$, 
then $\E_{x\sim D_1}[f(h(x))]\le \lambda \E_{x\sim D_2}[f(h(x))]$ for any bounded and positive measurable function $f(x)$. 
\end{lemma}
Then we prove the following lemma, which directly lead to Proposition \ref{jyc}.

\begin{lemma}
\label{jy}
Use the notations in Proposition \ref{jyc}.
We further assume that $h(x,y_1)+h(x,y_2)\ge1$ for any $h\in H$, $x\in\R^n$ and $y_1\ne y_2$, and for some $\tau, V\in\R_+$, it holds:\\
(1): $\Pr_{x\sim\D^{poi}}(x\in A)\ge\tau \Pr_{(x,y)\sim\D_S^{\ne l_p}}(x\in A)$ for all sets $A$, where $\tau\in\R_+$;\\
(2): with probability $1-\delta$, there exists an $h\in H$ such that $\E_{(x,y)\in \D_P}[h(x,y)]\le \bfepsilon$.\\
Then for any $K\le N\alpha$, we have
\begin{equation*}
\label{eq-th_12}
\begin{array}{l}
\Rad^{\D_S^{\ne l_p}}_K(H)\ge 0.5-N\bfepsilon/K-\frac{(2-N\bfepsilon/K)\delta\alpha}{\eta^{N}\tau^K}.\\

\Rad^{\D_S^{\ne l_p}}_N(H)\ge K(0.5-N\bfepsilon/K-\frac{(2-N\bfepsilon/K)\delta\alpha}{\eta^{N}\tau^K})/N.\\
\end{array}
\end{equation*}
\end{lemma}

\begin{proof}
In order to calculate the probability easily in the proof, we treat $\D_{\tr}$ and $\D_P$ as ordered sets. 
Let $(x_i,y_i)$ be the $i$-th element of $\D_{\tr}$.
$\D_P$ is obtained from $\D_{\tr}$ as follows: Let $\D_P=\D_{\tr}$ first and if poison $\p(x_i)$ is added to $(x_i,y_i)$ for some $i$, then replace the $i$-th element of $\D_{\tr}$ by $(x_i+\p(x_i),y_i)$.

{\bf First, we give three notations:}

(1): For the poisoned training set $\D_P$ and $q\in\N$, we take the first $q$ clean samples without $l_p$ in $\D_P$ to form a subset $F_{clean}(q,\D_p)$, that is, if we write $\D_{\tr}=\{(x_{i},y_i)\}_{i=1}^{N}$, then $F_{clean}(q,\D_P)=\{(x_{i_k},y_{i_k})\}_{k=1}^{q}$, where $\{i_k\}_{i=1}^q$ is the smallest $q$ numbers in $[N]$ that satisfy $y_{i_k}\ne l_p$, $(x_{i_k},y_{i_k})\in\D_{clean}$, and $i_a<i_b$ if $a<b$.

(2): For the poisoned training set $\D_P$ and $q\in\N$, we take the first $q$ poisoned samples in $\D_P$ to form the subset $F_{poison}(q,\D_p)$, that is, if we write $\D_{\tr}=\{(x_{i},y_i)\}_{i=1}^{N}$, then $F_{poison}(q,\D_P)=\{(x_{i_k}+\p(x_{i_k}),y_{i_k})\}_{k=1}^{q}$,  $\{i_k\}_{i=1}^q$ is the smallest $q$ numbers in $[N]$ that satisfy $y_{i_k}=l_p$,  $(x_{i_k}+\p(x_{i_k}),y_{i_k})\in\D_{poi}$, and $i_a<i_b$ if $a<b$.

(3): For any given $K$ samples $\{(x_i,y_i)\}_{i=1}^{K}$ where $y_i\ne l_p$ and a given set $S\subset[K]$, let $S_i$ be the $i$-th minimum number in $S$, in particular, $\S_1$ is the minimum number in $S$ and $S_{|S|}$ is the maximum number in $S$. 

{\bf Second, we define a property of $\D_P$:}

The poisoned training set $\D_P$ obtained by poisoning $\D_{\tr}$ is said to be {\em nice-inclusion} of a new dataset $\D_G$
if 
$\D_G = \D_{T_p} \cup D_T$ such that $D_T\subset \D_{\tr}$, $D_{T_p}\subset \D^{poi}$, and $\min_{h\in H} \frac{1}{|\D_P|} \sum_{(x,y)\in \D_P} h(x,y)\le\bfepsilon$.
For convenience, we write these conditions more explicitly as follows.


The poisoned training set $\D_P$ obtained by poisoning $\D_{\tr}=\{(x'_i,y'_i)\}_{i=1}^N$ is said to be {\em nice-inclusion} of the ordered dataset $\D_G=\{(z_i,l_i)\}_{i=1}^{K}$ and $S\subset[K]$ satsfying $|[K]\setminus S|=|\D_{poi}|$, if  \\
(e1): Let $F_{clean}(|S|,\D_P)=\{(x'_{i_k},y'_{i_k})\}_{k=1}^{|S|}$ such that $x'_{i_k}=z_{S_k}$ and $y'_{i_k}=l_{S_k}$ for any $k\in[S]$;\\
(e2):  Let $F_{poison}(|[K]\setminus S|,\D_P)=\{(x'_{i_k}+\p(x'_{i_k}),y'_{i_k})\}_{k=1}^{K-|S|}$. There must be $x'_{i_k}+\p(x'_{i_k})=z_{([K]\setminus S)_k}$ for any $k\in[K-|S|]$. \\
(e3): $\min_{f\in H}\sum_{(x,y)\in \D_P}\frac{h(x,y)}{|\D_P|}\le\bfepsilon$.

We say that the poisoned training set $\D_P$ obtained by poisoning $\D_{\tr}=\{(x'_i,y'_i)\}_{i=1}^N$ is said to be {\em common-nice-inclusion} of the ordered dataset $\D_G=\{(z_i,l_i)\}_{i=1}^{K}$ and $S\subset[K]$ satisfying $|[K]\setminus S|=|\D_{poi}|$, if (e1) and (e2) hold.

{\bf Now, we define some functions:}

Let $v_i\in\{-1,1\}$ for $i\in[K]$ and $S((v_i)_{i=1}^{K})=\{i\|i\in[K],v_i<0\}$.
Given $K$ samples $\{(x_i,y_i)\}_{i=1}^{K}$ where $y_i\ne l_p$, we define that $S_i(\{(x_i,y_i)\}_{i=1}^{K},S((v_i)_{i=1}^{K}))=1$, if there is a {\em nice-inclusion} poison set $\D_P$ of $\{(x_i,y_i)\}_{i=1}^{K}$ and $S((v_i)_{i=1}^{K})$; otherwise $S_i(\{(x_i,y_i)\}_{i=1}^{K},S((v_i)_{i=1}^{K}))=0$. 
Then we have the following results.

{\bf Result one:} If $S_i({(x_i,y_i)}_{i=1}^{K},S((v_i)_{i=1}^{K}))=1$ and $y_i\ne l_p$, then $\sup_{f\in H}\sum_{i\in[K]}v_ih(x_i,y_i)\ge\sum_{i=1}^K\ID(v_i>0)-N\bfepsilon$. 

Let $\D_P$ be a nice-inclusion of ${(x_i,y_i)}_{i=1}^{K},S((v_i)_{i=1}^{K})$. Then 
\begin{equation*}
\begin{array}{ll}
&\bfepsilon\\
\ge&\min_{f\in H}\sum_{(x,y)\in \D_P}\frac{h(x,y)}{|\D_P|}(by\ (e3))\\
\ge&\min_{f\in H} (\sum_{i\in S((v_i)_{i=1}^{K})}h(x_i,y_i))/|\D_P|+(\sum_{i\in [K]/S((v_i)_{i=1}^{K})}h(x_i,l_p))/|\D_P|(by\ (e1,e2))\\
\ge&\min_{f\in H} (\sum_{i\in S((v_i)_{i=1}^{K})}h(x_i,y_i))/|\D_P|+(\sum_{i\in [K]/S((v_i)_{i=1}^{K})}1-h(x_i,y_i))/|\D_P|(by\ y_i\ne l_p)\\
=&\min_{f\in H}(|[K]/S((v_i)_{i=1}^{K})|-\sum_{i\in[K]}v_ih(x_i,y_i))/|\D_P|\\
=& (\sum_{i=1}^K\ID(v_i>0)-\sup_{f\in H}\sum_{i\in[K]}v_ih(x_i,y_i))/|\D_P|\\
=& (\sum_{i=1}^K\ID(v_i>0)-\sup_{f\in H}\sum_{i\in[K]}v_ih(x_i,y_i))/N.\\
\end{array}
\end{equation*}
The result is proved.

{\bf Result Two}: In order to give the lower bound of the Rademacher complexity $\Rad^{\D^l}_K(H)$, we just need to consider the upper bound of $\E_{(x_i,y_i)\sim \D_S^{\ne l_p}}[\ID(S_i(\{(x_i,y_i)\}_{i=1}^K,S((\sigma_i)_{i=1}^{K}))\ne1)]$ for each  $(\sigma_i)_{i=1}^{K}\subset\{-1,1\}^K$. 

Let $\sigma_i$ be Rademacher random variables, that is $P(\sigma_i=1)=P(\sigma_i=-1)=0.5$.  Then, by the definition of Rademacher complexity, we have
\begin{equation*}
\begin{array}{ll}
&\Rad^{\D^l}_K(H)\\
=&\E_{(x_i,y_i)\sim \D_S^{\ne l_p}}[\E_{\sigma_i}[\sup_{f\in H}\frac{\sum_{i=1}^K \sigma_ih(x_i,y_i)}{K}]]\\
=&\E_{(x_i,y_i)\sim \D_S^{\ne l_p}}[\sum_{\sigma_i}\sup_{f\in H}\frac{\sum_{i=1}^K \sigma_ih(x_i,y_i)}{2^KK}]\\
\ge&\E_{(x_i,y_i)\sim \D_S^{\ne l_p}}[\sum_{\sigma_i}\ID(S_i(\{(x_i,y_i)\}_{i=1}^K,S((\sigma_i)_{i=1}^{K}))=1)\frac{\sum_{i=1}^K\ID(\sigma_i>0)-N\bfepsilon}{2^KK}\\
&-\ID(S_i(\{(x_i,y_i)\}_{i=1}^K,S((\sigma_i)_{i=1}^{K}))\ne1)\frac{1}{2^K}]\\
\ge&\E_{(x_i,y_i)\sim \D_S^{\ne l_p}}[\sum_{\sigma_i}\frac{\sum_{i=1}^K\ID(\sigma_i>0)-N\bfepsilon}{2^KK}-\ID(S_i(\{(x_i,y_i)\}_{i=1}^K,S((\sigma_i)_{i=1}^{K}))\ne1)\frac{2K-N\bfepsilon}{2^KK}]\\
=&0.5-N\bfepsilon/K-\E_{(x_i,y_i)\sim \D_S^{\ne l_p}}[\sum_{\sigma_i}\ID(S_i(\{(x_i,y_i)\}_{i=1}^K,S((\sigma_i)_{i=1}^{K}))\ne1)\frac{2K-N\bfepsilon}{2^KK}].
\end{array}
\end{equation*}
The first inequality uses Result one. 
So, if we want to give a lower bound of $\Rad^{\D^l}_K(H)$, we just need to give an upper bound of $\E_{(x_i,y_i)\sim \D_S^{\ne l_p}}[\sum_{\sigma_i}\ID(S_i(\{(x_i,y_i)\}_{i=1}^K,S((\sigma_i)_{i=1}^{K}))\ne1)]$. Furthermore, we just need to consider $\E_{(x_i,y_i)\sim \D_S^{\ne l_p}}[\ID(S_i(\{(x_i,y_i)\}_{i=1}^K,S((\sigma_i)_{i=1}^{K}))\ne1)]$ for each  $(\sigma_i)_{i=1}^{K}$. 
Result two is proved.

{\bf Result Three:} Now, we prove $\E_{(x_i,y_i)\sim \D_S^{\ne l_p}}[\ID(S_i(\{(x_i,y_i)\}_{i=1}^K,S((v_i)_{i=1}^{K}))\ne1)]<\frac{\delta}{\eta^{N}\tau^K}$ for any $(v_i)_{i=1}^{K}\in\{-1,1\}^K$.

For a given $(v_i)_{i=1}^{K}\in\{-1,1\}^K$, let set $C_{(v_i)_{i=1}^{K}}=\{\{(x_i,y_i)\}_{i=1}^K|S_i(\{(x_i,y_i)\}_{i=1}^K,S((v_i)_{i=1}^{K}))\ne1,y_i\ne l_p\}$, then $\E_{(x_i,y_i)\sim \D_S^{\ne l_p}}[\ID(S_i(\{(x_i,y_i)\}_{i=1}^K,S((v_i)_{i=1}^{K}))\ne1)]=\E_{(x_i,y_i)\sim\D_S^{\ne l_p}}[\ID(\{(x_i,y_i)\}_{i=1}^K\in C_{(v_i)_{i=1}^{K}})]$.

For $(v_i)_{i=1}^{K}\in\{-1,1\}^K$ and $\{(x_i,y_i)\}_{i=1}^{K}$, let $\E_{(v_i)_{i=1}^{K}}^{\{(x_i,y_i)\}_{i=1}^{K}}$ be the set of $\D_P$, which is a {\em common-nice-inclusion} for 
%
$\{(x_i,y_i)\}_{i=1}^{K}$ and $S((v_i)_{i=1}^{K})$.
It is easy to see that:\\
(r1): $\E_{(v_i)_{i=1}^{K}}^{((x_i,y_i))_{i=1}^{K}}\cap \E_{(v'_i)_{i=1}^{K}}^{((x'_i,y'_i))_{i=1}^{K}}=\phi$ when 
$v_i\ne v'_i$ for some $i\in[K]$ or $x_i\ne x'_i$ for some $i\in[K]$.\\
(r2): If $\D_P\in \E_{(v_i)_{i=1}^{K}}^{\{(x_i,y_i)\}_{i=1}^{K}}$ satisfies (e3), then $S_i(\{(x_i,y_i)\}_{i=1}^K,S((v_i)_{i=1}^{K}))=1$.

So by (r2), if $\{(x_i,y_i)\}_{i=1}^K\in C_{(v_i)_{i=1}^{K}}$, then for any $\D_P\in \E_{(v_i)_{i=1}^{K}}^{\{(x_i,y_i)\}_{i=1}^{K}}$, (e3) cannot stand. Let the set $B$ contain all the $\D_P$ that do not satisfy (e3). Then, if $\{(x_i,y_i)\}_{i=1}^K\in C_{(v_i)_{i=1}^{K}}$, there must be $\E_{(v_i)_{i=1}^{K}}^{{(x_i,y_i)}_{i=1}^{K}}\subset B$. 

Now we prove the following two results:

{\bf Result S1:}  $\int_{\D_P}\ID(D\in \bigcup_{\{(x_i,y_i)\}_{i=1}^{K}\in C_{(v_i)_{i=1}^K}}\E_{(v_i)_{i=1}^{K}}^{\{(x_i,y_i)\}_{i=1}^{K}})\d D<\delta$.

By (r1) and $\E_{(v_i)_{i=1}^{K}}^{{(x_i,y_i)}_{i=1}^{K}}\subset B$ for all $\{(x_i,y_i)\}_{i=1}^K\in C_{(v_i)_{i=1}^{K}}$, we have that: $$\ID(D\in B)\ge \ID(\D\in \bigcup_{\{(x_i,y_i)\}_{i=1}^{K}\in C_{(v_i)_{i=1}^K}}\E_{(v_i)_{i=1}^{K}}^{\{(x_i,y_i)\}_{i=1}^{K}})$$
So $\int_{\D_P}\ID(\D\in \bigcup_{\{(x_i,y_i)\}_{i=1}^{K}\in C_{(v_i)_{i=1}^K}}\E_{(v_i)_{i=1}^{K}}^{\{(x_i,y_i)\}_{i=1}^{K}})\d D\le\int_{\D_P}\ID(\D\in B)\d D\le\delta$, using condition (2) of the theorem here.

{\bf Result S2:}  $\int_{\D_P}\ID(D\in \bigcup_{\{(x_i,y_i)\}_{i=1}^{K}\in C_{(v_i)_{i=1}^K}}\E_{(v_i)_{i=1}^{K}}^{\{(x_i,y_i)\}_{i=1}^{K}})\d D\ge \E_{(x_i,y_i)\sim \D_S^{\ne l_p}} [\ID((x_i,y_i)\in C_{(v_i)_{i=1}^K})]\eta^N\tau^K/\alpha$.

Consider the definition of $F_{clean}$ and $\F_{poison}$ in (e1) and (e2). When $D_p\in \E_{(v_i)_{i=1}^{K}}^{\{(x_i,y_i)\}_{i=1}^{K}}$, the first $\sum_{i=1}^K\ID(v_i<0)$ samples without label $l_p$ in $\D_{\tr}$ must be $\{(x_{i_k},y_{i_k})\}$, where $i_k$ satisfies $v_{i_k}=-1$; and $\D_{poi}=\{(x_{i_k},y_{i_k})\}$, where $i_k$ satisfies $v_{i_k}=1$. Let $N_0=\sum_{i=1}^K\ID(v_i<0)$, then
\begin{equation*}
\begin{array}{ll}
&\int_{\D_P}\ID(D\in \bigcup_{\{(x_i,y_i)\}_{i=1}^{K}\in C_{(v_i)_{i=1}^K}}\E_{(v_i)_{i=1}^{K}}^{\{(x_i,y_i)\}_{i=1}^{K}})\d D\\
=&\int_{(\D_S^{\ne l_p})^{N_0}(\D_{poi})^{K-N_0}} \ID(\{(x_i,y_i)\}\in C_{(v_i)_{i=1}^K})(\sum_{q=0}^N\ID([q\alpha]=K-N_0)C_{N}^{q}\eta^q(1-\eta)^{N-q})\d (x_i,y_i)\\
\ge&\int_{(\D_S^{\ne l_p})^{N_0}(\D_{poi})^{K-N_0}} \ID(\{(x_i,y_i)\}\in C_{(v_i)_{i=1}^K})(\sum_{q=0}^N\ID([q\alpha]=K-N_0)\eta^N)\d (x_i,y_i)\\
\ge&\int_{(\D_S^{\ne l_p})^{N_0}(\D_{poi})^{K-N_0}} \ID(\{(x_i,y_i)\}\in C_{(v_i)_{i=1}^K})(\eta^N/\alpha)\d (x_i,y_i)\\
\ge&\int_{(\D_S^{\ne l_p})^K} \ID((x_i,y_i)\in C_{(v_i)_{i=1}^K})\eta^N\tau^K/\alpha \d (x_i,y_i)\\
=&\E_{(x_i,y_i)\sim \D_S^{\ne l_p}} [\ID((x_i,y_i)\in C_{(v_i)_{i=1}^K})]\eta^N\tau^K/\alpha.
\end{array}
\end{equation*}

The first inequality uses $\eta\le0.5$ and $C_N^i\ge 1$. The second inequality uses at least $1/\alpha$ numbers of $q\in[N]$ such that $[q\alpha]=K-N_0$. The last inequality uses Lemma \ref{mjy} and condition (1) in theorem. This proves Result S2.

Finally, by Result S1 and Result S2, we have $(\eta^N\tau^K/\alpha)\E_{(x_i,y_i)\sim\D_S^{\ne l_p}}[\ID(\{(x_i,y_i)\}_{i=1}^K\in C_{(v_i)_{i=1}^{K}})]\le \int_{\D_P}\ID(D\in \bigcup_{\{(x_i,y_i)\}_{i=1}^{K}\in C_{(v_i)_{i=1}^K}}\E_{(v_i)_{i=1}^{K}}^{\{(x_i,y_i)\}_{i=1}^{K}})\d D\le\delta$, that is, $\E_{(x_i,y_i)\sim\D_S^{\ne l_p}}[\ID(\{(x_i,y_i)\}_{i=1}^K\in C_{(v_i)_{i=1}^{K}})]\le\alpha\delta/(\eta^N\tau^K)$.

{\bf Prove the lemma.}
Use Results three and two, we have
\begin{equation*}
\begin{array}{ll}
&\Rad^{\D_S^{\ne l_p}}_K(H)\\
\ge&0.5-N\bfepsilon/K-\E_{(x_i,y_i)\sim \D_S^{\ne l_p}}[\sum_{\sigma_i}\ID(S_i(((x_i,y_i))_{i=1}^K,S((\sigma_i)_{i=1}^{K}))\ne1)\frac{2K-N\bfepsilon}{2^KK}]\\
\ge&0.5-N\bfepsilon/K-\frac{(2-N\bfepsilon/K)\delta\alpha}{\eta^{N}\tau^K}.
\end{array}
\end{equation*}
%
Since $\Rad^D_M(H)\le\frac{N}{M}\Rad^D_N(H)$ for any $M\le N$ and distribution $D$, the lemma is proved.
\end{proof}

Now we use Lemma \ref{jy} to prove Proposition \ref{jyc}:
\begin{proof}
Use reduction to absurdity. Assume that with probability at least $1-\delta$, there is $\E_{(x,y)\in \D_P}[h(x,y)]\le \frac{0.5-2\delta\alpha/Q-V\alpha}{\alpha-\delta/Q}$, and take the right-hand size value as $\epsilon$.

By Lemma \ref{jy}, take $K=N\alpha$. Then $\Rad_N^{\D_\S^{\ne l_p}}\ge (0.5-N\epsilon/K-\frac{(2-N\epsilon/K)\delta\alpha}{\eta^N\tau^K})/\alpha$.
We substitute $\epsilon$ in it. Then
\begin{equation*}
\begin{array}{ll}
&\Rad_N^{\D_\S^{\ne l_p}}\\
\ge&(0.5-N\epsilon/K-\frac{(2-N\epsilon/K)\delta\alpha}{\eta^N\tau^K})/\alpha\\
=&1/\alpha(0.5-\frac{2\delta\alpha}{Q}-\epsilon(\alpha-\delta/Q))\\
=&1/\alpha(0.5-\frac{2\delta\alpha}{Q}-\frac{0.5-2\delta\alpha/Q-V\alpha}{\alpha-\delta/Q}(\alpha-\delta/Q))\\
=&1/\alpha(0.5-\frac{2\delta\alpha}{Q}-(0.5-2\delta\alpha/Q-V\alpha))\\
=&1/\alpha(V\alpha)\\
=&V.
\end{array}
\end{equation*}
So  $\Rad_N^{\D_\S^{\ne l_p}}\ge V$, which is contradictory to (2) in Proposition \ref{jyc}, and the proposition is proved.
\end{proof}

\section{Optimality of the Generalization Bound}
\label{tl}
In this section, we show that for the general hypothesis space and   data distribution, the generalization gap between the empirical error and the population error cannot be smaller than $O(\frac{1}{\sqrt{N}})$, mentioned in Remark \ref{rem-optgb}.
%
%
%

\begin{proposition}
\label{lizi}
%
Let $m=2$ and the data distribution $\D_{\S}$ satisfy $P_{(x,y)\sim \D_S}(y=1)=P_{(x,y)\sim \D_S}(y=0)=0.5$.
Let $\Hyp=\{L(\F(x),y)\}$ be the hypothesis space, $L(\F(x),y)=\ID(\widehat\F(x)\ne y)$ the loss function,
and $\F_0$ a neural network classifying $\widehat{\F}_0(x)=0$ for $(x,y)\sim\D_{\S}$. 
Then for any $c>0$, $k>0.5$, and $\D_{\tr}$ i.i.d. sampled from $\D_{\S}$ with $|\D_{\tr}|=N$, we have $$\Pr(|\E_{(x,y)\sim\D_S}[L(\F_0(x),y)]-\E_{(x,y)\in\D_{\tr}}[L(\F_0(x),y)]|<\frac{c}{N^k})=O({c}{N^{0.5-k}}).$$
\end{proposition}
\begin{proof}
First, we have $\E_{(x,y)\sim\D_S}[L(\F_0(x),y)]=\E_{(x,y)\sim\D_S}[\ID(y=1)]=0.5$.
    Then we have $\E_{(x,y)\in\D_{\tr}}[L(\F_0(x),y)]=\E_{(x,y)\in\D_{\tr}}[\ID(y=1)]=N_1/N$, where $N_1$ is the number of $x$ with label 1 in $\D_{\tr}$.
    So, $\Pr(|\E_{(x,y)\sim\D_S}[L(\F_0(x),y)]-\E_{(x,y)\in\D_{\tr}}[L(\F_0(x),y)]|<\frac{c}{N^k})=\Pr(|0.5-N_1/N|<\frac{c}{N^k})$.

    To estimate $\Pr(|0.5-N_1/N|<\frac{c}{N^k})$, we only need to calculate the probability of $N_1\in(0.5N-cN^{1-k},0.5N+cN^{1-k})$.
    Since $\D_{\tr}$ is i.i.d. sampled from $\D_{\S}$, a sample labeled 1 is selected with probability 0.5. Thus, 
\begin{equation*}
\begin{array}{ll}
&\Pr(N_1\in(0.5N-cN^{1-k},0.5N+cN^{1-k}))\\
\le& \sum_{i=0.5N-cN^{1-k}}^{0.5N+cN^{1-k}}C_N^i0.5^N\\
\le&  2cN^{1-k} C_N^{N/2}0.5^N.\\
\end{array}
\end{equation*}
When $N\to\infty$, using Stirling's approximation to calculate the $C_N^{N/2}$, we have
\begin{equation*}
\begin{array}{ll}
&\Pr(N_1\in(0.5N-cN^{1-k},0.5N+cN^{1-k}))\\
\le& 2cN^{1-k} C_N^{N/2}0.5^N\\
=& 2cN^{1-k}0.5^NO(\frac{\sqrt{2\pi N}(N/e)^N}{\pi N(N/(2e))^N})\\
=&O(cN^{0.5-k}).
\end{array}
\end{equation*}
The proposition is proved.
\end{proof}

It is easy to see that $O({c}{{N^{0.5-k}}})\to 0$ when $N\to\infty$, so by Proposition \ref{lizi}, the generalization gap cannot be smaller than
$O(\frac{1}{\sqrt{N}})$. 
Together with Theorem \ref{fanhua}, we show the optimality of $O(\frac{1}{\sqrt{N}})$ of the generalization gap for a clean dataset.

This is also for the generalization bound under poison attacks, because the proof of Theorem \ref{fenbugj} need the generalization bound Theorem \ref{fanhua} for the dataset. 

%

\section{Proof of Theorem \ref{poifan}}
\label{a43}
We provide a more intuitive explanation on
how to estimate the poison generalization bound in Theorem \ref{poifan}, which is the core of our theorem.

%
%

Based on research on indiscriminate poisoning, neural networks always prioritize learning simple features, i.e., shortcut. So we try to make the trigger to be a shortcut (as said in condition (c3) in Theorem \ref{poifan} and Proposition \ref{xbx}).

However, only a few samples were poisoned in the backdoor attack, which is different from the setting of indiscriminate poisoning, so only using shortcut cannot establish the bound. Therefore, we aim to disrupt the original features of the poisoned images, such that the classification of the poisoned images and the classification of clean images become two independent tasks for the network (as said in condition (c1) in Theorem \ref{poifan}). By doing so, when the network completes the task of classifying poison data, the clean part of the data set is useless, and the network will learn the feature in the poison part of data, so that the shortcut can be learned. 

Finally, if the shortcuts are similar for different images, the shortcuts learned on a small portion of the data can be generalized to all the data (as said in condition (c2) in Theorem \ref{poifan}).

By the above idea, we will prove a more general form of Theorem \ref{poifan}, as shown below, and Theorem \ref{poifan} is an easy corollary of this theorem.

\begin{theorem}
\label{poifan-yiban}
Use the notation in Section \ref{backdoor-g}.
Let $N=|\D_{\tr}|$.
%
For any two hypothesis spaces $\Hyp\subset\Hyp_{W,D}$ and $F\subset\Hyp_{W,D}$, if $\p(x)$ satisfies the following conditions for some $\epsilon>0,\tau>0,\lambda\ge 1$:\\
(c1): $\E_{(x,y)\sim \D^{l_p}_{\S}}[f_{y}(x+\p(x))]\le\epsilon$ for all $f\in F$,\\
(c2): $\Pr_{(x,y)\sim \D_{\S}}(\p(x)\in A|y\ne l_p)\le\lambda $ $\Pr_{(x,y)\sim \D_{\S}}$ $(\p(x)\in A| y=l_p)$ for any set $A$, \\
(c3): some $h\in H$ satisfies that: $\exists f\in F$ such that $\E_{x\sim \D_{\S}}[|(h-f)_{l_P}(\p(x))-(h-f)_{l_p}(x+\p(x))|]\le\tau$, where $(h-f)_{l_P}(x)=h_{l_P}(x)-f_{l_p}(x)$,\\
then with probability at least $1-\delta-O(1/N)$, the following inequality holds for all $h\in H$ satisfying (c3): 
%
\begin{equation}
\label{eq-th43-yb}
\begin{array}{ll}
&\Risk_P(h,\D_{\S})\le \lambda O(\frac{1}{\alpha}(\E_{(x,y)\in \D_P}[L_{CE}(h(x),y)]
+\Rad_{N}^{\D_S^{l_p}}(\Hyp_{W,D,1}))
+\sqrt{\frac{\ln(1/\delta)
}{N\alpha}}+\epsilon+\tau+\frac{\lambda-1}{\lambda}).
\end{array}
\end{equation}
\end{theorem}

It is easy to see that we just need to take the hypothesis spaces $H,F$ in Theorem \ref{poifan-yiban} to be $H=\{\F(x)\}$ and $F=\{\G(x)\}$ (i.e. hypothesis space just contains only one network). Then Theorem \ref{poifan-yiban} naturally equivalent to Theorem \ref{poifan}.

\subsection{Proof of Theorem \ref{le-poifan1}}

We give the following theorem, which is a generalization bound theory under more general hypothesis space. 
\begin{theorem}
\label{le-poifan1}
%
%
For any two hypothesis spaces $H=\{h(x,y)\in\S\times[m]\to[0,1]\},F=\{f(x,y)\in\S\times[m]\to[0,1]\}$, if $\p(x)$ satisfies the following conditions for some $\epsilon>0,\tau>0,\lambda\ge 1$:\\
(1) $\E_{(x,y)\sim \D^{l_p}_{\S}}[f(x+\p(x),y)]\ge1-\epsilon$ for any $f\in F$,\\
(2) $\Pr_{(x,y)\sim \D_S^{\ne l_p}}(\p(x)\in A)\le \lambda \Pr_{(x,y)\sim \D^{l_p}_{\S}}(\p(x)\in A)$ for any set $A$, \\
(3) Some $h\in H$ satisfies that there exists an $f\in F$ such that $\E_{x\sim \D_{\S}}[|(h-f)(\p(x),l_p)-(h-f)(x+\p(x),l_p)|]\le\tau$, where $(h-f)(x)=h(x)-f(x)$,\\
then with probability at least $1-\delta-\frac{4-4\eta}{4-4\eta+\eta N}$, for any $h\in H$ satisfying (3), we have:
{\small
\begin{equation}
\begin{array}{l}
\E_{(x,y)\sim \D_S}[h(x+\p(x),l_p)]
\le \lambda(2\E_{(x,y)\in\D_{P}}h(x,y)/(\alpha\eta)+2\Rad^{\D_S^{l_p}}_{N\alpha\eta/2}(H)+\sqrt{\frac{\ln(2/\delta)
}{N\alpha\eta}}+\tau/\eta+\epsilon)
+\tau+(\lambda-1)\eta.
\end{array}
\end{equation}
}
\end{theorem}

We will use Theorem \ref{le-poifan1} to prove Theorem \ref{poifan-yiban} in Section \ref{ss2}.
Now, we prove Theorem \ref{le-poifan1} and give a detailed explanation of the ``proof idea'' of Theorem \ref{poifan}. 
Please note that, in the proof of Theorem \ref{le-poifan1}, the poison generalization error is $\E_{(x,y)\sim \D_S}[h(x+\p(x),l_p)]$.
\begin{proof}

The proof is divided into two parts.

{\bf Part one, we estimate the upper bound of $\E_{(x,y)\sim \D^{l_p}_\S}[h(x+\p(x),l_p)]$. This part corresponds to the ``firstly'' part in the proof idea shown under Theorem \ref{poifan}.}

We first give three bounds in the following Results (e1), (e2), and (e3). Since (e1), (e2), (e3) are similar to (d1), (d2), (d3) in the proof of Theorem \ref{fenbugj1}, we omitted the proof.

{\bf Result (e1)}: Let the random variable $X$ be the number of samples with label $l_p$ in $\D_\p$. Then with probability at least $1-\frac{4(1-\eta)}{4-4\eta+N\eta}$, it holds $X\ge N\eta/2$.

{\bf Result (e2)}: Now we even randomly select $N\eta\alpha/2$ poisoned samples in $\D_P$. If the number of poisoned samples in $\D_P$ is smaller than $N\eta\alpha/2$, then we let $D_{l_p}$ be the set of all such samples.

Let $\D_{l_p}$ obey the distribution $\D_{\S2}$. Then we have $\Pr_{x\sim \D_{\S2}}(x\in A|X\ge [N\eta/2])=P(X_{N\eta\alpha/2}^{\D^{l_p}_{\S}}\in A)$ for any set $A$, where $X_{N\eta\alpha/2}^{\D^{l_p}_{\S}}$ is the set that i.i.d. sampled $N\eta\alpha/2$ data from distribution $\D^{l_p}_{\S}$, and add $\p(x)$ to each data.

{\bf Result (e3)}: With probability $1-\frac{4-4\eta}{4-4\eta+N\eta}-\delta/2$, for any $h\in H$, we have  $$\E_{(x,y)\sim \D^{l_p}_\S}[h(x+\p(x),l_p)]\le\frac{\sum_{(x,y)\in \D_{poi}}h(x,y)}{N\alpha\eta/2}+2\Rad^{\D_S^{l_p}}_{N\alpha\eta/2}(H)+\sqrt{\frac{\ln(2/\delta)
}{N\alpha\eta}}.$$

We use $\D_{poi}=\{(x+\p(x),l_p)\|(x,l_p)\in \D_{sub}\}$, Result (e1), and Result (e2) to prove Result (e3).
The concrete steps are similar to that of Result (c3). $\D_{poi}$ and $\D_{sub}$ are defined in Section \ref{bgg}.

{\bf Part two, estimate the upper bound of $\E_{(x,y)\sim \D_S}[h(x+\p(x),l_p)]$.}

Please note that when $y\ne l_p$, $h(x+\p(x),l_p)$ will not appear in the empirical error $\E_{(x,y)\in \D_P}[h(x+\p(x),y)]$, so we need to use some other methods to estimate $\E_{(x,y)\sim \D_S}[h(x+\p(x),l_p)]$.

For any $h\in H$ and $f\in F$, let $c_{h,f}(x)=h(x,l_p)-f(x,l_p)$. Let $Q$ be the upper bound mentioned in Result (e3) in Part one.

The upper bound of $\E_{(x,y)\sim \D_S^{\ne l_p}}[h(x+\p(x),l_p)]$ will be given by the following Results (f1), (f2), (f3), and (f4). 
Note that (f1) and (f2) correspond to the ``Secondly'' step in the proof idea shown under the Theorem \ref{poifan}; (f3) and (f4) correspond to the ``Finally'' step in the proof idea shown under the Theorem \ref{poifan}.

{\bf Result (f1)}: If $f\in F$ and $h\in H$ satisfy (3), then we have  $\E_{(x,y)\sim \D^{l_p}_\S}[c_{h,f}(\p(x))]\le Q+\tau/\eta-(1-\bfepsilon)$.

By condition (3), we have: 
\begin{equation}
\label{yy1}
\begin{array}{ll}
&\tau\\
\ge&\E_{(x,y)\sim \D_S}[|c_{h,f}(\p(x))-c_{h,f}(x+\p(x))|](use\ (3))\\
\ge&\eta(\E_{(x,y)\sim \D^{l_p}_\S}[|c_{h,f}(\p(x))-c_{h,f}(x+\p(x))|])\\
\ge&\eta(\E_{(x,y)\sim \D^{l_p}_\S}[c_{h,f}(\p(x))-c_{h,f}(x+\p(x))]).\\
\end{array}
\end{equation}
Now by condition (1) and Result (e3),   with probability $1-\delta$, we have 
\begin{equation}
\label{yy2}
\begin{array}{ll}
&\E_{(x,y)\sim \D^{l_p}_\S}[c_{h,f}(x+\p(x))]\\
=& \E_{(x,y)\sim \D^{l_p}_\S}[h(x+\p(x),l_p)]-\E_{(x,y)\sim \D^{l_p}_\S}[f(x+\p(x),l_p)]\\
 \le&  Q-(1-\epsilon).
\end{array}
\end{equation}
Combine inequalities \eqref{yy1} and \eqref{yy2}, we have: 
\begin{equation}
\label{yy3}
\begin{array}{ll}
&\E_{(x,y)\sim \D^{l_p}_\S}[c_{h,f}(\p(x))]\\
\le& \E_{(x,y)\sim \D^{l_p}_\S}[c_{h,f}(x+\p(x))]+\tau/\eta\\
\le& Q-(1-\epsilon)+\tau/\eta.
\end{array}
\end{equation}
Result (f1) is proved.

{\bf Result (f2)}:  $\E_{(x,y)\sim \D_S^{\ne l_p}}[c_{h,f}(\p(x))+1]\le \lambda \E_{(x,y)\sim \D^{l_p}_\S}[c_{h,f}(\p(x))+1]$.

Result (f2) can be proved by using Lemma \ref{mjy} and condition (2).

{\bf Result (f3)}: When $h\in H$ and $f\in F$ satisfy condition (3), we have  $\E_{\D_S}[c_{h,f}(x+\p(x))]\le (\eta+(1-\eta)\lambda)(Q+\epsilon+\tau/\eta-1)+\lambda-1+\tau$.

By condition (3), we have
\begin{equation*}
\begin{array}{ll}
&\tau\\
\ge&\E_{(x,y)\sim \D_S}[|c_{h,f}(\p(x))-c_{h,f}(x+\p(x))|]\\
\ge&\E_{(x,y)\sim \D_S}[-c_{h,f}(\p(x))+c_{h,f}(x+\p(x))]\\
\end{array}.
\end{equation*}
Then, we have $\E_{\D_S}[c_{h,f}(x+\p(x))]\le \E_{\D_S}[c_{h,f}(\p(x))]+\tau$. Now, substitute Results (f1) and (f2) in it to estimate $\E_{\D_S}[c_{h,f}(\p(x))]$, and we can get  
\begin{equation*}
\begin{array}{ll}
&\E_{(x,y)\sim \D_S}[c_{h,f}(\p(x))]\\
=&\eta\E_{(x,y)\sim \D^{l_p}_\S}[c_{h,f}(\p(x))]+(1-\eta)\E_{(x,y)\sim \D_S^{\ne l_p}}[c_{h,f}(\p(x))]\\
\le&(\eta+(1-\eta)\lambda)\E_{(x,y)\sim \D^{l_p}_\S}[c_{h,f}(\p(x))]+\lambda-1\\
\le&(\eta+(1-\eta)\lambda)(Q+\epsilon+\tau/\eta-1)+\lambda-1.
\end{array}
\end{equation*}
So, we prove Result (f3): $\E_{\D_S}[c_{h,f}(x+\p(x))]\le(\eta+(1-\eta)\lambda)(Q+\epsilon+\tau/\eta-1)+\lambda-1+\tau$.



{\bf Result (f4)}:  $\E_{(x,y)\sim \D_S}[h(x+\p(x),l_p)]\le \lambda(Q+\tau/\eta+\epsilon)+\tau+(\lambda-1)\eta$.

Since $c_{h,f}(x+\p(x))=h(x+\p(x),l_p)-f(x+\p(x),l_p)$ and $f(x+\p(x),l_p)\le 1$, using Result (f3), we have
\begin{equation*}
\begin{array}{ll}
&\E_{(x,y)\sim \D_S}[h(x+\p(x),l_p)]\\
\le &\E_{(x,y)\sim \D_S}[c_{h,f}(x+\p(x),l_p)]+1\\
\le&(\eta+(1-\eta)\lambda)(Q+\epsilon+\tau/\eta-1)+\lambda+\tau\\
\le&\lambda(Q+\tau/\eta+\epsilon)+\tau+(\lambda-1)\eta.
\end{array}
\end{equation*}
This proves Result (f4).

When $h$ satisfies condition (3),  using the value of $Q$ into the Result (f4), we see that with probability $1-\delta-\frac{4-4\eta}{4-4\eta+\eta N}$, we have: 
$$\E_{(x,y)\sim \D_S}[h(x+\p(x),l_p)]\le \lambda(\frac{\sum_{(x,y)\in \D_{poi}}h(x,y)}{N\alpha\eta/2}+2\Rad^{\D_S^{l_p}}_{N\alpha\eta/2}(H)+\sqrt{\frac{\ln(2/\delta)
}{N\alpha\eta}}+\tau/\eta+\epsilon)+\tau+(\lambda-1)\eta.$$

Finally, using the facts $\D_{poi}\subset \D_P$ and Result (e3) is valid for $X=|\D_{poi}|\ge [N\eta/2]$, we have  
$$\E_{(x,y)\sim \D_S}[h(x+\p(x),l_p)]\le \lambda(\frac{2\E_{(x,y)\in\D_P}[h(x,y)]}{\alpha\eta}+2\Rad^{\D_S^{l_p}}_{N\alpha\eta/2}(H)+\sqrt{\frac{\ln(2/\delta)
}{N\alpha\eta}}+\tau/\eta+\epsilon)+\tau+(\lambda-1)\eta.$$
The theorem is proved.
\end{proof}

\subsection{Proof of Theorem \ref{poifan-yiban}}
\label{ss2}
Now, we use Theorem \ref{le-poifan1} to prove Theorem \ref{poifan-yiban}.

First, for any $h:\R^n\to\R^m$, we define $h_{-1}(x,y)=1-h_{y}(x):\R^n\times[m]\to\R$, and for any $H\subset \Hyp_{W,D}$, we define the hypothesis space: $H_{-1}=\{h(x,y)=1-h_{y}(x)\|h(x)\in H\}$. Using Theorem \ref{le-poifan1}, we have the following lemma.
\begin{lemma}
\label{le-poifan2}
%
%
For any two hypothesis spaces $H,F\subset\Hyp_{W,D}$, if $\p(x)$ satisfies the following conditions for some $\epsilon>0,\tau>0,\lambda\ge 1$:\\
(1) $\E_{(x,y)\sim \D^{l_p}_{\S}}[f_y(x+\p(x))]\le\epsilon$ for any $f\in F$;\\
(2) $\Pr_{(x,y)\sim \D_S^{\ne l_p}}(\p(x)\in A)\le \lambda \Pr_{(x,y)\sim \D^{l_p}_{\S}}(\p(x)\in A)$ for any set $A$, where $\lambda\ge 1$; \\
(3) Some $h\in H$ satisfies that there exists an $f\in F$ such that $\E_{x\sim \D_{\S}}[|(h_{l_p}-f_{l_p})(\p(x))-(h_{l_p}-f_{l_p})(x+\p(x))|]\le\tau$, where $(h_{l_p}-f_{l_p})(x)=h_{l_p}(x)-f_{l_p}(x)$,\\
then with probability at least $1-\delta-\frac{4-4\eta}{4-4\eta+\eta N}$, for any $h\in H$ satisfied (3), we have:
{\small
\begin{equation}
\begin{array}{l}
\E_{(x,y)\sim \D_S}[1-h_{l_p}(x+\p(x))]
\le \lambda(2\E_{(x,y)\in\D_{P}}[1-h_{y}(x)]/(\alpha\eta)+2\Rad^{\D_S^{l_p}}_{N\alpha\eta/2}(H_{-1})+\sqrt{\frac{\ln(2/\delta)
}{N\alpha\eta}}+\tau/\eta+\epsilon)\\
+\tau+(\lambda-1)\eta.
\end{array}
\end{equation}
}
\end{lemma}
\begin{proof}
We can use Theorem \ref{le-poifan1} to $H_{-1}$ and $F_{-1}$ to prove the lemma. 
We just need to verify that the three conditions in Theorem \ref{le-poifan1} for $H_{-1}$ and $F_{-1}$.

{\bf Verify condition (1) in Theorem \ref{le-poifan1}.}

We just need to show that $\E_{(x,y)\sim \D^{l_p}_{\S}}[f_{-1}(x+\p(x),y)]\ge1-\epsilon$ for any $f_{-1}\in F_{-1}$.

By condition (1) in Lemma \ref{le-poifan2}, and considering that $f_{-1}(x,y)=1-f_{y}(x)$ for the corresponding $f\in F$, we get the result.

{\bf Verify condition (2) in Theorem \ref{le-poifan1}.}

This is obvious because condition (2) in Theorem \ref{le-poifan1} and Lemma \ref{le-poifan2} are the same.

{\bf  Verify condition (3) in Theorem \ref{le-poifan1}.}

We just need to show that: Some $h_{-1}\in H_{-1}$ satisfies the requirement that there exists an $f_{-1}\in F_{-1}$ such that $\E_{x\sim \D_{\S}}[|(h_{-1}-f_{-1})(\p(x),l_p)-(h_{-1}-f_{-1})(x+\p(x),l_p)|]\le\tau$, where $(h_{-1}-f_{-1})(x,y)=h_{-1}(x,y)-f_{-1}(x,y)$.

Since $(h_{-1}-f_{-1})(x,l_p)=f_{l_p}(x)-h_{l_p}(x)$ for the corresponding $h\in H$ and $f\in F$,  by condition (3) in Lemma \ref{le-poifan2}, we get the result.

Since the three conditions in Theorem \ref{le-poifan1} stand for $H_{-1}$ and $F_{-1}$, the lemma is proved.
\end{proof}

Second, we give three lemmas.
\begin{lemma}
    \label{ysxd}
    For any $h\in\Hyp_{W,D}$, we have
    $\E_{(x,y)\sim \D_S}[\ID(\widehat{h}(x+\p(x))\ne l_p)]\le2\E_{(x,y)\sim \D_S}[1-h_{l_p}(x+\p(x))]$.
\end{lemma}
\begin{proof}
    When $\widehat{h}(x+\p(x))\ne l_p$, we have $h_{l_p}(x+\p(x))<0.5$, which implies that $0.5*\ID(\widehat{h}(x+\p(x))\ne l_p)\le 1-h_{l_p}(x+\p(x))$.
    Then, $\E_{(x,y)\sim \D_S}[0.5*\ID(\widehat{h}(x+\p(x))\ne l_p)]\le \E_{(x,y)\sim \D_S}[1-h_{l_p}(x+\p(x))]$, this is what we want.    
\end{proof}

\begin{lemma}[\citet{mohri2018foundations}]
\label{ysxd1}
For any hypothesis space $F=\{f:\R^n\to[0,1]\}$ and distribution $\D$, $N>0$. Let $F_1=\{1-f\|f\in F\}$. Then $\Rad_N^\D(F_{-1})=Rad_N^\D(F)$.
\end{lemma}

\begin{lemma}
    \label{ysxd2}
    For any $x\in(0,1]$, we have $1-x\le -\ln x$. 
\end{lemma}
\begin{proof}
Let $f(x)=1-x+\ln x$, then $f'(x)=1/x-1\ge 0$ when $x\in[0,1]$.
Because $f(1)=0$, we have $f(x)\le 0$ for all $x\in(0,1]$
\end{proof}

Finally, we use Lemmas \ref{le-poifan2}, \ref{ysxd}, \ref{ysxd1}, and \ref{ysxd2} to prove Theorem \ref{poifan-yiban}.
\begin{proof}
Firstly, it is easy to see that, conditions (1), (2), (3) in Lemma \ref{le-poifan2} are the same as (c1), (c2), (c3) in Theorem \ref{poifan-yiban}. So by Lemma \ref{le-poifan2}, we have 
    $$\E_{(x,y)\sim \D_S}[1-h_{l_p}(x+\p(x))]
\le \lambda(2{\E_{(x,y)\in\D_{P}}[1-h_{y}(x)]/(\alpha\eta)}+2\Rad^{\D_S^{l_p}}_{N\alpha\eta/2}(H_{-1})+\sqrt{\frac{\ln(2/\delta)
}{N\alpha\eta}}+\tau/\eta+\epsilon)\\
+\tau+(\lambda-1)\eta.$$
Then, by Lemma \ref{ysxd}, we further have 
$$
\begin{array}{l}
\E_{(x,y)\sim \D_S}[\ID(h_{l_p}(x+\p(x))\ne l_p)]
\le \\
2\lambda(2{\E_{(x,y)\in\D_{P}}[1-h_{y}(x)]/(\alpha\eta)}+2\Rad^{\D_S^{l_p}}_{N\alpha\eta/2}(H_{-1})+\sqrt{\frac{\ln(2/\delta)
}{N\alpha\eta}}+\tau/\eta+\epsilon)
+2\tau+2(\lambda-1)\eta.
\end{array}
$$

We have $H_{W,D,1}=\{\F_y(x)\|\F\in\Hyp_{W,D}\}$ and $H\subset\Hyp_{W,D}$. 
Then, by Lemma \ref{ysxd1} and considering that under distribution $\D_S^{l_p}$, all samples have label $l_p$, so we have $\Rad^{\D_S^{l_p}}_{N\alpha\eta/2}(H_{-1})\le\Rad^{\D_S^{l_p}}_{N\alpha\eta/2}(H_{W,D,1})$.

Finally, using Lemma \ref{ysxd2} and the fact: $\Rad^D_M(H)\le\frac{N}{M}\Rad^D_N(H)$ for any $M\le N$, distribution $\D$, and hypothesis space $H$, we obtain
$$
\begin{array}{l}
\E_{(x,y)\sim \D_S}[\ID(h_{l_p}(x+\p(x))\ne l_p)]
\le \\
2\lambda(\frac{2{\E_{(x,y)\in\D_{P}}[L_{CE}(f(x),y)]}}{\eta\alpha}+4\Rad^{\D_S^{l_p}}_{N}(H_{W,D,1})/(\alpha\eta)+\sqrt{\frac{\ln(2/\delta)}{N\alpha\eta}}+\tau/\eta+\epsilon)
+2\tau+2(\lambda-1)\eta.
\end{array}
$$
The theorem is proved.
\end{proof}

\subsection{Estimate the Rademacher Complexity}
\label{tvor}
In this section, we estimate
$\Rad^{\D_S^{l_p}}_{N\alpha\eta/2}(H_{W,D,1})$. Since all samples have label $l_p$ in distribution $\D_S^{l_p}$, without loss of generality, we let $l_p=1$. In this section, we only need to consider the Radmacher complexity of the following hypothesis space $\Hyp_{W,D,0}=\{\F_1(x):\F(x)\in \H_{W,D}\}$.

We show that, if we bound the norm of the network parameters in $\Hyp_{W,D,0}$, then we can calculate $\Rad_N^\D(\Hyp_{W,D,0})$ for any distribution $\D$. Please note that the condition of bounded network parameters is reasonable.

Some definitions and a lemma are required.
\begin{definition}
Let $F=\{f:\R^n\to[0,1]\}$ be a hypothesis space, and $\{x_i\}_{i=1}^N$ be $N$ samples in $\R^n$. Then the empirical Rademacher Complexity of $\F$ under $\{x_i\}_{i=1}^N$ is defined as
    $$\Rad^{\{x_i\}_{i=1}^N}(F)= \E_{\sigma}[\sup_{f\in F}\frac1N\sum_{i=1}^N\sigma_if(x_i)],$$
    where $\sigma=(\sigma_i)_{i=1}^N$ is a set of random variables such that $\Pr(\sigma_i=1)=\Pr(\sigma_i=-1)=0.5$.
\end{definition}
It is easy to see that $\Rad^{\D}_N(\F)\le \max_{\{x_i\}_{i=1}^N\sim\D}\Rad^{\{x_i\}_{i=1}^N}(\F)$, so we can try to calculate the empirical Rademacher Complexity to bound $\Rad^{\D}_N(\F)$.

\begin{definition}[Covering Number,\cite{wainwright2019high}]
Let $(T,L)$ be a metric space, $T$ be a space, and $L$ be the distance in $T$. We say that a $K\subset T$ is an $\epsilon$ cover of $T$, if for any $x\in T$, there is a $y\in K$, such that $L(x,y)\le\epsilon$. The minimum $|K|$ is defined as $C(K,L,\epsilon)$, that is $C(K,L,\epsilon)=\min |K|$ where $K\subset T$ is an $\epsilon$ cover of $T$.
%
\end{definition}

\begin{lemma}[\cite{wainwright2019high}]
\label{calrea}
Let $F=\{f:\R^n\to[0,1]\}$ be a hypothesis space, and $\{x_i\}_{i=1}^N$ be $N$ samples in $\R^n$. Define $L(f,g)=\sqrt{1/N\sum_{i=1}^N(f(x_i)-g(x_i))^2}$ for any $f,g\in F$. Then
    $$\Rad^{\{x_i\}_{i=1}^N}(F)\le O(\int_{0}^1\sqrt{\frac{\ln C(F,L,t)}{N}}dt).$$
\end{lemma}

When network parameters are bounded, Lemma \ref{calrea} is often used to calculate the Rademacher complexity. A classical result is given  below.
\begin{lemma}
\label{jdg}
Let $T$ be a ball with radius $r$ in $\R^p$. Then for any $t$, we have  $C(T,L_{o},t)\le(\frac{3r}{t})^p$, where $L_o$ is the Euclid distance.
\end{lemma}

Now, we will try to calculate the covering number of $\Hyp_{W,D,0}$. First, we give a definition of the bound of the network parameters and the relationship between the bound of the network parameters and the network output.

\begin{definition}
Let $\F_i(i=1,2)$ be two networks, the $j$-th transition matrix and bias vector are $W^i_j$ and $b^i_j$. Then let $B(\F_i)=\sum_{j=1}^{D_i}(||W_j^i||_2+||b_j^i||_2)$, where $D_i$ is the depth of $\F_i$.
If $D_1=D_2=D$ and the widths of $\F_i$ are the same, then we can define $B(\F_1-\F_2)=\sum_{j=1}^D(||W_j^1-W_j^2||_2+||b_j^1-b_j^2||_2)$.
\end{definition}

Then, we have the following existing result.
\begin{lemma}[\citet{tsai2021formalizing}]
\label{cnnum1}
For networks $\F_1$ and $\F_2$ with depth $D$, with Relu activation function, and output layers do not have activation function. 
If $B(\F_1)\le C$, $B(\F_2)\le C$, and $B(\F_1-\F_2)\le \epsilon$, then $||\F_1(x)-\F_2(x)||_2\le D\epsilon C^D||x||_2$ for any $x$.
\end{lemma}
It is easy to show that the Softamx function is a Liptschitz function:
\begin{lemma}
\label{cnnum2}
 If $a,b\in\R^m$, then 
 $|\Softmax(a)_1-\Softmax(b)_1|\le \sqrt{m}||a-b||_2$, 
 where $\Softmax(a)_1$ is the first weight of $\Softmax(a)$.
\end{lemma}

So using Lemmas \ref{cnnum1} and \ref{cnnum2}, we have 
\begin{lemma}
\label{uxx1}
    For networks $\F_1,\F_2\in\Hyp_{W,D}$ and $\F_i\in\R^n\to\R^m$. If $B(\F_1)\le C$, $B(\F_2)\le C$ and $B(\F_1-\F_2)\le \epsilon$, then $||\F_1(x)-\F_2(x)||_2\le \sqrt{mn}D\epsilon C^D$ for any $x\in[0,1]^n$. 
\end{lemma}

Now, we calculate the covering number of $\Hyp_{W,D,0}$.
\begin{lemma}
\label{uxx}
$L(f,g)$ is defined in Lemma \ref{calrea}. Let $\Hyp^A_{W,D,0}=\{\F:\F\in\Hyp_{W,D,0},B(\F)\le A\}$.
Then for any $t>0$, we have  $C(\Hyp^A_{W,D,0},L,t)\le(3A^{D+1}D\sqrt{mn}/t)^{O(W^2D)}$.
\end{lemma}
\begin{proof}
By Lemma \ref{uxx1}, when $\F_1,\F_2\in\Hyp^A_{W,D,0}$, if $B(\F_1-\F_2)\le \frac{t}{DA^D\sqrt{mn}}$, then we have  $||\F_1(x)-\F_2(x)||_2\le t$, which implies that $L(\F_1,\F_2)\le t$. So we just need to minimize the number of the $\frac{t}{DA^D\sqrt{mn}}$ cover of the parameter space. Using Lemma \ref{jdg}, we get the result.    
\end{proof}
Finally, using Lemmas \ref{calrea} and \ref{uxx}, we can calculate the Rademacher Complexity.
\begin{lemma}
Let $\Hyp^A_{W,D,0}=\{\F:\F\in\Hyp_{W,D,0},B(\F)\le A\}$ and $ADmn\ge e$.
Then $\Rad_{N}^\D(\Hyp^A_{W,D,0})\le {\frac{O(WD\ln(ADmn))}{\sqrt{N}}}$ for any distribution $\D$, which implies that $\Rad_{N}^\D(\Hyp^A_{W,D,0})$ approaches to 0 when $N$ is big enough.
\end{lemma}
\begin{proof}
We just need to prove that, for any $N$ points $\{x_i\}_{i=1}^N$ in $[0,1]^N$, we have  $\Rad^{\{x_i\}_{i=1}^N}(\Hyp^A_{W,D,0})\le  {\frac{O(WD\ln(ADmn))}{\sqrt{N}}}$.

Using Lemmas \ref{calrea} and \ref{uxx}, we have  $\Rad^{\{x_i\}_{i=1}^N}(\Hyp^A_{W,D,0})\le O(\int_{0}^1\sqrt{\frac{O(W^2D^2\ln(ADmn/t))}{N}}dt)\le O(\int_{0}^1\frac{O(WD\ln(ADmn/t))}{\sqrt{N}}dt)=\frac{O(WD\ln(ADmn))}{\sqrt{N}}$. Here, we use $\sqrt{\ln(q/t)}\le ln q/t$ for all $q\ge e$ and $t\in(0,1)$.
\end{proof}









\section{Proof of Proposition \ref{xbx}}
\label{app-44}

\subsection{Strict }
\label{se}




We first give a precise version of the proposition and definition in Section \ref{sce}:

\begin{definition}[Binary Shortcut]
\label{dyd1}
$\p(x)$ is called a binary shortcut of the binary linear inseparable classification dataset $\D=\{(x_i,1)\}_{i=1}^{N_{1}}
\cup\{(\widehat{x}_i,0)\}_{i=1}^{N_{0}}$,  
if   $\D_1=\{(x_i,1)\}_{i=1}^{N_{1}}
\cup\{(\hat{x}_i+\p(\widehat{x}_i),0)\}_{i=1}^{N_{0}}$ is linear separable.
%
Moreover, if there exists a linear function $h$ with a unit normal vector and $0<\eta_1<0.5$ such that $h(x)\ge1-\eta_1$ for any $(x,0)\in \D_1$ and  $h(x)\le\eta_1$ for any $(x,1)\in \D_1$.
Then we say that $\p(x)$ is a {\em binary shortcut of $\D$} with bound $\eta_1$.
\end{definition}

{\bf The strict version of the definition for the Simple Feature Recognition Space:}
It is generally believed that networks learn simple features, and based on this characteristic, adding shortcut to dataset affects network training, so we use the following definition to describe this property:

\begin{definition}[Simple Features Recognition Space]
\label{sim1}
We say that $\Hyp$ is a {\em simple feature recognition space} with a constant $c$, if for any binary linear inseparable classification dataset $\D$ and a binary shortcut $\p(x)$ of $\D$ with bound $\eta_1$, $\Hyp$ satisfies the following properties: Let $k=\max_{(x,0),(z,0)\in D}\{||\p(x)-\p(z)||_2\}$.
Then for any $h\in \Hyp$ that satisfies $h(x+\p(x))\ge 1-\eta_1,\forall (x,0)\in \D$ and $h(x)\le \eta_1,\forall (x,1)\in \D$, it holds $h(x_1+\p(x_0))-h(x_1)\ge c(1-2\eta_1-k)$ and $h(\p(x_0))\ge c(1-2\eta_1-k)$ for any $(x_0,0)\in \D$ and $(x_1,y_1)\in\D$.
\end{definition}

As mentioned above, indiscriminate poison can be considered as a shortcut, and we have two important conclusions that have been proved in the study of indiscriminate poison \cite{zhu2023detection}: (1) The network trained on dataset with indiscriminate poison will classify shortcut and 
(2) Adding shortcut to samples will affect the output of network.
%
The above definition mainly uses mathematical methods to describe these two conclusions: we express that ``network will classify shortcut'' by giving a lower bound to $h(\p(x_0))$; we express ``shortcut effects the classification results'' by giving lower bound to $h(x_1+\p(x_0))-h(x_1)$.
%
%
%
Moreover, we have considered the impact of differences in $\p(x)$ for different $x$ (i.e., $k$ in definition) on the output: the larger the difference, the smaller the impact. Thus, our definition is valid for some spaces, as shown below. 
%

\begin{proposition}
\label{zjz1}
Let $L$ be the set of linear functions with a normal unit vector and without bias. Then $L$ is a simple feature recognition space with constant $1$.
Furthermore, if $f:\R\to\R$ is an increasing differentiable function with derivatives in $[a,1]$   and $f(0)=0$, then the hypothesis space $\Hyp=\{f(l(x))|l\in L\}$ is a simple feature recognition space with constant $a$.
\end{proposition}

{\bf The strict version of Proposition \ref{xbx}:}
Using the above definitions, we can show how condition (c3) in Theorem \ref{poifan} stands for a small $\tau$, the strict version of Proposition \ref{xbx} is given below.

\begin{proposition}
\label{xbx1-q}
Use notation introduced in Theorem \ref{fenbugj}.
For any distribution $\D$, let $\D(\p)$ be the distribution of $\p(x)$ when $x\sim \D$, and $\D(\p+x)$ be the distribution of $x+\p(x)$ when $x\sim \D$. $\Hyp_{W,D,1}$ is defined in Section \ref{backdoor-g}. 
Define $\Rad2(\Hyp_{W,D})$ as: 
\begin{equation}
\label{Rad2-eq}
\begin{array}{ll}
&\Rad2(\Hyp_{W,D})=\\
&O(\Rad_{N(1-\eta)}^{\D_S^{\ne l_p}(\p)}(\Hyp_{W,D,1})+\Rad_{N(1-\eta)}^{\D_S^{\ne l_p}(x+\p)}(\Hyp_{W,D,1})+\Rad_{N\eta\alpha}^{D^{l_p}_\S(\p)}(\Hyp_{W,D,1})+\Rad_{N\eta\alpha}^{D^{l_p}_\S(x+\p)}(\Hyp_{W,D,1}))
\end{array}
\end{equation}

Let $D'_P=\{(x,0)|(x,y)\in\D_{\tr}\setminus\D_{clean}\}\cup\{(x,1)|(x,y)\in\D_{clean}\}$. Assume that with probability $1-\delta_1$, $D'_P$ is linear inseparable.
Let the hypothesis spaces $H,F\subset\Hyp_{W,D}$ satisfy that $h_y(x)$ and $f_y(x)$ have Lipschitz constant $L$ for any $h\in H$, $f\in F$ and any $x$, $y$. Assume that trigger $\p(x)$ satisfies the following three conditions for some $\epsilon,k>0$:\\
(t1) For any $f\in F$, we have $f_y(x+\p(x))<\epsilon,\forall(x,y)\sim\D_\S$;\\
(t2) $||\p(x_1)-\p(x_2)||_2\le k$ for all $x_1,x_2$;\\
(t3) $\p(x)$ is the shortcut of the dataset $\D'_P$ with bound $2\epsilon$.\\
When the hypothesis space $H-F=\{h_{l_p}(x)-f_{l_p}(x)\in\R^m\to\R|f\in F,h\in H\}$ is a simple feature recognition space with constant $c$ where $c$ satisfies $1-2\epsilon\ge c(1-4\epsilon)+k(c+4L)\ge2\epsilon$. Then with probability $1-\delta_1-\delta-\frac{4(1-\eta)}{4(1-\eta)+N\eta}-\frac{4\eta}{4\eta+N(1-\eta)}$, for any 
$h\in \{h\in H |h_y(x)>1-\epsilon,\forall(x,y)\in\D_P\}$ and any $f\in \{f\in F |f_y(x)>1-\epsilon,\forall(x,y)\in\D_{\tr}\}$, we have

{\small
\begin{equation}
\label{y1z}
\begin{array}{ll}
&\E_{x\sim \D_{\S}}[|(f-h)_{l_p}(\p(x))-(f-h)_{l_p}(x+\p(x))|]\\
\le &2(1-c(1-4\epsilon)+k(c+4L))+2\epsilon)+\Rad2(\Hyp_{W,D})+16\sqrt{\frac{\ln(1/\delta)}{N\alpha}}.
\end{array}
\end{equation}}
\end{proposition}
It is easy to see that, when $c$ close to 1 and $k$ is small, such value is close to $\widetilde O(\epsilon)$.
\begin{remark}
    Notice that (t1) is similar to (c1) in Theorem \ref{poifan}, which means the trigger is adversarial; (t2) is similar to (c2) in Theorem \ref{poifan}, which means trigger should be similar for different samples. And when $k$ tends to 0, $c$ tends to 1, and $N$ is big enough, it holds that $\E_{x\sim \D_{\S}}[|(f-h)_{l_p}(\p(x))-(f-h)_{l_p}(x+\p(x))|]$ tends to $O(\epsilon)$. Generally, we can consider that the hypothesis space $F$ only contains the network that performs well on distribution $\D_\S$. Then, based on the transferability of adversarial examples, condition (t1) can be established.
\end{remark}


\subsection{Prove Proposition \ref{xbx1-q}}
We first prove the following more general proposition, where the hypothesis spaces are replaced with more general hypothesis spaces, which will imply Proposition \ref{xbx1-q}.


\begin{proposition}
\label{xbx1}
Use notation introduced in Theorem \ref{fenbugj1}. 
Let $D'_P=\{(x,0)|(x,y)\in\D_{\tr}\setminus\D_{clean}\}\cup\{(x,1)|(x,y)\in\D_{clean}\}$. 
Assume that with probability $1-\delta_1$, $D'_P$ is linear inseparable.
Let the hypothesis spaces $H=\{h(x,y):\R^n\times[m]\to[0,1]\}$ and $F=\{f(x,y):\R^n\times[m]\to[0,1]\}$ satisfy $h(x,y_1)+h(x,y_2)\ge1$, $h(x,y_1)+h(x,y_2)\ge1$, $h(x,y_1)$ and $f(x,y_1)$ have Lipschitz constant $L$ about $x$ for any $h\in H$, $f\in F$,  $x\in\R^n$ and $y_1\ne y_2$. Assume that trigger $\p(x)$ satisfies the following three conditions for some $\epsilon,k>0$:\\
(t1) For any $f\in F$, we have $f(x+\p(x),y)>1-\epsilon,\forall(x,y)\sim\D_\S$;\\
(t2)$||\p(x_1)-\p(x_2)||_2\le k$ for all $x_1,x_2\in\R^n$;\\
(t3) $\p(x)$ is the shortcut of the dataset $\D'_P$ with bound $2\epsilon$.\\
When the hypothesis space $F-H=\{g_{f,h}(x)=f(x,l_p)-h(x,l_p)\in\R^m\to\R|f\in F,h\in H\}$ is a simple feature recognition space with constant $c$ where $c$ satisfies $1-2\epsilon\ge c(1-4\epsilon)-(c+4L)k\ge2\epsilon$. Then with probability $1-\delta_1-\delta-\frac{4(1-\eta)}{4(1-\eta)+N\eta}-\frac{4\eta}{4\eta+N(1-\eta)}$, for any 
$h\in \{h\in H |h(x,y)<\epsilon,\forall(x,y)\in\D_P\}$ and any $f\in \{f\in F | f(x,y)<\epsilon,\forall(x,y)\in\D_{\tr}\}$, we have

{\small
\begin{equation}
\label{yz}
\begin{array}{ll}
&\E_{x\sim \D_{\S}}[|(f-h)(\p(x),l_p)-(f-h)(x+\p(x),l_p)|]\\
\le &2(1-c(1-4\epsilon)+(c+4L)k+2\epsilon)+\Rad(H,F)+16\sqrt{\frac{\ln(1/\delta)}{N\alpha}}.
\end{array}
\end{equation}}
Here $\Rad(H,F)$ is a value depending on the Rademacher complexity of $H$ and $F$, and the specific value of it is explained in the following proof.
\end{proposition}

    
\begin{proof}





Let $\D(\p)$ be the distribution of $\p(x)$ where $x\sim \D$, and $\D(\p+x)$ be the distribution of $x+\p(x)$ where $x\sim \D$. We define $g_{f,h}(x)=f(x,l_p)-h(x,l_p)$.

Next, under ``$\D'_P$ is linear inseparable'', we will prove the inequality \eqref{yz}  with probability $1-\delta-\frac{4(1-\eta)}{4(1-\eta)+N\eta}-\frac{4\eta}{4\eta+N(1-\eta)}$, which will directly lead to the conclusion of the proposition.

To prove this, we just used the following three results:

 

{\bf Result one: If $h(x,y)<\epsilon$ for any $x\in \D_P$ and $f(x,y)<\epsilon$ for any $x\in \D_{\tr}$, then we have $|g_{f,h}(\p(x))-g_{f,h}(\p(x)+x)|\le 1-c(1-4\bfepsilon)+(c+4L)k$ for any $(x,y)\in \D_{\tr}\setminus\D_{clean}$ and $|g_{f,h}(\p(x))-g_{f,h}(\p(x)+x)|\le 1+2\epsilon-c(1-4\bfepsilon)+(c+4L)k$ for any $(x,y)\in \D_{clean}$.}

It is easy to see that when $h(x,y)<\epsilon$ for any $x\in \D_P$ and $f(x,y)<\epsilon$ for any $x\in \D_{\tr}$, we have  $|f(x,l_p)-h(x,l_p)|\le2\epsilon,\forall(x,y)\in\D_{clean}$,
where  we use $h(x,y_1)+h(x,y_2)\ge 1$ and $f(x,y_1)+f(x,y_2)\ge 1$ for any $y_1\ne y_2$. Since $f(x+\p(x),y)>1-\epsilon$ for any $x\in \D_{\tr}$, so $-h(x,l_p)+f(x,l_p)\ge1-2\epsilon,\forall(x,y)\in\D_{poi}$. Then we get $g_{f,h}(x)\ge1-2\epsilon,\forall(x,y)\in\D_{poi}$ and $|g_{f,h}(x)|\le2\epsilon,\forall(x,y)\in\D_{clean}$.

Because $F-H$ is a Simple Features recognition space, with conditions (t2) and (t3), we know that $g_{f,h}(\p(x))\ge  c(1-4\epsilon-k)$ and $g_{f,h}(x+\p(x))-g_{f,h}(x)\ge c(1-4\epsilon-k)$ for all $x\in \D_{\tr}\setminus\D_{clean}$.

Considering that $h(x,l_p)$ and $f(x,l_p)$ have Lipschitz constant $L$, and by condition (t2), we have $||\p(x_1)-\p(x_2)||_2\le k$ for all $(x_1,y_1),(x_2,y_2)\in\D'_P$. Then we have $g_{f,h}(\p(x))\ge c(1-4\epsilon)-(c+4L)k$ and $g_{f,h}(x+\p(x))-g_{f,h}(x)\ge c(1-4\bfepsilon)-(c+4L)k$ for all $x\in \D_{\tr}$.

Using the above result, when $x\in \D_{\tr}\setminus\D_{clean}$, we have $1\ge g_{f,h}(x+\p(x))\ge 1-2\epsilon$ and $g_{f,h}(\p(x))\ge c(1-4\bfepsilon)-(c+4L)k$, and considering that $|a-b|\le \max\{|1-a|,|1-b|\}$ when $a,b\in[-1,1]$, we have $|g_{f,h}(\p(x))-g_{f,h}(\p(x)+x)|\le \max\{1-c(1-4\bfepsilon)+(c+4L)k,2\epsilon\}= 1-c(1-4\bfepsilon)+(c+4L)k$, by $1-c(1-4\bfepsilon)+(c+4L)k-2\epsilon\ge0$.

Then, when $x\in \D_{clean}$, we have $1\ge g_{f,h}(x+\p(x))\ge g_{f,h}(x)+c(1-4\bfepsilon)-(c+4L)k\ge c(1-4\bfepsilon)-(c+4L)k-2\bfepsilon>0$, so considering that $g_{f,h}(\p(x))\ge c(1-4\bfepsilon)-(c+4L)k$, we have $|g_{f,h}(\p(x))-g_{f,h}(\p(x)+x)|\le \max\{1-g_{f,h}(\p(x)),1-g_{f,h}(x+\p(x))\}\le
1-c(1-4\bfepsilon)+(c+4L)k+2\bfepsilon$.
So we get result one.

{\bf Result Two: With probability $1-2\frac{4(1-\eta)}{4(1-\eta)+N\eta}-2\delta$, we have  $E_{x\sim D^{l_p}_\S}|g_{f,h}(\p(x))-g_{f,h}(\p(x)+x)|\le 2(1-c(1-4\epsilon)+(c+4L)k)+R_1(H,F)+8\sqrt{\frac{\ln(1/\delta)}{N\eta\alpha}}$, where $\Rad_1(H,F)$ is a value of the Rademacher complexity of $H$ and $F$.}

By Result one, we can use $\sum_{(x,y)\in\D_{\tr}\setminus\D_{clean}}|g_{f,h}(\p(x))-g_{f,h}(\p(x)+x)|$ to estimate $E_{x\sim D^{l_p}_\S}[|g_{f,h}(\p(x))-g_{f,h}(\p(x)+x)|]$.

First, use Theorem \ref{fanhua} and similar to the proof of Theorem \ref{fenbugj1}, with probability $1-\frac{4(1-\eta)}{4(1-\eta)+N\eta}-\delta$, we have 
\begin{equation*}
\begin{array}{ll}
&E_{x\sim D^{l_p}_\S}[g_{f,h}(\p(x))-g_{f,h}(\p(x)+x)]\le
 1-c(1-4\epsilon)+(c+4L)k+\\
 &2(\Rad_{N\eta\alpha/2}^{D^{l_p}_\S(\p)}(H)+\Rad_{N\eta\alpha/2}^{D^{l_p}_\S(\p)}(F)+\Rad_{N\eta\alpha/2}^{D^{l_p}_\S(x+\p)}(H)+\Rad_{N\eta\alpha/2}^{D^{l_p}_\S(x+\p)}(F))+4\sqrt{\frac{\ln(1/\delta)}{N\eta\alpha}}.
\end{array}
\end{equation*}
Notice that, we use $\Rad_{k}^D(H_1)+\Rad_{k}^D(H_2)\ge \Rad_{k}^D(H_1-H_2)$ 
for any hypothesis space $H_1,H_2$ and $k\ge0$, distribution $D$ here.

Second, similar as before, with probability $1-\frac{4(1-\eta)}{4(1-\eta)+N\eta}-\delta$, we have 
\begin{equation*}
\begin{array}{cc}
&E_{x\sim D^{l_p}_\S}[-g_{f,h}(\p(x))+g_{f,h}(\p(x)+x)]
\le 1-c(1-4\epsilon)+(c+4L)k+\\
&2(\Rad_{N\eta\alpha/2}^{D^{l_p}_\S(\p)}(H)+\Rad_{N\eta\alpha/2}^{D^{l_p}_\S(\p)}(F)+\Rad_{N\eta\alpha/2}^{D^{l_p}_\S(x+\p)}(H)+\Rad_{N\eta\alpha/2}^{D^{l_p}_\S(x+\p)}(F))+4\sqrt{\frac{\ln(1/\delta)}{N\eta\alpha}}.
\end{array}
\end{equation*}
Adding these two equations, we get the result, where $$\Rad_1(H,F)=4(\Rad_{N\eta\alpha/2}^{D^{l_p}_\S(\p)}(H)+\Rad_{N\eta\alpha/2}^{D^{l_p}_\S(\p)}(F)+\Rad_{N\eta\alpha/2}^{D^{l_p}_\S(x+\p)}(H)+\Rad_{N\eta\alpha/2}^{D^{l_p}_\S(x+\p)}(F)).$$ 

{\bf Result Three: With probability $1-2\frac{4(1-\eta)}{4(1-\eta)+N\eta}-2\delta$, we have  $E_{x\sim \D_S^{\ne l_p}}|g(\p(x))-g(\p(x)+x)|\le 2(1-c(1-4\epsilon)+(c+4L)k+2\bfepsilon)+\Rad_2(H,F)+8\sqrt{\frac{\ln(1/\delta)}{N(1-\eta)}}$, where $\Rad_2(H,F)$ is a value of the Rademacher complexity of $H$ and $F$.}

By Result one, we can use $\sum_{(x,y)\in\D_{clean}}|g_{f,h}(\p(x))-g_{f,h}(\p(x)+x)|$ to estimate $E_{x\sim \D_S^{\ne l_p}}[|g_{f,h}(\p(x))-g_{f,h}(\p(x)+x)|]$.

First, using Theorem \ref{fanhua} and similar as in proof of Theorem \ref{fenbugj}, with probability $1-\frac{4\eta}{4\eta+N(1-\eta)}-\delta$, we have 
\begin{equation*}
\begin{array}{ll}
&E_{x\sim \D_S^{\ne l_p}}[g_{f,h}(\p(x))-g_{f,h}(\p(x)+x)]
\le1+2\bfepsilon-(c(1-4\epsilon)-(c+4L)k)+\\
&2(\Rad_{N(1-\eta)/2}^{\D_S^{\ne l_p}(\p)}(H)+\Rad_{N(1-\eta)/2}^{\D_S^{\ne l_p}(\p)}(F)+\Rad_{N(1-\eta)/2}^{\D_S^{\ne l_p}(x+\p)}(H)+\Rad_{N(1-\eta)/2}^{\D_S^{\ne l_p}(x+\p)}(F))
+4\sqrt{\frac{\ln(1/\delta)}{N(1-\eta)}}
\end{array}
\end{equation*}
and then, similar as above, with probability $1-\frac{4\eta}{4\eta+N(1-\eta)}-\delta$, we have 
\begin{equation*}
\begin{array}{ll}
&E_{x\sim \D_S^{\ne l_p}}[-g_{f,h}(\p(x))+g_{f,h}(\p(x)+x)]
\le1+2\bfepsilon-(c(1-4\epsilon)-(c+4L)k)+\\
&2(\Rad_{N(1-\eta)/2}^{\D_S^{\ne l_p}(\p)}(H)+\Rad_{N(1-\eta)/2}^{\D_S^{\ne l_p}(\p)}(F)+\Rad_{N(1-\eta)/2}^{\D_S^{\ne l_p}(x+\p)}(H)+\Rad_{N(1-\eta)/2}^{\D_S^{\ne l_p}(x+\p)}(F))
+4\sqrt{\frac{\ln(1/\delta)}{N(1-\eta)}}.
\end{array}
\end{equation*}
Adding them, we get the result, and $\Rad_2(H,F)=4(Rad_{N(1-\eta)/2}^{\D_S^{\ne l_p}(\p)}(H)+\Rad_{N(1-\eta)/2}^{\D_S^{\ne l_p}(\p)}(F)+\Rad_{N(1-\eta)/2}^{\D_S^{\ne l_p}(x+\p)}(H)+\Rad_{N(1-\eta)/2}^{\D_S^{\ne l_p}(x+\p)}(F))$.

{\bf Summarize:}
Finally, considering that 
\begin{equation}
\label{drs}
\begin{array}{ll}
&\E_{x\sim \D_{\S}}[|(f-h)(\p(x),l_p)-(f-h)(x+\p(x),l_p)|]\\ 
=&(1-\eta)\E_{x\sim \D_S^{\ne l_p}}[|(f-h)(\p(x),l_p)-(f-h)(x+\p(x),l_p)|]\\
&+\eta\E_{x\sim \D^{l_p}_{\S}}[|(f-h)(\p(x),l_p)-(f-h)(x+\p(x),l_p)|]\\
=&(1-\eta)\E_{x\sim \D_S^{\ne l_p}}[|g_{f,h}(\p(x))-g_{f,h}(x+\p(x))|]+\eta\E_{x\sim \D^{l_p}_{\S}}[|g_{f,h}(\p(x))-g_{f,h}(x+\p(x))|]
\end{array}
\end{equation}
by using Result two and three in equation \eqref{drs}, defining $\Rad(H,F)=\Rad_1(H,F)+\Rad_2(H,F)$ and using $\eta<1$, we get the result.
\end{proof}

Now, we use Proposition \ref{xbx1} to prove the Proposition \ref{xbx1-q}:
\begin{proof}
    For $H$ in Proposition \ref{xbx1-q}, we define a new hypothesis space $H_1=\{h_1(x,y)=1-h_y(x)\|h\in H\}$. Similarly, we define $F_1$.

    We show that $H_1$ and $F_1$ satisfy the conditions in Proposition \ref{xbx1}:

    (1): It is easy to see that for any $h_1\in H_1$, we have  
    $h_1(x,y_1)+h_1(x,y_2)=2-(h_{y_1}(x)+h_{y_2}(x))\ge 1$, and similar for $F_1$. $1-h_y(x)$ and $h_y(x)$ have the same Lipschitz constant.

    (2): Condition (1) in Proposition \ref{xbx1} stands for $F_1$. By condition (t1) in Proposition \ref{xbx1-q} and $F_1(x,y)=1-f_y(x)$, we can prove this.

    (3): Condition (2) in Proposition \ref{xbx1} and condition (t2) in \ref{xbx1-q} are the same; condition (3) in Proposition \ref{xbx1} and condition (t3) in \ref{xbx1-q} are the same.

    So, $h_1$ and $F_1$ satisfy the conditions in Proposition \ref{xbx1}. We now use the proposition for $h_1$ and $F_1$. Considering that $h_1(x,y)-F_1(x,y)=f_y(x)-h_y(x)$, so we have
\begin{equation*}
\begin{array}{ll}
&\E_{x\sim \D_{\S}}[|(f-h)_{l_p}(\p(x))-(f-h)_{l_p}(x+\p(x))|]\\
\le &2(1-c(1-4\epsilon)+k(c+4L))+2\epsilon)+\Rad(H_1,F_1)+16\sqrt{\frac{\ln(1/\delta)}{N\alpha}}.
\end{array}
\end{equation*}
$\Rad(H_1,F_1)$ is the value of Radermacher complexity of $H_1$ and $F_1$, mentioned in Proposition \ref{xbx1}. For such a Rademacher complexity, using Lemma \ref{ysxd1} and $\Rad_M^\D(H)\le \frac{N}{M}Rad_N^\D(H)$ for any $M<N$, distribution $\D$ and hypothesis space $H$, we have
\begin{equation*}
\begin{array}{ll}
&\E_{x\sim \D_{\S}}[|(f-h)_{l_p}(\p(x))-(f-h)_{l_p}(x+\p(x))|]\\
\le &2(1-c(1-4\epsilon)+(c+4L)k+2\epsilon)+\Rad2(\Hyp_{W,D})+16\sqrt{\frac{\ln(1/\delta)}{N\alpha}}\\
=&2(1-c+(c+4L)k+(2+4c)\epsilon)+\Rad2(\Hyp_{W,D})+16\sqrt{\frac{\ln(1/\delta)}{N\alpha}}.
\end{array}
\end{equation*}
$\Rad2$ is defined in Definition \ref{xbx1-q}, and the proposition is proved.
\end{proof}

\subsection{Proof of Proposition \ref{zjz1}}
\label{zyxdl}
\begin{proof}
We first show that $L$ is a simple feature recognition space with constant 1.

If $w$ satisfies $w(x+\p(x))\ge1-\eta_1$ for $\forall(x,0)\in D$ and $wx\le \eta_1$ for $\forall(x,1)\in D$, then we consider the linear function $wx-\eta_1$. 
Because we have $wx\le \eta_1$ for $\forall(x,1)\in D$, so $wx-\eta_1\le 0$ for $\forall(x,1)\in D$. Because $D$ is linearly inseparable, and there must be a $(x_0,0)\in D$ such that $wx_0-\eta_1\le 0$, but $w(x_0+\p(x_0))\ge1-\eta_1$. Thus, we have $w\p(x_0)\ge1-2\eta_1$.


Considering that $||\p(x)-\p(x_0)||_2\le k$ for all $(x,0)\in D$, we have $w\p(x)\ge 1-2\eta_1-||w||_2||(\p(x)-\p(x_0))||_2\ge 1-2\eta-k$. Moreover, $w(x_1+\p(x))-wx_1=w\p(x)$, so as said above, when $(x,0)\in D$ we have  $w(x_1+\p(x))-wx_1\ge 1-2\eta-k$, which is what we want.

Second, we show that $H$ is a simple feature recognition space with constant $a$.

Because $f$ is an increasing function, similar as before, if $h(l(x+\p(x)))\ge1-\eta_1$ for $\forall(x,0)\in D$ and $h(l(x))\le \eta_1$ for $\forall(x,1)\in D$, then we have $h(l(x_0))-\eta_1\le 0$ for some $(x_0,0)\in D$. 
As a consequence, we have  $1-2\eta_1\le h(l(x_0+\p(x_0)))-h(l(x_0))\le l(x_0+\p(x_0))-l(x_0)$ by $f'(x)\le 1$. Because $l(x)\in L$, let $l(x)=wx$, so $w\p(x_0)\ge (1-2\eta_1)$, and considering that $h(0)=0$, we have $h(w\p(x_0))\ge a(1-2\eta_1)$.

Then, similar to Step one, we can get $w\p(x)\ge1-2\eta_1-k$. As a consequence, $h(l(x))=h(w\p(x))\ge a(1-2\eta_1-k)$, and $h(l(x_1+\p(x_0)))-h(l(x_1))\ge a(w(x_1+\p(x_0))-w(x_1))=aw\p(x_0)\ge a(1-2\eta_1-k)$, so we get the result.

%

\end{proof}


\section{More  Details on the Experiments}
\label{mat}


\subsection{The Experiment Setting}
\label{tbc}
We want to perform experiments in more practical settings:

1: The attacker only accesses part of the training set, so we only use a portion of the training set data in the process of generating triggers.

2: The attacker cannot control the training process of the victim network, so we use standard training models for the victim network.

3: The attacker does not know the structure of the victim network and does not have great computing power, so we only use smaller networks independent of the victim network in the process of generating triggers.

In previous backdoor attacks, some of these conditions have not been assumed. For example, attackers are assumed to be able to access the entire network training set  \cite{gao2023not}; 
attackers know the structure of the network \cite{ba-zhuanyi}; attackers can control the training process \cite{gu2017badnets}; and attackers have some additional information \cite{ba-zhuanhua}.

We use networks VGG16 \cite{VGG}, ResNet18, WRN34-10 \cite{Res} and datasets CIFAR-10, CIFAR-100 \cite{krizhevsky2009learning}, SVHN, and TinyImagenet with 100 classes \cite{le2015tiny}. 
When we train victim network, we use SGD, we have 150 epochs in the training, the learning rate is 0.01, and reduce to $80\%$ at 40-th,80-th, 120-th epochs, use weight decay $10^{-4}$, momentum $0.9$, each data in the training set will flip or randomly crop before inputting network in the training. 

We will use Algorithm \ref{alg-ap1} to find the trigger $\p(x)$, and the basic settings of Algorithm \ref{alg-ap1} are as follows.
We randomly choose $50\%$ samples from the original training set to be $T$.

We choose $\F_1$ as VGG9 for dataset CIFAR-10 (VGG16 for CIFAR100 or SVHN, Resnet34 for TinyImageNet), use adversarial training with PGD-10 and $L_\infty$ norm budget $8/255$ for dataset CIFAR-10 or SVHN ($4/255$ for CIFAR-100 and TinyImageNet) to train $\F_1$. There are 200 epochs in the adversarial training, learning rate is 0.01, and reduce by half at 100-th and 150-th epochs; use weight decay $10^{-4}$, momentum $0.9$; each data in the training set will flip or randomly crop after doing PGD-10.

We choose $\F_2$ to be a two layer network (structure is shown below). There are 40 epochs in the training of $\F_2$; learning rate is 0.01; use weight decay $10^{-4}$, momentum $0.9$; each data in the training set will flip or randomly crop before inputting $\F_2$.
%
%
The budget of poison is $L_\infty$ norm $8/255$ to $32/255$.

Once we obtain the trigger $\p(x)$, we will randomly select some samples with label $l_p$ in the training set and add the trigger to them. 
Then, we will train the network by using the poisoned training set and measure the accuracy and the attack success rate on the test set.
%

\textbf{Running Time}
We do our experiments on Pytorch and GPU NVIDIA GeForce RTX 3090. Under the above experimental setup, the time required for the experiment is shown in Table \ref{tab-time}. It can be seen that most of the time is spent on adversarial training of the network $\F_1$.

\begin{table}[!ht]
\caption{Training time (in minutes). Gp100 means generate poison for 100 samples by using $\F_1$ and $\F_2$.}
\label{tab-time}
\centering
\begin{tabular}{lccccccccc}
\hline
dataset& $\F_1$& $\F_2$& victim $\F$ &Gp100\\
CIFAR-10&  400 & 30 & 120 & 2\\ 
CIFAR-100& 600& 30 & 140 & 3\\ 
SVHN&   1000&  35 & 160 & 3 \\ 
TinyImageNet&   1600&  48& 250& 4\\
\hline
\end{tabular}
\end{table}

\textbf{Reasons for $\F_1$.} 
$\F_1$ is a network which is used to create adversarial noise, and is trained on a small clean dataset. 
We hope to conduct the experiment under the premise ``Attacker does not about victim network,'' so we avoid using a network with the same structure as victim network in the process of producing poison. On the other hand, we also hope to do the experiment under premise ``Attacker does not have great computing power,'' so we try to use a smaller network to generate the poison as much as possible.

\textbf{The structure of $\F_2$.}

The first layer: with Channel 64 and $3\times 3$ convolution, padding=1, do Relu and Maxpooling.

The second layer: shape to $16384(=64*16*16)$ dim vector ($65536(=64*32*32)$ for TinyImageNet), and do fully connected layer, and output a 2-dim vector. 

\textbf{Reasons for $\F_2$.} 
$\F_2$ is a network that is used to create shortcuts, and such shortcuts are used to make the clean and poison dataset linear separable, and $\F_2$ is trained by Min-Min method. Generally speaking, if the data is not too complex, the structure of this network does not need to be particularly large, a two layer network is enough to create short cut.

We do not choose $\F_2$ to be a linear function, because we always use data enhancement in the training, consider that making the enhanced data linearly separable is very difficult, so we let $\F_2$ to be a two-layer network.

\textbf{About PGD.}

PGD-$N$ means using PGD with $N$ steps, and $1/255$ ($2/255$, $3/255$, $4/255$) attacking rate to get adversarial with budget $8/255$ ($16/255$, $24/255$, $32/255$).

\subsection{More Experimental Result}
Table \ref{tab-baseline-detail} is the supplement of Tables \ref{tab-cifar10} and \ref{tab-budget}, which is the result on CIFAR10 under backdoor attacks with various settings.

\begin{table}[!ht]
\caption{Baseline evaluations on CIFAR-10. Poison model Accuracy ($A$), poison model target label accuracy ($A_t$), and attack success rate ($ASR$).}
\label{tab-baseline-detail}
\centering
\begin{tabular}{cccccccccc}
\hline
Poison budget:&0.6\%&1\%&2\%&0.6\%&1\%&2\%\\
\hline
\multicolumn{7}{c}{VGG16}\\
\hline
Bound:&\multicolumn{3}{c}{$8/255$}
&\multicolumn{3}{c}{$16/255$}\\
\hline
$A$ & 91$\%$ & 91$\%$& 93$\%$ & 92$\%$& 91$\%$& 91$\%$\\
$A_t$ & 92$\%$ & 92$\%$& 90$\%$ & 92$\%$& 91$\%$& 90$\%$\\
$ASR$& 13$\%$ & 48$\%$& 51$\%$ & 82$\%$& 91$\%$& 94$\%$\\
\hline
Bound:
&\multicolumn{3}{c}{$24/255$}&\multicolumn{3}{c}{$32/255$}\\
\hline
$A$ & 91$\%$& 90$\%$& 91$\%$& 90$\%$& 92$\%$& 91$\%$\\
$A_t$ &  92$\%$& 91$\%$& 92$\%$& 92$\%$& 89$\%$& 90$\%$\\
$ASR$&  97$\%$& 99$\%$& 99$\%$& 99$\%$& 99$\%$& 99$\%$\\
\hline
\multicolumn{7}{c}{ResNet18}\\
\hline
Bound:&\multicolumn{3}{c}{$8/255$}
&\multicolumn{3}{c}{$16/255$}\\
\hline
$A$ & 93$\%$ & 92$\%$& 90$\%$ & 93$\%$& 91$\%$& 93$\%$\\
$A_t$ & 92$\%$ & 92$\%$& 93$\%$ & 93$\%$& 93$\%$& 92$\%$\\
$ASR$& 14$\%$ & 33$\%$& 47$\%$ & 86$\%$& 93$\%$& 94$\%$\\
\hline
Budget:
&\multicolumn{3}{c}{$24/255$}&\multicolumn{3}{c}{$32/255$}\\
\hline
$A$ & 92$\%$& 92$\%$& 91$\%$& 92$\%$& 92$\%$& 90$\%$\\
$A_t$ &  92$\%$& 92$\%$& 91$\%$& 91$\%$& 92$\%$& 90$\%$\\
$ASR$&  96$\%$& 99$\%$& 99$\%$& 98$\%$& 99$\%$& 99$\%$\\
\hline
\end{tabular}
\end{table}

Table \ref{tab-datasets-detail} is the supplement of Table \ref{tab-datasets}, which is the result on some datasets under backdoor attacks with various budgets.

\begin{table}[!ht]
\tabcolsep=0.18cm
\caption{Accuracy(A) and attack success rate ($ASR$) on CIFAR-100, SVHN and TinyImageNet, target label $0$.} 
\label{tab-datasets-detail}
\centering
\begin{tabular}{cccccccccc}
\hline
\multicolumn{7}{c}{CIFAR-100}\\
\hline
Budget:
&\multicolumn{2}{c}{$8/255$}&\multicolumn{2}{c}{$16/255$}&\multicolumn{2}{c}{$32/255$}\\
poison rate:&0.6\%&0.8\%&0.6\%&0.8\%&0.6\%&0.8\%\\ 
$A$&71$\%$&73$\%$&73$\%$&72$\%$&72$\%$&71$\%$\\
$ASR$&4$\%$&14$\%$&85$\%$&92$\%$&99$\%$&99$\%$\\
\hline
\multicolumn{7}{c}{SVHN}\\
\hline
Budget:
&\multicolumn{2}{c}{$8/255$}&\multicolumn{2}{c}{$16/255$}&\multicolumn{2}{c}{$32/255$}\\
poison rate:&1$\%$&2\%&1\%&2\%&1\%&2\%\\ 
$A$&93$\%$&93$\%$&93$\%$&92$\%$&93$\%$&91$\%$\\
$ASR$&16$\%$&22$\%$&63$\%$&79$\%$&90$\%$&99$\%$\\
\hline
\multicolumn{7}{c}{TinyImageNet}\\
\hline
Budget:
&\multicolumn{2}{c}{$8/255$}&\multicolumn{2}{c}{$16/255$}&\multicolumn{2}{c}{$32/255$}\\
poison rate:&0.6\%&0.8\%&0.6\%&0.8\%&0.6\%&0.8\%\\ 
$A$&61$\%$&62$\%$&60$\%$&60$\%$&59$\%$&60$\%$\\
$ASR$&2$\%$&9$\%$&61$\%$&82$\%$&99$\%$&99$\%$\\
\hline
\end{tabular}
\end{table}

Tables \ref{tab-attacks-detail} and \ref{tab-attacks-detailx} follow Table \ref{tab-attacks}, and more comparison is shown in it.
\begin{table}[!ht]
 \caption {Benchmark results of attack success rate on CIFAR-10 with VGG16 and ResNet18. Comparison of our method to popular clean-label attacks. Poison ratio is $1\%$ and perturbation have different $l_\infty$-norm bound from $8/255$ to $32/255$.}
\label{tab-attacks-detail}
\centering
\begin{tabular}{cccccccc}
\hline
Victim & \multicolumn{3}{c}{VGG16} & \multicolumn{3}{c}{ResNet18}\\
\hline
Budget:&$\frac{8}{255}$&$\frac{16}{255}$&$\frac{32}{255}$&$\frac{8}{255}$&$\frac{16}{255}$&$\frac{32}{255}$\\
\hline
Ours & {\bf 48$\%$} & {\bf 91}$\%$ & {\bf 99$\%$} & {\bf 33$\%$} & {\bf 93$\%$}&{\bf 99$\%$}\\
Clean Label  & 18$\%$ & 44$\%$ & 84$\%$ & 12$\%$ & 22$\%$&80$\%$\\
Invisible Poison  & 23$\%$ & 71$\%$ & 99$\%$ & 24$\%$ & 73$\%$&98$\%$\\
Hidden Trigger  & 36$\%$ & 80$\%$ & 99$\%$ & 26$\%$ & 75$\%$&99$\%$\\
Narcissu  & 20$\%$ & 60$\%$ & 92$\%$ & 16$\%$ & 50$\%$&92$\%$\\
Image-specific  & 22$\%$ & 68$\%$ & 94$\%$ & 18$\%$ & 70$\%$&95$\%$\\
Reflection  & 26$\%$ & 68$\%$ & 99$\%$ & 20$\%$ & 54$\%$&99$\%$\\
Sleeper-Agent & 20$\%$ & 61$\%$ & 97$\%$ &29 $\%$& 70$\%$& 99$\%$\\
\hline
\end{tabular}
\end{table}

\begin{table}[!ht]
 \caption {Benchmark results of attack success rate on TinyImagenet with network WRN34-10. Comparison of our method to popular clean-label attacks. Poison ratio is $0.8\%$ and budget is $l_\infty$-norm bound $16/255$. You can see that our results are very outstanding.}
\label{tab-attacks-detailx}
\centering
\begin{tabular}{cccccccc}
\hline
Attack methods & ASR\\
\hline
Ours & {\bf 82$\%$} \\
Clean Label  & 4$\%$ \\
Invisible Poison  & 27$\%$ \\
Hidden Trigger  & 44$\%$ \\
Narcissu  & 2$\%$\\
Image-specific  & 8$\%$ \\
Reflection  & 11$\%$ \\
Sleeper-Agent & 4$\%$\\
\hline
\end{tabular}
\end{table}

%

\subsection{Detail of each Attack}
\label{bct}

This section shows the experiment settings of the attack methods in Section \ref{com}.
For all attacks, we basically follow the algorithm in the original paper for experimentation, but they have some different settings from ours in creating poison and backdoor, which need to be slightly modified according to our experimental settings.

{\bf Ensure Invisible.} These attacks design the poison added to the training set as invisible (i.e. ensure the $L_\infty$ norm not more than 8/255, 16/255 or 32/255), but some attack methods will design the trigger as a patch, which has no norm limitation, as shown in Figure \ref{fig-lp--p}. This is unfair for the attacks that design triggers as invisible. For greater fairness, we require that triggers must also be invisible (i.e. bound by the $L_\infty$ norm) for all attacks, and as compensation, triggers can be added to the whole image rather than a patch. If the trigger of an attack method exceeds the norm limit, we use the following method to compress its trigger within the norm of the limitation:
$x$ is a sample, $t(x)$ is the trigger of $x$ without norm constraint, $\epsilon$ is the norm limitation. We will compress $t(x)$ to a trigger $t_{wn}(x)$ satisfying $||t_{wn}(x)||\le\epsilon$ as: 
$$t_{wn}(x)=\argmin_{||t||<\epsilon}||F(x+t)-F(x+t(x))||$$
where $F$ is a feature extraction network which has nine convolution hidden layers.

\begin{figure}[!ht]
\centering
\includegraphics[width=14cm]{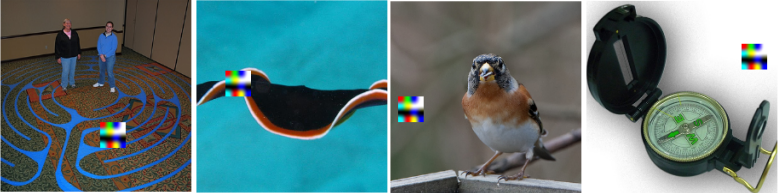}
\caption{When trigger is a patch without norm limitation, it is not invisible. This figure is from \cite{souri2022sleeper}.}
\label{fig-lp--p}
\end{figure}


{\bf Cost of Attack.} For the sake of fairness, we try to ensure that all attack methods use networks with similar scale in the process of generating poison and trigger(VGG9 for Cifar10 and ResNet34 for TinyImagenet). Details about these attack methods are as follows.

(Clean Label): Use PGD-40 to find adversarial of a VGG9, which is trained on the clean training set.

(Invisible Poison): The auto-encoder architecture has the structure described in paper \cite{ba-zhuanhua}. The original trigger is a disturbance with $L_0$ norm 100 to the lower left part of the image.

(Hidden Trigger): The original trigger is a patch that disturbs 100 pixels to the lower left part of the image. Use a VGG9 to be a feature extractor in creating the poison of the training set. 

(Narcissu): Use a VGG9 as the surrogate network.

(Image-specific): A U-net with depth 18 has been taken as Trigger Generator.

(Reflection): Select an image and use its reflection as a trigger, adding the reflection to image through convolution.

(Sleeper-agent): The original trigger is a patch that disturbs 100 pixels to the lower left part of the image. Use a $VGG9$ to be the proxy network in creating poison.

\subsection{Strengthen Attack}
\label{xxc}

We try to strengthen our attack method to bypass the defense by using the following methods. 

Bypass (AT): Use the method \cite{fu2022robust}(min-min-max method, which can create the shortcut for adversarial learning) to train $\F_2$, and $\F_2$ expands to VGG9.

Bypass (Data Augmentation): Use strong data enhancement in the training process $\F_1$ and $\F_2$ in algorithm \ref{alg-ap1}, and $\F_2$ expands to VGG9.


Bypass (DPSGD): Use DPSGD in the training process $\F_1$ and $\F_2$ in Algorithm \ref{alg-ap1}; 

Bypass (Frequency Filter): Using the method in \cite{dabouei2020smoothfool} to control frequency domain of poison, making poison low frequency.

\subsection{Some Others Attacks}
\label{LF}

\textbf{The transferability of attacks.}

Please note that we do experiment under the setting 'the attacker does not know the structure of the victim network' as said in the experimental setup in Appendix \ref{tbc}. So, all surrogate networks we used during creating trigger will not change based on the victim network, only change based on the dataset. What is mentioned above is reflected in the experimental setup in Appendix \ref{tbc}. So the triggers which we used in table \ref{tab-cifar10} for ResNet18, VGG16 and WRN are the same, then table \ref{tab-cifar10} naturally implies the transferability of our trigger between different networks are good. 

But the drawback of our trigger is that it cannot be transferred between different datasets, for example, the trigger created for CIFAR-10 can not be used on CIFAR-100. 

{\bf Weather Victim Network learns the trigger feature?} We will continue the experiment in Table \ref{tab-cifar10} and consider two special sets: $D_{op}=\{(\p(x),0)|(x,y)\in\D_{test}\}$ and $D_{op1}=\{(0.5+\p(x),0)|(x,y)\in\D_{test}\}$, where $\D_{test}$ is the test set of CIFAR-10, and $0.5+\p(x)$ represents each weight of $\p(x)$ plus $0.5$. These are the sets of poisons. We tested the accuracy of the victim network on them to measure whether the network has learned noise features, and the result is given in Table \ref{tabp}.

\begin{table}[!ht]
\caption{Accuracies on $\D_{op}$ and $\D_{op1}$ for ResNet18 (R) and VGG16 (V). PN is the number of poisoned samples.
}
\label{tabp}
\centering
\begin{tabular}{lccccccccc}
\hline
Budget:&\multicolumn{3}{c}{$8/255$}
&\multicolumn{3}{c}{$16/255$}\\
\hline
PN:&0.6\%&1\%&2\%&0.6\%&1\%&2\%\\
\hline
V, $\D_{op}$ & 84$\%$ & 95$\%$& 98$\%$ & 100$\%$& 100$\%$& 100$\%$\\
R, $\D_{op}$ & 80$\%$ & 92$\%$& 98$\%$ & 99$\%$& 100$\%$& 100$\%$\\
V, $\D_{op1}$ & 94$\%$ & 98$\%$& 99$\%$ & 100$\%$& 100$\%$& 100$\%$\\
R, $\D_{op1}$ & 93$\%$ & 98$\%$& 98$\%$ & 99$\%$& 100$\%$& 100$\%$\\
\hline
\end{tabular}
\end{table}


The victim network has a very high accuracy for the poisoned sets $\D_{op}$ and $\D_{op1}$, even with the budget $8/255$, which means the victim network has learned the trigger feature. 
However, the data in Table \ref{tab-cifar10} are not as good as those in Table \ref{tabp}, because when the input contains both original features and poison features, each of them will affect the output of the network and the final result is decided by them together. Therefore, when the network learns the original features well, it is also necessary to increase the scale of poison, and this is why under the premise of budget $8/255$, the effect does not appear to be good in Table \ref{tab-cifar10}. So, we can reach the result: Victim network is very sensitive to poison features and uses them for classification. But the victim network still focuses on original image features, and it will make a choice on the stronger side of these two features.

\subsection{More Detail for Section \ref{tf}}
\label{mdf}
The construction details of the poison in Section \ref{tf} are as follows.

(RN($L_\infty$ budget)): $\p(x)$ is random noise with $L_\infty$ budget $16/255$ or $32/255$. The method for random selection of noise is: each pixel of noise is i.i.d. obeying the Bernoulli distribution in $\{-16/255,16/255\}$ or $\{-32/255,32/255\}$. Please note that a noise vector is selected as trigger for all samples, but not each sample selects a noise as trigger.

(RN($L_0$ budget)) $\p(x)$ is random noise with $L_0$ budget $200$ or $300$. The method for generating noise is: Randomly select 200 or 300 pixels and change their values to 0. Please note that the position of each pixels that becomes 0 in each image is the same, but not each sample random selects some pixel to become 0.

(UA) $\p(x)$ is universal adversarial disturbance with $L_\infty$ budget $16/255$ and $32/255$, we use the method  \cite{moosavi2017universal} to find universal adversarial disturbance of a trained VGG9 (training method is as same as $\F_1$ as mentioned in Section \ref{tbc}).

(Adv) $\p(x)$ is adversarial disturbance with $L_\infty$ budgets $16/255$ and $32/255$. We use PGD-40 to find adversarial disturbance of a trained VGG9 (training method is as same as $\F_1$ as mentioned in Section \ref{tbc}).

(SCut) $\p(x)$ is shortcut noise with $L_\infty$ budgets $16/255$ and $32/255$. The method for generating shortcut noise is Min-Min method \cite{huang2021unlearnable}, using a 2-depth neural network to find shortcut.

(Ours): Following   Algorithm \ref{alg-ap1} and Section \ref{tbc}.

The result on VGG is in the following table \ref{tab-2-bc}.
\begin{table}[!ht]
\caption{The supplement of Tabel \ref{tab-verify}. The value of $V_{adv}$ and $V_{sc}$, and Accuracy (A) and attack success rate (ASR) on the test set. Use 12 different triggers.}
\label{tab-2-bc}
\centering
\begin{tabular}{lcccccccc}
\hline
Poison Type &$V_{adv}(\uparrow)$&$V_{sc}(\downarrow)$&$ASR(\uparrow)$&A\\
RN $L_\infty$, $16/255$ & 2.72 & 0.014& 16$\%$ & 91$\%$& \\
RN $L_\infty$, $32/255$ & 6.31 & $10^{-4}$& 98$\%$ & 90$\%$\\
RN $L_0$, $200$ & 4.64& 0.004& 76$\%$ & 91$\%$\\
RN $L_0$, $300$ & 6.20 &0.003& 92$\%$ & 90$\%$\\
UA $16/255$ & 3.19  & 0.002& 63$\%$ & 91$\%$\\
UA $32/255$ & 17.92 &$10^{-4}$& 93$\%$ & 90$\%$\\
Adv $16/255$ & 9.77  & 1.27& 44$\%$ & 90$\%$\\
Adv $32/255$ & 18.63 &0.35& 84$\%$ & 90$\%$\\
SCut  $16/255$ & 1.21 & $10^{-4}$& 33$\%$ & 91$\%$\\
SCut  $32/255$ &4.02 & $10^{-5}$& 91$\%$ & 90$\%$\\
Ours  $16/255$ & 7.21 & 0.001& 91$\%$ & 90$\%$\\
Ours $32/255$ & 15.95 & $10^{-4}$& 99$\%$ & 92$\%$\\
\hline
\end{tabular}
\end{table}

\subsection{Verify Theorem \ref{fenbugj}}
\label{ver}
In this section, we use experiments to verify Theorem \ref{fenbugj}. 

We will show that: when $\p(x)$ is fixed, {\bf the poison rate } will affect accuracy.

We use dataset CIFAR-10, network ResNet18, target label 0, and poison 1000, 2000, 3000 or 4000 randomly selected images with label 0. We use the follow trigger $\p(x)$ to test our conclusion.  

(PN): $\p(x)$ is random noise with $L_\infty$ budget $8/255$, $16/255$ and $32/255$. The method for random select noise is: each pixel of noise is i.i.d. obeying the Bernoulli distribution in $\{-8/255,8/255\}$ or $\{-16/255,16/255\}$ or $\{-32/255,32/255\}$.

(MI): Mixed image poisoning method. $x+\p(x)$ is calculated as: randomly find an $x_1$ without label $0$, and make $x+\p(x)=(1-\lambda) x+\lambda x_1$, where $\lambda=0.05,0.15,0.25$. 

(Ours): $\p(x)$ is generated by Algorithm \ref{alg-ap1} with budgets $8/255$, $16/255$ or $32/255$.

The following table shows the accuracy and accuracy of the image with label 0.

\begin{table}[!ht]
\caption{Accuracy (A) and accuracy of image with label 0($A_t$) on the test set. The ones in parentheses are $A_t$. Use nine different triggers. }

\centering
\begin{tabular}{lcccccccc}
\hline

&1000&2000&3000&4000\\
PN, $8/255$: & 92(91)$\%$& 92(91)$\%$& 91(90)$\%$& 90(88)$\%$\\
PN, $16/255$: & 91(92)$\%$& 91(91)$\%$& 90(90)$\%$& 90(88)$\%$\\
PN, $32/255$: & 91(92)$\%$& 90(89)$\%$& 89(89)$\%$& 89(87)$\%$\\
MI:$\lambda=0.05$ & 92(91)$\%$& 91(91)$\%$& 92(90)$\%$& 90(89)$\%$\\
MI:$\lambda=0.15$ & 91(92)$\%$& 90(91)$\%$& 90(90)$\%$& 89(89)$\%$\\
MI:$\lambda=0.25$ & 92(90)$\%$& 91(90)$\%$& 89(89)$\%$& 88(87)$\%$\\

Ours, $8/255$ & 92(93)$\%$& 91(92)$\%$& 91(91)$\%$& 90(89)$\%$\\
Ours, $16/255$ & 92(92)$\%$& 90(91)$\%$& 90(89)$\%$& 89(88)$\%$\\
Ours, $32/255$ & 90(90)$\%$& 91(90)$\%$& 90(89)$\%$& 90(88)$\%$\\
\hline
\end{tabular}
\end{table}

We can see that the higher the poison rate, the greater the impact on accuracy. However, the degree of decline is not significant: for more than 3,000 poisoned samples, the accuacy just decreases 4$\%$. 
This is because the poison budget is  controlled to be small.

\end{document}